\documentclass[twoside,11pt]{article}

\usepackage{blindtext} 
\usepackage{jmlr2e}
\usepackage{amsmath}
\usepackage{bbm}
\usepackage{lastpage}
\hypersetup{hidelinks}
\usepackage{booktabs}
\usepackage{textcomp}
\jmlrheading{24}{2023}{1-\pageref{LastPage}}{6/23; Revised 12/23}{12/23}{23-0791}{Meysam Alishahi, Anna Little, and Jeff M. Phillips}

\def\1{\mathbbm{1}}
\def\d{{\rm d}}

\def\R{\mathbb{R}}

\def\S{\mathcal{S}}

\def\N{\mathcal{N}}

\def\x{\mathbf{x}}
\def\z{\mathbf{z}}
\def\y{\mathbf{y}}

\def\x{\boldsymbol{x}}
\def\y{\boldsymbol{y}}
\def\z{\boldsymbol{z}}
\def\M{{\boldsymbol{M}}}
\def\U{{\boldsymbol{U}}}
\def\D{{\boldsymbol{D}}}

\def\A{{\boldsymbol{A}}}
\def\X{{\boldsymbol{X}}}
\def\0{{\boldsymbol{0}}}
\def\eps{\varepsilon}
\def\blambda{{\boldsymbol{\lambda}}}
\def\Md{{\mathcal{M}_d}}
\DeclareMathOperator*{\expt}{\mathbb{E}}

\newtheorem{observation}{Observation}

\usepackage{times}
\ShortHeadings{Linear DML}{Alishahi, Little, and Phillips}

\firstpageno{1}

\begin{document}

\title{Linear Distance Metric Learning with Noisy Labels}

\author{\name Meysam Alishahi \email alishahi@cs.utah.edu \\
       \addr Kahlert School of Computing\\ 
       University of Utah\\
       Salt Lake City, UT 84112, USA
       \AND
       \name Anna Little \email little@math.utah.edu \\
       \addr Department of Mathematics,  Utah Center For Data Science\\
       University of Utah\\
       Salt Lake City, UT 84112, USA
       \AND
       \name Jeff M. Phillips \email jeffp@cs.utah.edu \\
       \addr Kahlert School of Computing, Utah Center for Data Science\\ 
       University of Utah\\
       Salt Lake City, UT 84112, USA}

\editor{Aryeh Kontorovich}

\maketitle

\begin{abstract}%
In linear distance metric learning, we are given data in one Euclidean metric space and the goal is to find an appropriate linear map to another Euclidean metric space which respects certain distance conditions as much as possible. In this paper, we formalize a simple and elegant method which reduces to a general continuous convex loss optimization problem, and for different noise models we derive the corresponding loss functions. 
We show that even if the data is noisy, the ground truth linear metric can be learned with any precision provided access to enough samples, and we provide a corresponding sample complexity bound. 
Moreover, we present an effective way to truncate the learned model to a low-rank model that can provably maintain the accuracy in the loss function and in parameters -- the first such results of this type.  Several experimental observations on synthetic and real data sets support and inform our theoretical results.  
\end{abstract}


\section{Introduction}\label{intro}
The goal of distance metric learning is to map data in a metric space into another metric space in such a way that the distance between points in the second space optimizes some condition on the data.  
Early work in this area focuses mostly on the Euclidean to Euclidean setting, and specifically on the case of learning linear transformations.  For data $\X \in \R^{n \times d}$, it attempts to learn a Mahalanobis distance $\d_\M : \R^d \times \R^d \to \R^{\geq 0}$ as 
 $\d_\M(\x, \y) = \|\x-\y\|_\M = \sqrt{(\x-\y)^t \M (\x-\y)}$. 
This is a metric on the original space $\R^d$ as long as $\M$ is positive definite.  We can decompose $\M$ as $\M = \A\A^t$ for $\A \in \R^{d \times d}$.  Then $\A$ can be used as a  linear map so that $\X' = \A^t \X$ is another point set of $n$ points in $d$ dimensions; for $\x' = \A^t\x$ and $\y' = \A^t\y$ in the new space, $\|\x - \y\|_\M$ is equivalent to the standard Euclidean distance $\|\x' - \y'\|$.  

Linear metric learning has been studied in \citet{JMLR:v10:weinberger09a} for kNN classification, in \cite{NIPS2006_dc6a7e65,4270149,XIANG20083600} via margin/distance optimization, in  \cite{JMLR:v8:sugiyama07b, barhillel05a} via discriminant analysis, and in \cite{NGUYEN2017215} via Jeffrey divergence. 
Many of these linear methods also propose kernelized versions, and kernelized metric learning was also considered in \cite{788121,roth1999nonlinear}. 
%More recently, this field has explored non-linear transformations,
%Initially these were kernelized extensions of the linear approaches~\cite{JMLR:v10:weinberger09a, NIPS2006_dc6a7e65,4270149,NGUYEN2017215,788121,JMLR:v8:sugiyama07b}.  
But the current state of the art uses arbitrarily complex neural encoders that attempt to optimize the final objective with very little restriction on the form or structure of the mapping ~\citep{9880306}.  Merging with the area of feature engineering~\citep{nargesian2017learning,shi2009hash}, these approaches are an integral element of information retrieval~\citep{Patel2021RecallkSL, ramzi:hal-03712933}, natural language processing~\citep{6083923, li2017deep}, and image processing~\citep{1467314}. To access further details, readers can refer to two well-conducted surveys~\cite{BelletHS13,MAL-019}.

In this work, we revisit \emph{linear} distance metric learning.  We posit that there are two useful extremes in this problem, the anything-goes non-linear approaches mentioned above, and the very restrictive linear approaches.  The linear approaches exhibit a number of important properties which are essential for certain applications:
\begin{itemize}
\item  When the original coordinates of the data points have meaning, but for instance are measured in different units (e.g., inches and pounds), then one may want to retain that meaning and interpretability while making the process invariant to the original underlying (and often arbitrary) choice of units.  
\item  Many geometric properties such as linear separability, convexity, straight-line connectivity, vector translation (linear parallel transport) are preserved under affine transformations.  If such features are assumed to be meaningful on the original data, then they are retained under a linear transformation.  
\item  Some physical equations, such as those describing ordinary differential equations (ODEs) can be simulated through a linear transform~\citep{sutherland2009combustion}.  We will demonstrate an example application of this (in Section \ref{sec:exp-real}) where because of changing units, it is not clear how to measure distance in the original space, and locality based learning can be more effectively employed after a linear transformation.  
\end{itemize}

While several prior works have already explored linear distance metric learning~\citep{JMLR:v10:weinberger09a,NIPS2006_dc6a7e65,JMLR:v8:sugiyama07b,NGUYEN2017215}, they often reduce to novel optimization settings where specially designed solvers and analysis are required.  For instance, \cite{ying2012distance} utilize a clever subgradient descent formulation to ensure the learned $\M$ retains its positive-definiteness.  
In our work we provide a simple and natural formulation that converts the linear distance metric learning task into a simple supervised convex gradient descent procedure, basically a standard supervised classification task where any procedure for smooth convex optimization can be employed.    

\subsection{Formulation} 
Specifically, we assume $N$ i.i.d. observations $(\x_i, \y_i) \in \R^d \times \R^d$ and each pair is given a label $\ell_i \in \{{\rm Far, Close}\}$.  Our goal is to learn a positive semi-definite (p.s.d.) matrix $\M$ and threshold $\tau \geq 0$ so that $\|\x_i - \y_i\|_\M^2 \geq \tau$ if $\ell_i = {\rm Far}$ and $\|\x_i - \y_i\|_\M^2 < \tau$ if $\ell_i = {\rm Close}$. 
%\avl{Anna: I found this sentence confusing since the label is noisy. Maybe ``Our goal is to learn a positive semi-definite (p.s.d.) matrix $\M$ and threshold $\tau \geq 0$ so that $\|\x_i - \y_i\|_\M^2 \geq \tau$ is predictive of $\ell_i = {\rm Far}$."} 
Towards solving this we formulate an optimization problem
\begin{equation}
\label{genopt}
\min_{\begin{array}{cc}
%\M,\tau:\\
\tau&\geq 0\\
\M &\succeq 0
\end{array}}
R_N(\M,\tau) = 
\min_{\begin{array}{cc}
%\M,\tau:\\
\tau&\geq 0\\
\M &\succeq 0
\end{array}} \frac{1}{N} \sum_{i=1}^N L(\x_i,\y_i,\ell_i; \M, \tau)
\end{equation}
where $L(\x_i,\y_i,\ell_i; \M, \tau)$ is a loss function that penalizes the mismatch between the observed label and the model-predicted label.   Then we propose to optimize this in an (almost, except for $\tau \geq 0$) unconstrained setting, where we can apply standard techniques like (stochastic) gradient descent: 
\[
\min_{\begin{array}{l}
%\M,\tau:\\
\tau \in \mathbb{R} \\
\A \in \mathbb{R}^{d \times d}
\end{array}}
R_N(\A \A^t,\tau)
\]

\subsection{Our core results}
We analyze this simple, flexible, and powerful formulation and show that:
\begin{itemize}
\item This optimization problem is convex over $\M, \tau$.  Moreover, while optimizing over $\A$ is not convex, we can leverage an observation of \citet{doi:10.1137/080731359} to show that the minimizer $\A^*$ over the unconstrained formulation generates $\M^* = \A^* (\A^*)^t$, which is the minimizer over the convex, but (positive semi-definite) constrained formulation over $\M$.  
\item The sample complexity of this problem is $N_d(\eps,\delta) = O(\frac{1}{\eps^2}(\log \frac{1}{\delta} + d^2 \log\frac{d}{\eps}))$. More specifically, let $f$ be the pdf of the distribution from which difference pairs $\x - \y$ are drawn, let $\ell$ be the (noisy) label associated with a pair, and let $R(\M,\tau) = \mathbf{E}_{\x-\y \sim f}[L(\x,\y,\ell; \M, \tau)]$ be the expected loss.  
Then, given $N_d(\eps,\delta)$ observations, $|R(\hat{\M},\hat{\tau}) - R(\M^*,\tau^*)|\leq \eps$ with probability at least $1-\delta$, where $(\hat{\M},\hat{\tau})$ is the minimizer of $R_N$. 
\item If the labels $\ell_i$ 
%or distances $\|\x_i - \y_i\|_\M$ generating labels, 
are observed with unbiased noise, and the loss function is chosen appropriately to match that noise distribution, then $R_N$ still approximates $R$, and in fact, the minimizers $\hat \M, \hat \tau$ of $R_N$  converge to the true minimizers $\M^*,\tau^*$ of $R$.  
\item Returning a low-rank approximation $\hat \M_k$ to $\hat \M$ can achieve bounded error with respect to $|R(\hat \M_k,\hat \tau) - R(\M^*, \tau^*)|$ and $\| \hat \M  - \M^*\|+|\hat \tau - \tau^*|$, as elaborated on just below.  To the best of our knowledge, this is the first such dimensionality reduction result of this kind.  
\end{itemize}

\subsection{Reasonable Choice for the Loss Function $L$: Logistic noise} 
We assume the labeling of \{\rm{Close, Far}\} through the evaluation of $\|\x_i-\y_i\|_\M^2$ is noisy.  
We will prove that if the noise comes from the Logistic distribution, then 
$$L(\x_i,\y_i,\ell_i; \M, \tau) = -\log\sigma(\ell_i(\|\x_i-\y_i\|_\M^2- \tau))$$
serves as an excellent theoretical choice; here $\sigma(x) = \frac{1}{1+e^{-x}}$ is known as the {\it Logistic} function.
We note that prior work~\citep{5459197} also considered this special case of our formulation, and provided an empirical study on face identification;  
%That is, they only assumed this Logistic model of the problem and then derive a loss function (using MLE approach), and 
however they did not theoretically analyze this formulation.  

As mentioned, our work will show that this form of $L$ is indeed optimal under a Logistic noise model.  We also show that if one assumes a different noise model (e.g., Gaussian), then a different loss function would be more appropriate. 
Furthermore, we show that irrespective to the amount of unbiased noise, we are able to recover the ground truth parameters if we observe enough noisy data, and we provide precise sample complexity bounds.  
Section \ref{ExperimentalResults} also confirms these theoretical results with careful experimental observations and demonstrates that the method is in fact robust to misspecification of the noise model.

\subsection{Dimensionality Reduction}
It is natural to ask if linear distance metric learning approaches can be used for linear dimensionality reduction.  That is, if one restricts to a rank-$k$ positive semi-definite $\M_k$, then we can write $\M_k = \A_k \A_k^t$ where $\A_k \in \R^{d \times k}$.  Hence $\A_k$ can be used as a linear map $x' = \A_k^t x$ from $\R^d \to \R^k$.  
Yet optimizing $R_N(\A_k \A_k^t, \tau)$ with $\A_k \in \R^{d \times k}$ is not only non-convex, the optimization has non-optimal local minima~\citep{doi:10.1137/080731359}.  

Another natural approach is to run the optimization with a full rank $\A$, and then truncate $\A$ by rounding down its smallest $d-k$ singular values to $0$.  As far as we know, no previous analysis of a distance metric learning (DML) approach has shown if this is effective; a direction (singular vector of $\A$) associated with a small singular value of $\A$ could potentially have out-sized relevance towards cost function $R_N$ that we seek to optimize, and this step may induce uncontrolled error in $R_N$.  

In this paper, we detail some reasonable assumptions on the data necessary for this singular value rounding scheme to have provable guarantees.  The key assumption is that the width of the support of the data distribution is bounded by $\sqrt{F}$ (for some parameter $F$), 
and either this support must include measure in regions which assign labels of both Close and Far, or the unbiased noise must be large enough to generate some of each label.  
%and the Malahanobis ball $B^{\M}_\tau$ must be in the interior of this distribution.  \jeff{I feel like this must be required, but not sure we formalize this ... it means projecting onto any direction in $\R^d$ that some pairs of points must be Close and some must be Far.}

Specifically we consider the following algorithm.  
\\ 1) Sample $N = N_d(\eps,\delta)$ pairs $\x, \y$ from a Lebesgue measurable distribution with the width of support at most $\sqrt{F}$ in each direction.  
\\ 2) Solve for $\hat{\A} \in \R^{d \times d}$ and $\hat{\tau} \geq 0$ in $R_N(\A \A^t, \tau)$ using any convex gradient descent solver. % \meysam{ (convex optimization solver. is it not better?)}\jeff{that is more general, like linear programming}. 
\\ 3) For positive integer $k \leq d$, set the $d-k$ smallest singular values of $\hat{\A}$ to $0$, resulting in low-rank matrix $\hat{\A}_k$.  Let $\gamma$ be the value of the $(d-k)$-th singular value of $\hat{\A}$ (the largest one rounded to $0$). 
\\ 4) Return $\hat{\M}_k = \hat{\A}_k \hat{\A}_k^t$, a rank-$k$ positive semi-definite matrix.  

For this algorithm (formalized in Theorems \ref{thm:DR-main1} and \ref{thm:DR-main2}), we claim with probability at least $1-\delta$ 
\[
| R(\hat{\M}_k, \hat \tau) - R(\M^*, \tau^*) | \leq \eps + F \gamma^2.
\]
Thus, for instance if we draw $\x-\y$ from a unit ball (so $F = 1$), and set $\eps' = 2 \eps$, and assume that $\M^*$ has $d-k$ eigenvalues less than $\eps'/4$, then with probability at least $1-\delta$ 
\[
| R(\hat{\M}_k, \hat \tau) - R(\M^*, \tau^*) | \leq \eps'.
\]
% Theorem~\ref{thm:DR-main1} implies $\gamma< 2\eps'/4 = \eps.$

%\avl{Can we really say we have shown all this for the regularized problem \ref{genopt}, since we don't include a regularizer?}
% However, optimizing the choice of $\M$ while maintaining $\M$ is psd can be complicated.  However, we can leverage a result of \citet{doi:10.1137/080731359} to show that we can instead represent $\M$ as $\A^t\A$ and minimize $\A$ over $\R^{\d \times d}$.  That is, we argue that the solution $\A^*, \tau^*$ for the unconstrained minimization problem 
% \[
% \min_{\begin{array}{l}
% %\M,\tau:\\
% \tau \in \mathbb{R} \\
% \A \in \mathbb{R}^{d \times d}
% \end{array}}
% R_N(\A^t \A,\tau)
% \]
% generates $\M^* = (\A^*)^t \A^*$.  That is, under this formulation, we do not need to specifically optimize under the constrained space of positive definite matrices, but can consider the unconstrained space of $\A \in \mathbb{R}^{d \times d}$ -- despite this formulation not being convex \avl{since $\A^*$ is not unique}, its optimization will result in the proper minimum.  

\subsection{Outline}
In Section \ref{sec:model} we more carefully unroll this model formulation, and in Section \ref{sec:conv} we show sample complexity and convergence results, including under noise.  
In Section \ref{sec:alg} we discuss the optimization procedure and show that the unconstrained optimization approach is provably effective, and useful for dimensionality reduction tasks.  
Finally, in Section \ref{ExperimentalResults} we verify our theory on a variety of synthetic data experiments and demonstrate the utility of this linear DML framework on two real data problems that benefit from a learned Mahalanobis distance.

\section{Model and Key Observations}
\label{sec:model}
This section contains the model, the underlying assumptions, and some observations used throughout the paper.

\subsection{Data and Model Assumptions}
We work under the following data model throughout the paper.  
%Let $\Md$ be the set of $d \times d$ positive definite matrices.  
We assume there exists some positive semi-definite $\M^* \in \R^{d \times d}$ that defines a Mahalanobis distance that we seek to discover.  We observe pairs of data points $\x_i, \y_i \in \R^d$, but we only consider the differences of the pairs $\z_i = \x_i - \y_i$.  Moreover, we assume that all observations $\z_i \in \R^d$ are i.i.d. from some unknown distribution with pdf $f(\z)$.  

We also assume there exists a threshold $\tau^*$ which generates labels $\ell_i \in \{{\rm Close, Far}\}$; more specifically, $z_i$ is Close if and only if 
\begin{equation}\label{assumption0}
\tag{Label Assumption}
\|\z_i\|_{\M^*}^2 + \eta_i < \tau^* \, ,
\end{equation}
where $\eta_i$ is a noise term.  Each $\eta_i$ is generated i.i.d. from a distribution ${\rm Noise}(\eta| 0, s)$ which 
%{\sout{is symmetric around $0$ with scale parameter $s>0$ uniquely determining the distribution ${\rm Noise}(\eta| 0, s)$} 
comes form a location-scale family of distributions with location $0$ and scale parameter $s>0$. 
For different distributions, $s$ can have a different meaning. For example, for the normal distribution, $s$ is the standard deviation. For mathematical convenience, we overload ``Close'' = -1 and ``Far'' = +1 so $\ell_i \in \{-1,+1\}$. As
$\|\z_i\|_{\M^*}^2 + \eta_i < \tau^*$ is equivalent to $\|\z_i\|_{\frac{\M^*}{t}}^2 + \frac{\eta_i}{t} < \frac{\tau^*}{t}$ for any $t>0$,  
scaling $\M^*$, $\tau^*$, and $s$ by a common parameter $t$ does not change the labeling distribution, so w.l.o.g., we remove this degree of freedom in the analysis presented in this paper by setting $s=1$.  Thus what the techniques in this paper ultimately recover is actually $\frac{\M^*}{s}$ and $\frac{\tau^*}{s}$.  In other words, we will see that our model is identifiable if $s=1$. Moreover, the following lemma indicates that for the noiseless case, the model is identifiable provided that $\tau=1$.
The following formalization is proved in Appendix~\ref{app:indicator}. 
\begin{lemma}\label{indicator}
Given two pairs $(\M_1, \tau_1)$ and $(\M_2, \tau_2)$ such that $\M_1,\M_2\succeq0$ and $\tau_1,\tau_2 > 0$, if the two functions 
$\z\mapsto\1_{\left\{\|\z\|^2_{\M_1} - \tau_1\geq 0\right\}}$ and $\z\mapsto\1_{\left\{\|\z\|^2_{\M_2} - \tau_2\geq 0\right\}}$ agree for all $\z$, then 
$\frac{\M_1}{\tau_1} = \frac{\M_2}{\tau_2}.$
\end{lemma}

In order to be able to solve our optimization challenges and bound the sample complexity, we need to make a few simple assumptions on $\M^*,\tau^*$ and the data distribution.  We have the following assumption on the model:
\begin{equation}\label{assumption1}
\tag{Model Assumption}
\|\M^*\|_2 \leq \beta \; \text{ and }\; \tau^*\in[0, B]\, .
\end{equation}
%\begin{equation}\label{assumption1}
%\tag{Model Assumption}
%\M^* \in \Md = \left\{\M \in \R^{d \times d}, \text{ is p.s.d. and } \|\M\|_2 \leq \beta\right\}\; \text{ and }\; \tau^*\in[0, B].
%\end{equation} 
Note that as we have assumed that $s =1$, $\M^*$ and $\tau^*$, as well as the upper bounds $\beta$ and $B$, have been scaled by $1/s$.
Since $\beta$ and $B$ later appear in the sample complexity (see Section \ref{noiseaffectcom}), noise affects the sample complexity through these terms.  

Next we assume something about the data we observe.  
Assume $\z_1,\ldots,\z_N \in \mathbb{R}^d$ are $N$ i.i.d. samples from an unknown distribution with probability density function $f(\z)$ with respect to  the Lebesgue measure on $\R^d$. 
Hence $\int_{\mathbb{R}^d}f(\z)\d x = 1$ which implies that the set $\{\z\colon f(\z) \neq 0\}$  has a positive Lebesgue measure.   
We also assume that the probability density function $f(\z)$ has bounded support, i.e., 
\begin{equation}\label{assumption2}
\tag{Data Assumption}
\max \{\|\z\|^2\colon f(\z) \neq 0\}\leq F.
\end{equation}
By this assumption, we know that almost surely    $\|\z\|^2\leq F$. 
Also, by the \ref{assumption2}, for each 
\[
\M \in \Md = \left\{\M \in \R^{d \times d} : \M \text{ is p.s.d., } \|\M\|_2 \leq \beta\right\},
\]
we have
$\|\z\|^2_\M\leq \|\M\|_2\|\z\|^2\leq \beta F$.
When it is clear by context, we sometimes write $\mathcal{M}$ for $\Md$.  
Accordingly, for $\z\sim f(\z)$, almost always 
$$\left| \|\z_i\|^2_{\M} - \tau\right|\leq \max\{B, \beta F\}\quad\quad \forall (\M, \tau) \in \Md \times [0,B].$$
Since the sign of $\|\z_i\|^2_{\M} - \tau$ determines the label of $\z$, it is reasonable to assume that 
\begin{equation}\label{assumptopnBdA}
\tag{Meta Assumption}
B\leq \beta F,
\end{equation}
which implies 
\begin{equation}
\label{uppernormz}\left| \|\z_i\|^2_{\M} - \tau\right|\leq \beta F\quad\quad \forall (\M, \tau) \in \Md \times [0,B].
\end{equation}

\subsection{Convexity}\label{convexitysubsection}
Our core objective is optimizing $R_N$ or $R$ over $(\M, \tau)$ such that $\M\succeq 0,\tau\geq0$. 
%\in \Md \times [0, \infty)$.  \avl{Should we update to $\M\succeq 0,\tau\geq0$?}\jeff{Yes, I agree since we only assume those conditions on $\M$ and $\tau$. } 
%\sout{Similar to typical supervised learning problems, we consider loss functions $L$ which for a data point $(\z_i, \ell_i)$ are convex functions of $\ell_i (\|\z\|_M^2 - \tau)$}
 In the same vein as regular supervised learning scenarios, we consider loss functions that are convex with respect to the contribution of each data point $(\z_i, \ell_i)$. To show the loss functions over the space of valid parameters are convex, we  first show that any convex combination of two valid models' parameters is still a valid parameter and for any two models with parameters $(\M_1, \tau_1)$ and $(\M_2, \tau_2)$ and an interpolation parameter $\lambda \in [0,1]$, we have  
\[
\|\z\|^2_{\lambda \M_1 + (1-\lambda)\M_2} - (\lambda \tau_1 + (1-\lambda) \tau_2) 
 = 
\lambda (\|\z\|^2_{\M_1} - \tau_1) + (1-\lambda)(\|\z\|^2_{\M_2} - \tau_2).  
\]
Hence, the convex interpolation of the two models gives us another model with parameter $(\M = \lambda \M_1 + (1-\lambda)\M_2, \tau = \lambda \tau_1 + (1-\lambda) \tau_2)$ which is in the space of the valid parameters.  Coupled with a convex loss function, this implies that minimizing $R_N$ or $R$ over $(\M, \tau)$ with $\M \succeq 0$ and $\tau \geq 0$ is a convex optimization problem.  Hence any critical point is a global minimum.  
However, restricting $\M \succeq 0$ under gradient descent is non-trivial since generically the gradient may push the solution out of that.  While manifold optimization methods have been developed for other matrix optimization challenges~\citep{balzano2010online,vandereycken2013low}, we develop a simpler unconstrained approach in this work.  
%is not a convex set of $\R^{d^2}$.  \jeff{I am confused about something here, just updated the last sentence}
%\avl{I don't understand why this is not a convex set, since for $M_1,M_2\in\Md$, $M=\lambda M_1 + (1-\lambda)M_2$ will also be symmetric, satisfy $x^TMx\geq 0$, and satisfy $\|M\|\leq\beta$.}

\section{Noise Observations and Optimal Loss Functions}
\label{sec:conv}
 
Recall that $\z_1,\ldots,\z_N \in \mathbb{R}^d$ are $N$ i.i.d. samples from an unknown distribution with probability density function $f(\z)$ with bounded support. As the noise distribution $\mathrm{Noise}(\eta |0, 1)$ makes the labeling probabilistic, 
for a given $\z\sim f(\z)$, the probability that the corresponding label is $1$ can be computed as %\avl{Notation: shouldn't we use $\M^*, \tau^*$ in this calculation?}  %Haven't introduced M^*, tau* yet, so I think its fine to use the less heavy notation here
\begin{align}\label{gennoisemodel}
P(\ell=1|\z; \M, \tau) & = P(\eta > \tau - \|\z\|_{\M}^2)\nonumber\\
& = \int_{-\infty}^{\|\z\|_{\M}^2 - \tau}{\rm Noise}(\eta| 0, 1)\d\eta\nonumber\\
& =  \Phi_{\rm Noise}(\|\z\|_{\M}^2 - \tau),
\end{align}
where $\Phi_{\rm Noise}(a) = \int_{-\infty}^a {\rm Noise}(\eta| 0, 1)\d\eta$.
%{\blue 
%\begin{align*}
%p(\ell=1|\z; \M, \tau) & = p(\varepsilon > \tau - \|\z\|_{\M}^2)\\
%& = p\left(\frac{\varepsilon}{s} > \frac{\tau}{s} - \|\z\|_{\frac{\M}{s}}^2\right)\\
%& = \int_{\frac{\tau}{s} - \|\z\|_{\frac{\M}{s}}^2}^{+\infty}{\rm Noise}(\varepsilon| 0, 1)d\varepsilon\\
%& = \int_{-\infty}^{\|\z\|_{\frac{\M}{s}}^2 - \frac{\tau}{s}}{\rm Noise}(\varepsilon| 0, 1)d\varepsilon\\
%& =  \Phi_{\rm Noise}(\|\z\|_{\bar{\M}}^2 - \bar{\tau}),
%\end{align*}
%where $\Phi_{\rm Noise}(a) = \int_{-\infty}^a {\rm Noise}(\varepsilon| 0, 1)d\varepsilon$, $\bar{\M} = \frac{\M}{s}$, and $\bar{\tau} = \frac{\tau}{s}$.
%As the final goal, we want to recover $\M$ and $\tau$. Since any scaler multiplication of $\M$ and $\tau$ gives us the same labeling function, we may assume that %$s=1$.  However, at the end, we recover $\frac{\M^*}{s}, \frac{\tau^*}{s}$ which gives us the labeling as well. } 
Observe that 
$$P(\ell=-1|\z; \M, \tau) =  1 - P(\ell=1|\z; \M, \tau) 
= \Phi_{\rm Noise}(-1(\|\z\|_{\M}^2 - \tau)).$$
Consequently, we have
$$\z, \ell \sim g(\z, \ell;\M, \tau) = f(\z)\Phi_{\rm Noise}\left(\ell(\|\z\|^2_{\M}-\tau)\right),$$
 where $g$ is the density of the random variable $(z, \ell)$ with respect to $\mu_L\otimes \nu$ (the product of the Lebesgue measure $\mu_L$ and the counting measure $\nu$). 
%\sout{For the full uniform convergence result, we will require some properties.}   
 Our theoretical results hold under quite general assumptions on the noise. In particular, 
the noise should be symmetric, continuous, non-zero everywhere, and $-\log \Phi_{\rm Noise}$ should be convex and $\zeta$-Lipschitz on $[-\beta F, \beta F]$. Throughout the paper, %\sout{such a} 
we refer to any noise satisfying these assumptions as  %\sout{is called} 
a \emph{simple noise model}, and specifically consider %\sout{we can consider various noise pdfs  ${\rm Noise}(\eta| 0, 1)$ in this section}  
%\avl{Perhaps we should move the last couple of sentences to the model and assumptions section and create a Noise Assumption? It would be nice to have all the assumptions in one place...we could still leave the specific discussion of the four noise models here.} In particular, four acceptable 
%\meysam{($-\log\Phi_{\rm Cauchy}$ is not convex!!)} 
 the Logistic, Normal, Laplace, and Hyperbolic secant (HS) distributions as special cases.  Table~\ref{tbl:simplenoises} describes these distributions and their associated model constants 
%\sout{for details of these examples and some important constants under our model and data assumptions} 
(see Appendix~\ref{app:table1proof} for verification of these constants). Note for the Logistic and Hyperbolic models,  
${\rm sech}(t) = \frac{1}{\cosh(t)} = \frac{2}{e^t+e^{-t}}$, and in the remainder of this article we will refer to ${\rm Noise}(\eta| 0, 1)$ as ${\rm Noise}(\eta)$ for brevity. 
One could consider other noise models like Cauchy, but then $-\log \Phi_{\rm Noise}$ is not convex.  In Figure~\ref{fig:simplenoisesplot}, we compare the four simple noise distributions when they share the same mean and variance. 
%\jeff{Looking closer, HS is just Logistic with a different scale parameter.  If you want, you can keep a version of it in the experiments, but call it Logistic with wrong scale parameter.}
\begin{table}[h]
    \centering
%\begin{tabular}{p{4.5cm} p{2.2cm} p{3cm} p{2.2cm} p{2.2cm}  }
\begin{tabular}{lllll}
%\hline
%\multicolumn{3}{|c|}{Accuracy} \\
\toprule[2pt]
%Noise	& Logistic: ${\rm L}(\eta|0, 1)$ & Normal: $\N(\eta|0, 1)$ & Laplace: $f(x|0, 1)$ & Hyperbolic Sec.: ${\rm HS}(x|0, 1)$\\
Noise	& Logistic & Normal & Laplace & Hyperbolic\\
	& ${\rm L}(\eta|0, 1)$ & $\N(\eta|0, 1)$ & $f(\eta|0, 1)$ & ${\rm HS}(\eta|0, 1)$\\
\midrule
pdf	& 
   $\frac{1}{4}{\rm sech}^2\left(\frac{\eta}{2}\right)$	& ${\frac {1}{ {\sqrt {2\pi}}}}e^{-{\frac{\eta^2}{2}}}$ & ${\frac{1}{2}}e^{-|\eta|}$ &  $\frac{1}{2}{\rm sech}(\frac{\pi}{2}\eta)$\\
%\hline
% \hline 
std & 
  $\frac{\pi}{\sqrt{3}}$	& $1$   	& $\sqrt{2}$ &  $1$\\ 
\midrule 
%$\zeta$ (cdf is $\zeta$-log-Lipschitz)	& $1$ & $\frac{1}{\Phi(-\beta F)}$	& $1$ & $\frac{\pi}{2\arctan e^{-\beta F}}$ \\
$\zeta$ (cdf is $\zeta$-log-Lipschitz)	& 
   $1$ & $O(\beta F)$	& $1$ & $\frac{\pi}{2}$ \\
%$\omega = \min\limits_{|\eta|\leq \beta F} {\rm Noise}(\eta)$ 
% & $\frac{1}{2(1+e^{\beta F})}$ & $\frac{1}{\sqrt{2 \pi}} e^{-\frac{\beta^2 F^2}{2}}$ & $\frac{1}{2}e^{-\beta F}$ & $\frac{1}{2}{\rm sech}(\frac{\pi}{2}\beta F)$ \\ 
% $\omega = \min\limits_{|\eta|\leq \beta F} {\rm Noise}(\eta)$ 
%  & $\frac{1}{2(1+e^{\beta F})}$ & $e^{-O(\beta^2 F^2)}$ & $e^{-O(\beta F)}$ & $\frac{1}{2}{\rm sech}(\frac{\pi}{2}\beta F)$ \\ 

$\omega = \min\limits_{|\eta|\leq \beta F} {\rm Noise}(\eta)$ 
 & $e^{-O(\beta F)}$ & $e^{-O(\beta^2 F^2)}$ & $e^{-O(\beta F)}$ & $e^{-O(\beta F)}$ \\ 
% $T = \max\limits_{|\eta|\leq \beta F} - \log \Phi_{{\rm Noise}}(\eta)$ & $\log(1+e^{\beta F})$ & $- \log \Phi_{{\rm Normal}}(-\beta F)$ & $\beta F+\log 2$ & \jeff{??????}\\
% \midrule
% $T = \max\limits_{|\eta|\leq \beta F} - \log \Phi_{{\rm Noise}}(\eta)$ & $\beta F + O(1)$ & $\frac{(\beta F)^3}{2} + O(1)$ & $\beta F+O(1)$ & $\frac{3\pi}{2}\beta F + O(1)$\\
 $T = \max\limits_{|\eta|\leq \beta F} - \log \Phi_{{\rm Noise}}(\eta)$ & $O(\beta F)$ & $O((\beta F)^2)$ & $O(\beta F)$ & $O(\beta F)$\\
\bottomrule[2pt]
%$ \frac{\left|\hat{\tau} - \frac{\tau^*}{s}\right|}{\frac{\tau^*}{s}}$ 						&0.023 		&?	& ? \\ 
%\hline
\end{tabular}
   \caption{Simple noise models and key properties and values.}
    \label{tbl:simplenoises}
\end{table} 

\begin{figure}[h]   
\begin{center}  
    \includegraphics[width=0.55\textwidth]{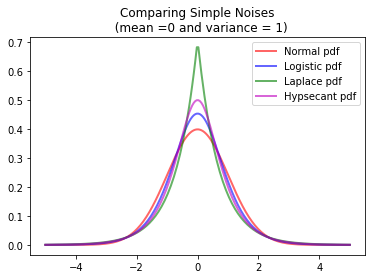}
\end{center}
    \caption{\label{fig:simplenoisesplot} Comparing simple noise models when $\mu = 0$ and $\sigma^2 = 1$.}
\end{figure}

%\avl{Notation: there is no distinction between $\mathcal{M}$ and $\Md$, yes? Need to make notation consistent.}
%\jeff{I added a note after we define $\Md$ that we sometimes use $\mathcal{M}$ when it is clear from context.  Still we should be locally consistent.  Do you think that that is ok?}
%We have also assumed that there are ground truth values for $\M, \tau$ (we called them $\M^*\in \mathcal{M}, \tau^*\in [0, B]$) that the given data and their labels are generated from the corresponding distribution, i.e., 
Note by the \ref{assumption0},
\begin{equation}\label{eq:z_l_dist}
    (\z_1,\ell_1),\ldots,(\z_N,\ell_N) \stackrel{i.i.d.}{\sim} g(\z, \ell;\M^*, \tau^*) = f(\z)\Phi_{\rm Noise}\left(\ell(\|\z\|^2_{\M^*}-\tau^*)\right).
\end{equation}
Since we have a probabilistic model, we now use the MLE method to estimate $\M^*, \tau^*$. 
As one of the main contributions of this work, we will prove that this method works and deduce the corresponding sample complexity. 
The average negative log-likelihood of the given data $\z_1\ldots,\z_N$ and their labels $\ell_i,\ldots,\ell_N$ as a function of $\M$ and $\tau$ is 
\begin{align*}
{\rm NLL}(\M, \tau) &= -\frac 1N\sum_{i=1}^N \log g(\z_i, \ell_i;\M, \tau)\\
& = -\frac 1N\sum_{i=1}^N\left[ \log f(\z_i) +\log p(\ell_i |\z_i; \M, \tau)\right]\\
& = \underbrace{-\frac 1N\sum_{i=1}^N \log f(\z_i)}_\text{independent of $\M,\tau$}  \underbrace{-\frac 1N\sum_{i=1}^N \log p(\ell_i |\z_i; \M, \tau)}_\text{the loss function}.
\end{align*}
Therefore, to solve the MLE, we need to find a p.s.d. matrix $\M$ and a $\tau\in[0, B]$ minimizing 
\begin{align*}
R_N(\M,\tau) & = -\frac 1N\sum_{i=1}^N  \log p(\ell_i |\x_i; \M, \tau)
 = -\frac 1N\sum_{i=1}^N \log \Phi_{\rm Noise}\left(\ell_i(\|\z_i\|^2_{\M} - \tau)\right).
\end{align*}
Thus the optimization problem we are dealing with is 
\begin{align}\label{genMLE1}
\min_{\M\succeq 0,\tau\geq0} R_N(\M,\tau),
\end{align}
%\avl{Shouldn't we just write this as
%\begin{align}
%\min_{\M\succeq 0,\tau\geq0} R_N(\M,\tau)
%\end{align}
%to match \ref{genMLE1new}, since the constraints on spectral norm of $M$ and magnitude of $\tau$ are not actually enforced in the optimization?
%}\jeff{yes, we can change this.  }
where $R_N(\M, \tau) = -\frac{1}{N}\sum_{i=1}^N \log \Phi_{\rm Noise}\left(\ell_i(\|\z_i\|_{\M}^2 - \tau)\right)$.

We will justify that solving this optimization problem with high probability ends up in a guaranteed approximation of $(\M^*, \tau^*)$. 
For fixed but arbitrary $\M, \tau$, using Chebyshev's inequality (or directly by the Law of Large Numbers), we know that 
$R_N(\M,\tau)$ tends to (in measure induced by $g(\z, \ell;\M^*, \tau^*)$)
\begin{align}
R(\M,\tau) & = -\mathbb{E}_{z,\ell \sim g(\z, \ell;\M^*, \tau^*)}\log \Phi_{\rm Noise}\left(\ell(\|\z\|^2_{\M}-\tau)\right)\label{R_N_M_t}\\
& = -\int g(\z, \ell;\M^*, \tau^*)\log \Phi_{\rm Noise}\left(\ell(\|\z\|^2_{\M}-\tau)\right)d\z d\ell\nonumber
\end{align}
provided that $\log \Phi_{\rm Noise}\left(\ell(\|\z\|^2_{\M} - \tau)\right)$ has a bounded variance, which is the case since 
$f(\z)$ has bounded support. 
%However, we need more, a similar result for all $\M, \tau$ simultaneously. \sout{This observation suggests that} 
Thus we can view $R_N(\M,\tau)$ as the empirical risk function and $R(\M,\tau)$ as the true risk function. 
However, we only have access to $R_N(\M,\tau)$. 
%\meysam{\begin{remark}
%Throughout the paper, we will see that there is nothing special about $-\log\Phi_{\rm Noise}(\cdot)$ function appearing in empirical and true risks, but it is a differentiable, convex, strictly decreasing, and $\zeta$-Lipschitz function which smoothly converts $(0,1)$ to $(0,\infty)$. We replace it with any other function with the same properties.
%\end{remark}}
The rest of this subsection can be summarized as follows:
\begin{enumerate}
\item We prove that both functions $R(\M,\tau)$ and $R_N(\M,\tau)$ are convex, so we are dealing with convex optimization problems.
\item We prove that $R(\M,\tau)$ is uniquely minimized at $(\M^*, \tau^*)$.
\item We show that $R_N(\M,\tau)$ converges uniformly in measure to $R(\M,\tau)$. We also bound the corresponding error for an arbitrarily given confidence bound.  
\item Combining these results, we conclude that minimizing $R_N(\M,\tau)$ is a good proxy for minimizing $R(\M,\tau)$. 
\end{enumerate}
\begin{theorem}\label{uniqueR}
%If $\Phi_{\rm Noise}(\cdot)$ is an one-to-one function, then   % Jeff: I think this is implicit in its definition as a CDF
The true loss $R(\M,\tau)$ is uniquely minimized at $(\M^*, \tau^*)$. 
\end{theorem}
%\avl{Is this the only theorem for which we don't require simple noise? If so, why don't we just add it to the model assumptions and make a comment here about which conditions aren't needed. Then we don't have to state the condition over and over.}
%\meysam{Meysam: This is not the only theorem which does not require simple noise assumption. We have some other theorems which only needs partially in simple noise assumptions. For instances, Lemma 3 only advantages of  $\zeta$-log-Lipschitz property, even Theorem 7 only needs the lower bound on $\omega$, and Lemma 9 is valid for any noise model.} \avl{Okay, let's leave it then}
\begin{proof}
First, note that 
\begin{align*}
R(\M,\tau) - R(\M^*,\tau^*) 
& = \mathbb{E}_{\z,\ell\sim g(\z, \ell;\M^*, \tau^*)} \left(\log\frac{  \Phi_{\rm Noise}(\ell(\|\z\|^2_{\M^*}-\tau^*))}{  \Phi_{\rm Noise}(\ell(\|\z\|^2_\M-\tau))}\right)\\
& = \mathbb{E}_{\z,\ell\sim g(\z, \ell;\M^*, \tau^*)} \left(\log\frac{ f(\z) \Phi_{\rm Noise}(\ell(\|\z\|^2_{\M^*}-\tau^*))}{ f(\z) \Phi_{\rm Noise}(\ell(\|\z\|^2_\M-\tau))}\right)\\
&  = {\rm D_{KL}}(g(\x, \ell;\M^*, \tau^*) \| g(\x, \ell;\M, \tau))\geq 0. 
\end{align*}
This indicates that $R(\M,\tau)$ takes its minimum at $(\M^*, \tau^*)$. Moreover, for any $(\M^+, \tau^+)$ at which $R(\M,\tau)$ attains its minimum, 
${\rm D_{KL}}(g(\z, \ell;\M^*, \tau^*) \| g(\z, \ell;\M^+, \tau^+))= 0.$
This implies that $g(\z, \ell;\M^*, \tau^*) = g(\z, \ell;\M^+, \tau^+)$ almost everywhere 
(according to the probability measure induced by $g(\z, \ell;\M^*, \tau^*)$ over $\mathbb{R}^d\times \{-1,1\}$).
Since $\mu_L(\{\z\colon f(\z)>0\})>0$ (Lebesgue measure) and $\log \Phi_{\rm Noise}(\cdot)$ is one-to-one, we can conclude that there is a set $S\subseteq \mathbb{R}^d$ 
such that $\mu_L(S)>0$ and for every $\z\in S$, $$\|\z\|^2_{\M^*}-\tau^* = \|\z\|^2_{\M^+}-\tau^+.$$  
By Lemma~\ref{M1=M2} in Appendix~\ref{Basicproperties}, we then conclude
 $\M^+ = \M^*$ and $\tau^+=\tau^*$. 
\end{proof}
Although we have proved that the true loss $R(\M,\tau)$ is uniquely minimized at $(\M^*, \tau^*)$, in reality, we do not have access to the true loss, but only to the empirical loss $R_N(\M,\tau)$. 
Next we will show that $R_N(\M,\tau)$ is uniformly close to $R(\M,\tau)$ as  $N$ gets large, and then conclude that instead of minimizing $R(\M,\tau)$, we can minimize $R_N(\M,\tau)$ to approximate $(\M^*, \tau^*)$. Note that 
for two given p.s.d. $\M_1$ and $\M_2$, 
\begin{equation} \label{eq:MahCS-obs}
    |\|\x\|^2_{\M_1} - \|\x\|^2_{\M_2}|\leq \|\M_1 -\M_2\|_2 \|\x\|^2.
\end{equation} 
For proof, see Appendix~\ref{app:uc}: Observation~\ref{obs4}.
In the next lemma, using this inequality, we prove that the true loss and empirical loss are both Lipschitz with respect to the metric 
$${\rm d}((\M_1, \tau_1), (\M_2, \tau_2)) = \|\M_1- \M_2\|_2 + |\tau_1 - \tau_2|.$$   
\begin{lemma}\label{upperR1}
If $\log \Phi_{\rm Noise}(\cdot)$ is $\zeta$-Lipschitz, then, for any given $(\M_1, \tau_1),(\M_2, \tau_2)\in  \mathcal{M}\times [0,B]$, 
\begin{enumerate}
\item $|R(\M_1, \tau_1) - R(\M_2, \tau_2)| < \zeta(F+1){\rm d}((\M_1, \tau_1), (\M_2, \tau_2)),$ 
\item $|R_N(\M_1, \tau_1) - R_N(\M_2, \tau_2)| < \zeta(F+1){\rm d}((\M_1, \tau_1), (\M_2, \tau_2)).$
\end{enumerate}
\end{lemma}
\begin{proof}
% Because of similarity, we only prove the first inequality which follows as below.  The proof of the other works the same way.   
We start with the proof of the first inequality. 
In view of Equation~\ref{R_N_M_t}, we have 
\begin{align*}
|R(\M_1, \tau_1) - R(\M_2, \tau_2)| 
& =  \left| \mathbb{E}_{\z,\ell}\Big[\log\Phi_{\rm Noise}\left(\ell(\|\z\|^2_{\M_1}-\tau_1)\right) - \log\Phi_{\rm Noise}\left(\ell(\|\z\|^2_{\M_2}-\tau_2)\right)\Big]\right|\\
& \leq  \mathbb{E}_{\z,\ell}\Big|\log\Phi_{\rm Noise}\left(\ell(\|\z\|^2_{\M_1}-\tau_1)\right) - \log\Phi_{\rm Noise}\left(\ell(\|\z\|^2_{\M_2}-\tau_2)\right)\Big|\\
\text{($\log\Phi_{\rm Noise}(\cdot)$ is $\zeta$-Lipschitz)}\quad
& \leq \zeta\mathbb{E}_{\z}\Big[ \Big| (\|\z\|^2_{\M_1}- \tau_1) - (\|\z\|^2_{\M_2}-\tau_2)\Big|\Big]\\
\text{(using \ref{eq:MahCS-obs})}\quad
& \leq \zeta\mathbb{E}_{\z} \Big |\z^t(\M_1 -\M_2)\z \Big| + \zeta\mathbb{E}\Big[|\tau_1 -\tau_2 |\Big]\\
& \leq \zeta\|(\M_1 -\M_2)\|_2 \mathbb{E}\Big[\|\z\|^2\Big]  + \zeta |\tau_1 -\tau_2 |\\
\text{(\ref{assumption2})}\quad
& \leq \zeta\|(\M_1 -\M_2)\|_2 F + \zeta |\tau_1 -\tau_2 |\\
& < \zeta(F+1){\rm d}\big((\M_1, \tau_1), (\M_2, \tau_2)\big), 
\end{align*}
where $(\z,\ell)\sim g(\z, \ell;\M^*, \tau^*)$ (see Equation~\ref{eq:z_l_dist}). Note that   
$$R_N(\M,\tau) = -\expt_{(\z,\ell)\in_u \{(\z_i,\ell_i)\colon i\in [N]\}} \log \Phi_{\rm Noise}\left(\ell_i(\|\z_i\|^2_{\M} - \tau)\right),$$
where $(\z,\ell)\in_u \{(\z_i,\ell_i)\colon i\in [N]\}$ indicates uniform distribution over $\{(\z_i,\ell_i)\colon i\in [N]\}$. 
If we use the uniform distribution over 
$\{(\z_i,\ell_i)\colon i=1,\ldots,N\}$ in place of $g$, then, with an identical approach, we have the second statement. 
\end{proof}

%\jeff{How about we say the following:}
%{\blue The proof of statement 2 is identical with $\frac{1}{N}\sum_{i=1}^N f(\x_i)$ in place of $\mathbb{E}_{\x} f(\x)$ for various functions $f$ at each step.}
%\meysam{Meysam: How about what I added at the beginning and end of proof?}

% {\blue 
% \meysam{Although, the reviewer asked, I am not convinced to keep the proofs of the both parts in the paper!}
% The proof of the other works the same way. 
% \begin{align*}
% |R_N(\M_1, \tau_1) - R_N(\M_2, \tau_2)| 
% & =  \left| \frac{1}{N}\sum_{i=1}^N\Big[\log\Phi_{\rm Noise}\left(\ell_i(\|\x_i\|^2_{\M_1}-\tau_1)\right) - \log\Phi_{\rm Noise}\left(\ell_i(\|\x_i\|^2_{\M_2}-\tau_2)\right)\Big]\right|\\
% & \leq  \frac{1}{N}\sum_{i=1}^N \Big|\log\Phi_{\rm Noise}\left(\ell_i(\|\x_i\|^2_{\M_1}-\tau_1)\right) - \log\Phi_{\rm Noise}\left(\ell_i(\|\x_i\|^2_{\M_2}-\tau_2)\right)\Big|\\
% \text{($\log\Phi_{\rm Noise}(\cdot)$ is $\zeta$-Lipschitz)} \quad
% & \leq \frac{\zeta}{N}\sum_{i=1}^N \Big|(\|\x_i\|^2_{\M_1}- \tau_1) - (\|\x_i\|^2_{\M_2}-\tau_2)\Big|\\
% & \leq  \frac{\zeta}{N}\sum_{i=1}^N \Big |\x_i^t(\M_1 -\M_2)\x_i \Big| + \frac{\zeta}{N}\sum_{i=1}^N|\tau_1 -\tau_2 |\\
% \text{(using \ref{eq:MahCS-obs})}\quad
% & \leq \zeta\|(\M_1 -\M_2)\|_2 \left(\frac{1}{N}\sum_{i=1}^N\|\x_i\|^2\right)  + \zeta |\tau_1 -\tau_2 |\\
% \text{(\ref{assumption2})}\quad
% & \leq \zeta\|(\M_1 -\M_2)\|_2 F + \zeta |\tau_1 -\tau_2 | \\
% & < \zeta(F+1){\rm d}\big((\M_1, \tau_1), (\M_2, \tau_2)\big). 
% \end{align*}
% }

As $-\log\Phi_{\rm Noise}(\cdot)$ is a decreasing function, using Equation~\ref{uppernormz}, we have 
\begin{equation}\label{Tvalue}
    0\leq -\log \Phi_{\rm Noise}\left(\ell_i(\|\z_i\|^2_{\M} - \tau)\right)\leq -\log\Phi_{\rm Noise}(-\beta F) = T,
\end{equation}
which indicates that the random variables $-\log \Phi_{\rm Noise}\left(\ell_i(\|\z_i\|^2_{\M} - \tau)\right)$ are bounded by a value $T$; see Table \ref{tbl:simplenoises}.   
%When dealing with a summation of bounded i.i.d. random variables,  one strong concentration inequality to use is Chernoff-Hoeffding bound.   %%% Jeff: no need for this commentary unless you actually restate the CH bound so readers know which version you are using.  
In the next theorem, we prove that, with high probability,  the empirical loss $R_N$ is everywhere close to the true loss $R$. We provide a sketch of the proof here; for the complete proof, see Appendix~\ref{app:uc}.  The full bound for $N_d(\eps,\delta)$ appears in \ref{samcomexact}; it is 
%\sout{$\mathrm{poly}(F,T, \log(B), \log(\beta))$} 
a polynomial function of $F,T, \log(B), \log(\beta)$, but for simplicity of presentation we omit the precise dependence on these constants.  

% \avl{Anna (notation): sometimes we capitalize Logistic and sometimes we don't; let's be consistent.}
% \meysam{ Meysam: 
% Now everywhere we have Logistic. 
% One question! If we prefer ``logistic'' over ``Logistic'', then we should delete ``the'' everywhere before ``Logistic''. Am I right? For example, ``the Logistic distribution'' would be wrong. }
% \avl{I think it is okay, I would still say ``the Logistic distribution'', just like for Gaussian.}

\begin{theorem}\label{uniformconv1}
Assume that the noise model ${\rm Noise}(\eta)$ is simple. 
For any $\varepsilon,\delta>0$, define 
$$N_d(\varepsilon, \delta) =O\left(\frac{1}{\varepsilon^2}\Big[
\log\frac{1}{\delta}  + d^2\log \frac{d}{\varepsilon} \Big]\right).$$
If 
$N>N_d(\varepsilon, \delta)$, then with probability at least $1-\delta$, 
$$\sup_{(\M,\tau)\in \mathcal{M}\times [0, B]}\left|R_N(\M, \tau) - R(\M, \tau)\right|<\varepsilon.$$
\end{theorem}
\begin{proof}
Set $\alpha = \frac{\varepsilon}{3\zeta(F+1)}$.
Consider $\mathcal{E}=\{(\M_i, \tau_i); i = 1,\ldots,m=m(\alpha)\}$ as an $\alpha$-cover for $\mathcal{M}\times[0, B].$  
%\jeff{Even in a sketch like this, you want to include an expression for $m(\alpha)$ in terms of $\zeta$ and $F$ so a reader can easily see how they show up in the big-Oh.}
By Lemma~\ref{upperR1}, for every $(\M,\tau)$, there exists an index $i\in[m]$ such that 
$$|R(\M, \tau) - R(\M_i, \tau_i)| < \frac{\varepsilon}{3} \quad\quad\text{and}\quad\quad 
|R_N(\M, \tau) - R_N(\M_i, \tau_i)| < \frac{\varepsilon}{3}$$
which concludes 
$\left|R_N(\M, \tau) - R(\M, \tau)\right| \leq \frac{2\varepsilon}{3} + |R_N(\M_i, \tau_i) - R(\M_i, \tau_i)|.$
Using an appropriate upper bound for $m(\alpha) = \frac{B}{\alpha}(4 \beta d \sqrt{d} /\alpha)^{d^2}$ in Lemma \ref{lem:eps-cover}, and applying a Chernoff-Hoeffding bound and union bound, we conclude the desired result (for full details, see Appendix~\ref{app:uc}).
\end{proof}
Note that combining Theorem~\ref{uniqueR} and Theorem~\ref{uniformconv1}, we conclude that minimizing $R_N(\M,\tau)$ is a good proxy for minimizing $R(\M,\tau)$. 
We restate this result in the next theorem.

%In this section, we will see that, with high probability, we can approximate $(\M^*, \tau^*)$ with any given precision if $N$ is large enough. 
%Let us remind that, Theorem~\ref{uniqueR} illustrates that $R(\M, \tau)$ is uniquely minimized at $(\M^*,\tau^*)$. 
%The next theorem asserts that if $N$ is large enough, the value of the true loss on parameters minimizing the empirical loss is close to the minimum of the true loss. 
\begin{theorem}~\label{lossapprox}
Assume that the noise model ${\rm Noise}(\eta)$ is simple. 
For any given $\varepsilon, \delta >0$, if $N > N(\frac{\varepsilon}{2}, \delta)$, then, with probability at least $1-\delta$,
for any point $(\hat{\M}, \hat{\tau})$ minimizing $R_N(\M, \tau)$, we have 
$$0< R(\hat{\M}, \hat{\tau}) -  R(\M^*, \tau^*) < \varepsilon.$$
\end{theorem}
\begin{proof}
As $N > N(\frac{\varepsilon}{2}, \delta)$, by Theorem~\ref{uniformconv1}, with probability at least $1-\delta$, we have 
$$\left|R_N(\M, \tau) - R(\M, \tau)\right|< \frac{\varepsilon}{2}\quad\quad\text{for all } (\M,\tau)\in\mathcal{M}\times [0,B].$$
Consequently, with probability at least $1-\delta$, 
\begin{align*}
R(\hat{\M}, \hat{\tau}) - \frac{\varepsilon}{2} & <   R_N(\hat{\M}, \hat{\tau})\\
& \leq R_N(\M^*, \tau^*)\quad\quad\text{($R_N(\M, \tau)$ minimized at $(\hat{\M}, \hat{\tau})$)}\\
& < R(\M^*, \tau^*) + \frac{\varepsilon}{2},
\end{align*}
implying that 
\[
0\leq  R(\hat{\M}, \hat{\tau}) -  R(\M^*, \tau^*) \leq \eps. 
\]
\end{proof}

In the next four subsections, we consider the simple noise models determined by the Logistic, Gaussian, Laplace, and Hyperbolic secant distributions.

\subsection{Logistic Distribution as the Noise}\label{Logisticnoisemodel}
The probability density function of the Logistic distribution is 
$${\rm L}(x|\mu, s) = \frac{1}{4s}{\rm sech}^2\left(\frac{x-\mu}{2s}\right),$$
where $\mu$ is the mean, $s$ is the scale parameter of this distribution, and $\frac{s^2\pi^2}{3}$ is 
the variance. 
The Logistic distribution looks very much like a normal distribution and is sometimes used as an approximation for it. In this subsection, we assume that the noise has a Logistic distribution with $\mu = 0$ and $s = 1$, i.e., 
${\rm Noise}(\eta) = {\rm L}(\eta|0, 1)$ (since scaling $\M^*, \tau^*$, and $\eta$ does not change the labeling distribution, w.l.o.g. we may assume $s=1$). The Cumulative distribution function of ${\rm L}(x|0, 1)$ is the sigmoid function 
$\sigma(x) = \frac{1}{1+e^{-x}}$, i.e., 
$\Phi_{\rm L}(x) = \sigma(x).$
As the Logistic noise is simple with $\zeta =1$ (see Table~\ref{tbl:simplenoises}), we can apply Theorem~\ref{uniformconv1}. 
Plugging $\sigma(x)$ in for $\Phi_{\rm Noise}(x)$ in Optimization Problem~\ref{genMLE1}, we obtain 
$R_N(\M, \tau) = -\frac{1}{N}\sum_{i=1}^N \log \sigma\left(\ell_i(\|\z_i\|_{\M}^2 - \tau)\right)$.

Since in the Logistic case we have a closed form for $\Phi_{\rm L}(x) = \sigma(x),$ the loss function becomes computationally easier to work with.
As the main setting for the paper, we will thus assume that the noise comes from a Logistic distribution, although we consider other noise models and loss functions in our experiments. 
%\sout{by default we work with the Logistic noise and corresponding loss function.}  

%In Section~\ref{Gaussiannoise}, we briefly study the case that noise is Gaussian. 
%Moreover, in Section~\ref{Noisylabeling}, we experimentally study another natural kind of noise as well. 

\subsection{Normal Distribution as the Noise}\label{subsec:normalnoise}
%As the Normal distribution is also simple, 
If we assume the noise has a Normal instead of a Logistic distribution, 
then $\Phi_{\rm Noise}(x)$ becomes a probit function instead of the sigmoid function. 
Indeed, if we set ${\rm Noise}(\eta) = \mathcal{N}(\eta| 0, 1)$, then 
$$\Phi(a) = \Phi_{\rm Noise}(a) = \int_{-\infty}^a \mathcal{N}(\eta| 0, 1)\d\eta,$$ 
which is known as the {\it probit function}.  
Once again plugging $\Phi$ into $R_N(\M, \tau)$ in Optimization Problem~\ref{genMLE1}, we obtain 
$R_N(\M, \tau) = -\frac{1}{N}\sum_{i=1}^N \log \Phi\left(\ell_i(\|\z_i\|_{\M}^2 - \tau)\right)$.
Unfortunately, the probit function has no closed-form formula, so optimizing $R_N$ is not as simple as in the Logistic case.  However we will observe in Section \ref{ExperimentalResults} that since the Logistic and Gaussian distributions are very similar, the Logistic loss does well under Normal noise.  

\subsection{Laplace Distribution as the Noise}\label{subsec:laplacenoise}
As the third natural option for noise, we assume that 
$${\rm Noise}(\eta) = {\mathcal{}\rm Laplace}(\eta | 0, 1) = \frac{1}{2}e^{-|\eta|}.$$
In this setting, 
$$\Phi_{{\rm Laplace}}(a) = \int_{-\infty}^a \frac{1}{2}e^{-|\eta|}\d\eta = \left\{
\begin{array}{ll}
\frac{1}{2}e^{a} & a\leq 0\\ \\ 
1 - \frac{1}{2}e^{-a} & a\geq 0.
\end{array}
\right.
$$
Similar to the Logistic case, we obtain a closed-form formula for $\Phi_{{\rm Laplace}}(a)$, 
so that this setting is also convenient to work with. 
Note that 
$$-\log\Phi_{{\rm Laplace}}(a) = 
\left\{
\begin{array}{ll}
-a + \log 2& a\leq 0\\ \\ 
-\log\left(1 - \frac{1}{2}e^{-a}\right) & a\geq 0
\end{array}\right. \, ,
$$
which yields a closed-form formula for $R_N$.

% Again we can rewrite $R_N(\M, \tau)$ as 
% \begin{align*}
% R_N(\M, \tau) = & \sum_{i=1}^N -\log\Phi_{{\rm Laplace}}\left(\ell_i(\|\z_i\|_{\M}^2 - \tau)\right)\\
% = & \sum\limits_{\ell_i(\|\z_i\|_{\M}^2 - \tau) \geq 0} \left[-\ell_i(\|\z_i\|_{\M}^2 - \tau) +\log 2\right]\\ 
% & + \sum\limits_{\ell_i(\|\z_i\|_{\M}^2 - \tau) < 0} -\log\left(1 - \frac{1}{2}e^{-\ell_i(\|\z_i\|_{\M}^2 - \tau)}\right)
% \end{align*}

 \subsection{Hyperbolic Secant Distribution as the Noise}\label{subsec:Hyperbolicnoise}
 The hyperbolic secant distribution is a continuous probability distribution whose probability density function is 
 $${\rm HS}(\eta|\mu, \sigma) = \frac{1}{2\sigma}{\rm sech}(\frac{\pi}{2\sigma}(\eta-\mu))$$
 and whose Cumulative distribution function is 
 $$\Phi_{\rm HS}(\eta|\mu, \sigma) = \frac{2}{\pi} \arctan(\exp({\frac{\pi}{2\sigma}(\eta-\mu)})).$$
 The mean and variance of this distribution are $\mu$ and $\sigma^2$ respectively.  
 As the last option for noise, we consider ${\rm Noise}(\eta) = {\rm HS}(\eta|0, 1)$.
 In this case, we have 
 $$\Phi_{\rm HS}(a) = \frac{2}{\pi} \arctan\left(\exp({\frac{\pi}{2}\eta})\right).$$
 Plugging $\Phi_{\rm HS}$ into $R_N(\M, \tau)$ in Optimization Problem~\ref{genMLE1}, we obtain 
 $$R_N(\M, \tau) = -\frac{1}{N}\sum_{i=1}^N \log \left[
  \arctan\left(\exp\left({\frac{\pi}{2}\left(
  -\ell_i(\|\z_i\|_{\M}^2 - \tau)
  \right)}\right)\right) 
 \right]+ {\rm Constant}\, ,$$
 and we can ignore the constant term. 

% \subsection{Cauchy distribution as the noize}
% As the last option for noize, we consider the Cauchy distribution with the following pdf 
% $${\rm Cauchy}(x| x_0,\gamma) = { 1 \over \pi } \left[ { \gamma \over (x - x_0)^2 + \gamma^2  } \right].$$
% This distribution had no mean and variance. However, it is symmetric with central point $x_0$ and its mean absolute deviation (MAD) is $\gamma$.
% This distribution uniquely determines by $x_0$ and $\gamma$. In our setting, we have $x_0 =0$ and, after scaling $\M^*$ and $\tau^*$ by $\frac{1}{\gamma}$, we have 
% $${\rm Noise}(\eta) = {\rm Cauchy}(\eta|0, 1) = { 1 \over \pi } \left[ { 1 \over \eta^2 + 1  } \right].$$
% Accordingly, 
% $$\Phi_{{\rm Cauchy}}(a) = { 1 \over \pi }\int_{-\infty}^a  { 1 \over \eta^2 + 1  } \d\eta = 
% \frac{1}{\pi} \arctan\left(a\right)+\frac{1}{2}.
% $$
% Plugging it into $R_N(\M, \tau)$ at Optimization Problem~\ref{genMLE1}, we obtain 
% $$R_N(\M, \tau) = -\sum_{i=1}^N \log \left[\frac{1}{\pi} \arctan\left(\ell_i(\|\z_i\|_{\M}^2 - \tau)\right)+\frac{1}{2}\right].$$
% \meysam{Unfortunately, $-\log\Phi_{{\rm Cauchy}}(\cdot)$ is not convex!}

\section{Algorithms, Approximation, and Dimensionality Reduction}
\label{sec:alg}
In this part, we mainly explore some properties of Optimization Problem~\ref{genMLE1} and consider the potential ways to solve it.  

\subsection{How to Solve Optimization Problem~\ref{genMLE1}}
We will prove that solving Optimization Problem~\ref{genMLE1} yields close approximations of the parameters $\M^*$ and $\tau^*$. 
We restate the optimization problem here:  
\begin{align}\label{genMLE1new}
\min_{\M\succeq 0,\tau\geq0} R_N(\M,\tau).
\end{align}
As we need to enforce $\M$ to be p.s.d., using gradient descent directly is difficult.
Notice that $\M$ is p.s.d. if and only if $\M = \A\A^\top$ for some $\A_{d\times k}$ where $k\leq d$; in this case  $\M$ indeed has rank at most $k$. 
Therefore, if we replace $\M$ by $\A\A^\top$ and optimize over $\A$, we no longer need to maintain the p.s.d. condition on $\M$.  Then the optimization problem can be rewritten as follows:
\begin{align}\label{genMLE1AAt}
\min_{\tau\geq 0, \A\in \mathbb{R}^{d\times k}} R_N(\A\A^\top,\tau).
\end{align}
If we set $k=d$, then Optimization~\ref{genMLE1AAt} is equivalent to Optimization~\ref{genMLE1}.
The only downside of this reformulation is that we lose the convexity by this change of variable.
So we are dealing with a non-convex optimization, and thus there may be no guarantee that gradient descent will converge to a global minimum.
Fortunately, the next theorem proved by \citet{doi:10.1137/080731359} resolves this issue.
We remind the reader that for convex optimization problems, global minimums and stationary points are equivalent. 
\begin{theorem}[\citealp{doi:10.1137/080731359}]
A local minimizer $\A^*$ of Problem~\ref{genMLE1AAt} provides a stationary point (global minimum) $\M = \A^*(\A^*)^\top$ of Problem~\ref{genMLE1new} 
if $\A^*$ is rank deficient $({\rm rank}(\A^*)<k)$.
Moreover, if $d=k$, then any local minimizer $\A^*$ of Problem~\ref{genMLE1AAt} provides a stationary 
point (global minimum) $\M^* = \A^*(\A^*)^\top$ of Problem~\ref{genMLE1new}.
\end{theorem}
So, we can use gradient descent for $k=d$ to find a local minimum $\A^*$ of Problem~\ref{genMLE1new}, then using this theorem we know that 
$\M^*=\A^*(\A^*)^\top$ is a global minimum of Problem~\ref{genMLE1}.
Another possible approach is to try some $k<d$, and if $\A^*$ is rank deficient, then again $\M^*=\A^*(\A^*)^\top$ is a global minimum of Problem~\ref{genMLE1}.
However even if we know that the solution for Problem~\ref{genMLE1new} has rank $r< k$, we might find $\A^*$ to be full rank. Indeed, setting $k>r$ does not imply that $\A^*(\A^*)^\top$ is the global minimum of Problem~\ref{genMLE1new}. Moreover, Problem~\ref{genMLE1new} is a generalization of Low-Rank Semi-definite Programming, which is known to be an NP-Hard problem~\citep{Anjos2002} (weighted Max-Cut is a special case of it), which indicates that solving it when $k<d$ might be a difficult task.

\subsection{How Well is $(\mathbf{M}^*, \tau^*)$ Approximated?}
Throughout this section, for simplicity of notation and w.l.o.g we assume that $\zeta F=1$.
In this section, we will see that, with high probability, we can approximate $(\M^*, \tau^*)$ with any given precision if $N$ is large enough. 
Recall that Theorem~\ref{uniqueR} establishes that $R(\M, \tau)$ is uniquely minimized at $(\M^*,\tau^*)$. 
Theorem~\ref{lossapprox} asserts that if $N$ is large enough, the value of the true loss on parameters minimizing the empirical loss, i.e. $R(\hat{\M}, \hat{\tau})$, is close to the minimum of the true loss, $R(\M^*,\tau^*)$. Although we can infer from this theorem that the error at the ground truth parameters  $(\M^*, \tau^*)$ is close to the error at $(\hat{\M}, \hat{\tau})$, it is still possible that $(\M^*, \tau^*)$ and $(\hat{\M}, \hat{\tau})$ are far from each other with respect to the metric 
${\rm d}((\M^*, \tau^*), (\hat{\M}, \hat{\tau})) = \|\M^*- \hat{\M}\|_2 + |\tau^* - \hat{\tau}|.$

%\begin{theorem}\label{M-approx}
%For any given $\varepsilon, \delta>0$, if $N$ is sufficiently large, then with probability at least $1-\delta$, 
%$$d((\hat{\M}, \hat{\tau}), (\M^*, \tau^*))< \varepsilon.$$
%\end{theorem}
%\begin{proof}
%Define $\mathcal{A}_\varepsilon = \{(\M, \tau)\in \mathcal{M}\times [0,B]\colon d((\M, \tau), (\M^*, \tau^*))\geq \varepsilon\}$.
%Note that $\mathcal{A}_\varepsilon$ is compact since it is a closed subset of a compact set.
%As $\gamma(\M,\tau) = R(\M,\tau) - R(\M^*,\tau^*)$ is a positive continues function over $\mathcal{A}_\varepsilon$, it takes its minimum on it, say 
%$$\gamma_\varepsilon = \min_{(\M,\tau)\in \mathcal{A}_\varepsilon} \gamma(\M,\tau)>0.$$
%Now, by Theorem~\ref{lossapprox}, if $N> N(\frac{\gamma_\varepsilon}{2}, \delta)$, then, with probability at least $1-\delta$,
 %$$0< R(\hat{\M}, \hat{\tau}) -  R(\M^*, \tau^*) < \gamma_\varepsilon.$$
 %This means that $(\hat{\M}, \hat{\tau})\not\in \mathcal{A}_\varepsilon$, or equivalently, 
 %$$d((\hat{\M}, \hat{\tau}), (\M^*, \tau^*))< \varepsilon.$$
%\end{proof}
%Although this theorem guarantees that $(\hat{\M}, \hat{\tau})$ can approximate $(\M^*, \tau^*)$ with any precision, the unknown parameter $\gamma_\varepsilon$ can make the sample complexity exponentially large in term of $\varepsilon$. 
%In the rest of this subsection, we explore this issue from another perspective revealing the connection between distribution $f$ and the desired precision. 

Recall that the random variable $\z \in \mathbb{R}^d$ is generated from an unknown distribution with probability density function $f(\z)$ with bounded support. 
Let us define the $L_1(f)$-norm of $(\M,\tau)$ as 
$$\|(\M, \tau)\|_{L_1(f)} = \int f(\x)\left|\x^t\M\x - \tau\right|\d \x.$$
This is a norm  on the vector space $\{(\M,\tau)\colon \M\in \R^{d\times d} \text{ is symmetric }, \tau\in\R\}$. Note that if $\|(\M, \tau)\|_{L_1(f)} = 0$, then Lemma~\ref{M1=M2} along with the fact that $\mu_L\left(\{\x\colon f(\x)>0\}\right)>0$ implies $\M=\0$ and $\tau = 0$. The other required properties follow by standard reductions.  
%Hereafter, for simplicity of notation, we refer to it as $\|(\M, \tau)\|_1$. 
This norm naturally induces the following $L_1(f)$-metric  
\begin{align}\label{l_1_f_metric}
\|(\M_1,\tau_1)-(\M_2,\tau_2)\|_{L_1(f)} & = \int f(\x)\left|(\x^t\M_1\x - \tau_1) - (\x^t\M_2\x - \tau_2)\right|\d \x.
\end{align}
\begin{theorem}\label{parapprox}
% \sout{Let ${\rm Noise}(\eta)$ be a simple noise}
Assume that the noise model ${\rm Noise}(\eta)$ is simple and set $\omega = \min\limits_{|\eta|\leq \beta F} {\rm Noise}(\eta)$; see Table \ref{tbl:simplenoises}. Then for all $(\M,\tau)\in \mathcal{M}\times [0, B]$ 
$$R(\M,\tau) - R(\M^*,\tau^*) \geq  \frac12\omega^2\left(\|(\M,\tau) - (\M^*,\tau^*)\|_{L_1(f)}\right)^2.$$
\end{theorem}
\begin{proof}
Define 
$\mu_{_{(\M,\tau)}}$ to be the measure induced by the probability density function 
$g(\x, \ell;\M, \tau) = f(\z)\Phi_{\rm Noise}(\ell(\|\z\|^2_\M-\tau)).$
Note that 
\begin{align*}
R(\M,\tau) - R(\M^*,\tau^*) 
& = \mathbb{E}_{\z,\ell\sim g(\z, \ell;\M^*, \tau^*)} \left(\log\frac{\Phi_{\rm Noise}(\ell(\|\z\|^2_{\M^*}-\tau^*))}{\Phi_{\rm Noise}(\ell(\|\z\|^2_\M-\tau))}\right)\\
& = \mathbb{E}_{\z,\ell\sim g(\z, \ell;\M^*, \tau^*)} \left(\log\frac{f(\z)\Phi_{\rm Noise}(\ell(\|\z\|^2_{\M^*}-\tau^*))}{f(\z)\Phi_{\rm Noise}(\ell(\|\z\|^2_\M-\tau))}\right)\\
&  = {\rm D_{KL}}(g(\x, \ell;\M^*, \tau^*) \| g(\x, \ell;\M, \tau))\\
& \geq   \frac{1}{2}\left(\left\|\mu_{_{(\M^*,\tau^*)}} - \mu_{_{(\M,\tau)}}\right\|_{\text{TV}}\right)^2
\end{align*}
where $\|\mu_{_{(\M^*,\tau^*)}} - \mu_{_{(\M,\tau)}}\|_{\text{TV}}$   is the total variation of the signed measure $\mu_{_{(\M^*,\tau^*)}} - \mu_{_{(\M,\tau)}}$ and
the last line follows from Pinsker's inequality (see~\cite{10.2307/2985711}). 
So to find a lower bound for $|R(\M,\tau) - R(\M^*,\tau^*)|$, it suffices to find a lower bound for $\|\mu_{_{(\M^*,\tau^*)}} - \mu_{_{(\M,\tau)}}\|_{\text{TV}}$.  
% \avl{Question: are the TV norm and 1-norm the same when the probability space is not countable? Do we have a reference for that?}
% \meysam{
% %What we need here is actually a lower bound. 
% For two distribuations $f$ and $g$, we know 
% \begin{align*}
% \|f-g\|_{\rm TV} & =\sup _{A\in {\mathcal {F}}}\left|P_f(A)-P_g(A)\right|
% \end{align*}
% Define $X^+ = \{x\colon f(x)\geq g(x)\}$ and $X^- = \{x\colon f(x)< g(x)\}$.
% It is clear that 
% $$P_f(X^+)-P_g(X^+) = \int_{X^+} [f(x) - g(x)] \d x = \int_{X^-} [g(x) - f(x)] \d x = P_g(X^-)-P_f(X^-).$$
% On the other hand,
% $$\int|f-g|\d x = P_f(X^+)-P_g(X^+) + P_g(X^-)-P_f(X^-) = 2 (P_f(X^+)-P_g(X^+)) \leq 2\|f-g\|_{\rm TV}.$$
% To prove the other side inequality, note that for each $A\in F$, we have 
% $$P_f(A)-P_g(A) = \int_A (f - g)\d x \leq \int_{\{x\in A\colon f(x)\geq g(x)\}} (f - g)\d x \leq  \int_{X^+}(f - g)\d x.$$
% }
To this end,
\begin{align*}
\|\mu_{_{(\M^*,\tau^*)}} - \mu_{_{(\M,\tau)}}\|_{\text{TV}} 
& = \frac{1}{2}\int f(\z)\left|\Phi_{\rm Noise}\left(\ell(\|\z\|^2_\M-\tau)\right) - \Phi_{\rm Noise}\left(\ell(\|\z\|^2_{\M^*}-\tau^*)\right) \right|\d\z\d\ell\\
\text{\footnotesize(Mean value theorem)}\quad
& = \frac{1}{2}\int f(\z) \left|\Phi_{\rm Noise}'(\xi(\z, \ell))\left[\ell(\|\z\|^2_\M-\tau) - \ell(\|\z\|^2_{\M^*}-\tau^*)\right]\right|\d\z\d\ell\\
\text{\footnotesize($\Phi_{\rm Noise}'(\cdot) = {\rm Noise}(\cdot)$)}\quad
& = \frac{1}{2}\int f(\z) {\rm Noise}(\xi(\z, \ell))\left|(\|\z\|^2_\M-\tau) - (\|\z\|^2_{\M^*}-\tau^*)\right|\d\z\d\ell\\
& \geq \frac{1}{2}\left(\min_{|\xi|\leq \beta {F}} {\rm Noise}(\xi)\right)\times \int f(\z) \left|(\|\z\|^2_\M-\tau) - (\|\z\|^2_{\M^*}-\tau^*)\right|\d\z\d\ell\\
& =  \omega \int f(\z) \left|(\|\z\|^2_\M-\tau) - (\|\z\|^2_{\M^*}-\tau^*)\right|\d\z\\
& =  \omega \|(\M,\tau) - (\M^*,\tau^*)\|_{L_1(f)},
\end{align*}
where the first step follows from Scheffé's Lemma, which relates the total variation distance to the $L_1$-norm (see Lemma 2.1 in~\cite{bookAlexandre2008}), the second step 
is true since Mean Value Theorem implies that there is a 
$\xi(\z, \ell)$ between $\ell(\|\z\|^2_\M-\tau)$ and $\ell(\|\z\|^2_{\M^*}-\tau^*)$ such that 
\begin{align*}
    \Phi_{\rm Noise}'(\xi(\z, \ell))
\Big[\ell(\|\z\|^2_\M-\tau) - & \ell(\|\z\|^2_{\M^*}-\tau^*)\Big]\\
& = \Phi_{\rm Noise}\left(\ell(\|\z\|^2_\M-\tau)\right) - \Phi_{\rm Noise}\left(\ell(\|\z\|^2_{\M^*}-\tau^*)\right),
\end{align*}
and the fourth step is true since $|\xi(\z, \ell)|\leq \beta F$.
% \sout{This is valid since $\xi(\z, \ell)$ is a value between $\ell(\|\z\|^2_\M-\tau)$ and $\ell(\|\z\|^2_{\M^*}-\tau^*)$ and for each $\z$ and $\M$, $|\|\z\|^2_\M-\tau|\leq \beta F$.} 
Therefore,  
\begin{align*}
R(\M,\tau) - R(\M^*,\tau^*) &\geq \frac{1}{2} \omega^2  \|(\M,\tau) - (\M^*,\tau^*)\|_{L_1(f)}^2.
\end{align*}
\end{proof}

% In Table~\ref{m_for_simplenoises}, we summarize some possible values for $m$ for some simple noises. 
% \begin{table}[h]
%     \centering
% \begin{tabular}{p{2.5cm}p{2cm}p{2cm}p{2cm}  }
% \toprule[2pt]
% Noise model	& Logistic: ${\rm L}(\eta|0, 1)$ & Normal: $\N(\eta|0, 1)$ & Laplace: $f(x|0, 1)$ \\ %& Cauchy distribution  $f(x|0, 1)$\\
% \midrule
% $m\geq$	& $\frac{1}{2(1+e^{\beta F})}$	& $\frac{1}{\sqrt{2\pi}}e^{-\frac{\beta^2 F^2}{2}}$ & $e^{-\beta F}$ \\ %&  $\frac{1}{\pi(1+\beta^2 A^2)}$\\
% \bottomrule[2pt]
% \end{tabular}
%    \caption{Bounds for $m = \min\limits_{|\eta|\leq \beta A} {\rm Noise}(\eta)$ for simple noise models.   }
%     \label{m_for_simplenoises}
% \end{table} 

The next corollary is an immediate consequence of Theorems ~\ref{lossapprox} and ~\ref{parapprox}. It indicates that, with high probability, $(\hat{\M},\hat{\tau})$ can
approximate $(\M^*,\tau^*)$ with any given precision with respect to the $L_1(f)$-metric  defined in~\ref{l_1_f_metric} provided 
that $N$ is large enough with respect to that precision. 
\begin{corollary}\label{corl1upperbound}
Assume that the noise model ${\rm Noise}(\eta)$ is simple. 
For $\varepsilon, \delta>0$, if $N>N\left(\frac12\varepsilon^2 \omega^2,\delta\right)$, then with probability at least $1-\delta$
$$\|(\hat{\M},\hat{\tau}) - (\M^*,\tau^*)\|_{L_1(f)}\leq \varepsilon.$$
\end{corollary}
The $L_1(f)$-metric is dependent on the distribution $f(\z)$, which is unavoidable. The following lemma bounds this norm under a general condition on $f$, and thus provides a more intuitive result (see Appendix~\ref{app:l_1_norm_spectral_norm} for the proof).

\begin{lemma}\label{uniformlowerL_1_f}
If $f(\z)\geq c>0$ for each $\z\in B^d(1)=\{\x\in\R^d\colon \|z\|_2\leq 1 \}$, then for all $(\M,\tau)\in \mathcal{M}\times [0, B]$ 
$$\|(\M, \tau) - (\M^*,\tau^*)\|_{L_1(f)}\geq  \frac{c\pi^{d/2}}{20\Gamma(d/2+1)} \left(\frac{1}{18}\right)^d \d((\M, \tau), (\M^*,\tau^*)).$$
In particular, if $f(\z)$ is uniform on the unit ball $B^d(1)$, then 
$$\|(\M, \tau) - (\M^*,\tau^*)\|_{L_1(f)}\geq  \frac{1}{20} \left(\frac{1}{18}\right)^d \d((\M,\tau), (\M^*,\tau^*)).$$
\end{lemma}

Combining Theorem~\ref{parapprox} and Lemma~\ref{uniformlowerL_1_f}, we obtain the following result.

\begin{theorem}\label{l_1_f_lowerbound}
For simplicity of notation, set 
$C(d) = \frac{c\pi^{d/2}}{\Gamma(d/2+1)}$. 
Let ${\rm Noise}(\eta)$ be a simple noise model and set $\omega = \min\limits_{|\eta|\leq \beta A} {\rm Noise}(\eta)$. 
If $f(\z)\geq c>0$ for each $\z\in B^d(1)$, then  
$$R(\M, \tau) - R(\M^*,\tau^*) \geq  \frac{\omega^2C(d)^2}{800\times 18^{2d}} \d^2\left((\M, \tau), (\M^*,\tau^*)\right).$$
\end{theorem}
In particular, if $f(\z)$ is uniform on $B^d(1)$, then $C(d) = 1$ and thus
$$R(\M,\tau) - R(\M^*,\tau^*) \geq \frac{\omega^2}{800\times 18^{2d}} \d^2((\M, \tau), (\M^*,\tau^*)).$$
Now, combining Theorems~\ref{lossapprox} and~\ref{l_1_f_lowerbound}, we obtain the following result.
\begin{theorem}\label{Mtauapprox}
Let ${\rm Noise}(\eta)$ be a simple noise model and set $\omega = \min\limits_{|\eta|\leq \beta A} {\rm Noise}(\eta)$. 
Assume $f(\z)\geq c>0$ for each $\z\in B^d(1)$. 
For any given $\varepsilon, \delta >0$, if $N > N(\frac{\omega^2\varepsilon^2 C^2(d)}{800\times  18^{2d}}, \delta)$, then with probability at least $1-\delta$, any point $(\hat{\M}, \hat{\tau})$ minimizing $R_N(\M, \tau)$ satisfies 
$$\d((\hat{\M},\hat{\tau}), (\M^*,\tau^*)) < \varepsilon.$$
\end{theorem}
We remark that we have not attempted to optimize the constants which appear in the sample complexity bound in Theorem \ref{Mtauapprox}.

\subsection{ Rank-Deficient Case}\label{SecRank-deficient}

No work prior to the present paper has proved that we can recover the matrix $\M^*$ when it is not full rank, %\avl{(really? or is this statement too strong?)}\jeff{I think this is indeed true.  searching over the low-rank setting is non-convex so much harder.  We get to this via a backdoor of parameter recovery (which I don't think others can do either), and argue if it is near low-rank we can round parameters.  This is what I am most excited about in this work!}
 or bounds the effect of truncating the derived $\hat \M$ to a low-rank $\hat \M_k$.
Theorem~\ref{Mtauapprox} indicates that for any given $\varepsilon>0$, $\|\hat{\M}-\M^*\|_2 < \varepsilon$ and $|\hat{\tau}-\tau^*|< \eps$ if $N$ is large enough.    
This will guarantee that there will be a small eigenvalue of $\hat{\M}$ for every small eigenvalue of $\M^*$.
\begin{lemma}\label{lem:singular_bound}
Let $\M_1,\M_2$ be two given $n\times d$ matrices. 
If $\|\M_1-\M_2\|_2 < \varepsilon$, then $|\sigma_i(\M_1) - \sigma _{i}(\M_2)| < \varepsilon$, where 
$\sigma_i(\M_j)$ is the $i$-th singular value of matrix $\M_j$ for $j=1,2$.
\end{lemma}

\begin{proof}
Set $\M_r = \M_1-\M_2$. Because of the definition of the spectral norm, 
$$\max_{\x\neq\mathbf{0}}\frac{\|\M_r\x\|_2}{\|\x\|_2}< \varepsilon.$$
Write $\M_2 = \M_1 - \M_r$ and for each $i\in[d]$, notice 
\begin{align*}
\sigma _{i}(\M_2 )&=\min _{\dim(W)=n-i+1}\max _{\underset {\|\x\|_{2}=1}{x\in W}}\left\|(\M_1 - \M_r)\x\right\|_{2}\\
&\leq \min _{\dim(W)=n-i+1}\max _{\underset {\|\x\|_{2}=1}{x\in W}}\left(\|\M_1\x\|_{2} + \|\M_r\x\|_2\right)\\
& \leq \min _{\dim(W)=n-i+1}\max _{\underset {\|\x\|_{2}=1}{x\in W}}\left(\|\M_1\x\|_{2} + \varepsilon\right)\\
& = \varepsilon + \min _{\dim(W)=n-i+1}\max _{\underset {\|\x\|_{2}=1}{x\in W}}\|\M_1\x\|_{2}\\
& = \sigma_i(\M_1) + \varepsilon.
\end{align*} 
With a similar approach, we can prove that 
$\sigma_i(\M_1) \leq \sigma _{i}(\M_2) + \varepsilon$,
which implies that 
$$|\sigma_i(\M_1) - \sigma _{i}(\M_2)| < \varepsilon\quad\quad\quad \forall i\in[d],$$
which completes the proof.
\end{proof}
Using this lemma, we obtain the following lemma 
\begin{lemma}\label{lem:eig_perturb}
Let $\alpha_1\geq \alpha_2>0$ be given and assume the ground truth $\M^*$ has $k$ eigenvalues greater than or equal to $\alpha_1$
and $k'$ eigenvalues less than or equal to $\alpha_2-3\varepsilon$   with $k+k'\leq d$. 
If the noise model ${\rm Noise}(\eta)$ is simple and  $N > N(\frac{\omega^2\varepsilon^2 C^2(d)}{800\times  18^{2d}}, \delta)$, then, with probability at least $1-\delta$, 
$\|\hat{\M}-\M^*\|_2<\varepsilon$, 
the number of eigenvalues of $\hat{\M}$ which are greater than $\alpha_1-\varepsilon$ is at least $k$, and 
the number of eigenvalues of $\hat{\M}$ which are less than $\alpha_2-2\varepsilon$ is at least $k'$. 
\end{lemma}
\begin{proof}
   Assume that $(\hat{\M}, \hat{\tau})$ minimizes $R_N(\M, \tau)$. 
By Theorem~\ref{Mtauapprox}, with probability at least $1-\delta$, we have $\|\hat{\M}-\M^*\|_2 < \varepsilon$. 
Now, Lemma~\ref{lem:singular_bound} implies the result. 
\end{proof}

The next theorem gives us a better understanding of the eigenvalue perturbation of $\hat{\M}$ when $\M^*$ is not full rank. 
\begin{theorem}\label{thm:DR-main1}
Assume that the noise model ${\rm Noise}(\eta)$ is simple and $\M^*$ has rank $0< r < d$.
For a given $\varepsilon,\delta >0$, if $0< 3\varepsilon< \sigma_{r}(\M^*)$ 
and $N > N(\frac{\omega^2\varepsilon^2 C^2(d)}{800\times 18^{2d}}, \delta)$, then, with probability at least $1-\delta$, $\hat{\M}$ has exactly $d-r$ eigenvalues less than $\varepsilon$ and the rest $r$ eigenvalues are at least $\frac{2}{3}\sigma_{r}(\M^*)$. So if we truncate the eigenvalues of $\hat{\M}$ which are less than $\varepsilon$ to zero, we obtain $\hat{\M}_k$ of rank $r$ for which $\|\M^* -\hat{\M}_r\|_2 < 2\varepsilon$.
\end{theorem}
\begin{proof}
    % As Lemma~\ref{lem:eig_perturb} is a consequence of Theorem~\ref{Mtauapprox}, with probability at least $1-\delta$, 
    % $\|\hat{\M}-\M^*\|_2 < \varepsilon$ and also 
    As Lemma~\ref{lem:eig_perturb} holds with $\alpha_1 = \sigma_{r}(\M^*),\alpha_2 = 3\varepsilon,k = r, k'= n-r$, we have 
    $\|\hat{\M}-\M^*\|_2 < \varepsilon$, 
    the number of eigenvalues of $\hat{\M}$ which are greater than $\alpha_1-\varepsilon > \frac{2}{3}\sigma_{r}(\M^*)$ is $r$, and the number of eigenvalues of $\hat{\M}$ which are less than $\alpha_2-2\varepsilon = \varepsilon$ is $n-r$. Now,
    $$\|\M^* -\hat{\M}_r\|_2 \leq \|\M^* -\hat{\M}\|_2 + \|\hat{\M} -\hat{\M}_r\|_2\leq 2\varepsilon$$
    completes the proof.
\end{proof}

As we are not given $\M^*$, in practice we are unaware of its rank or spectral properties. 
%\sout{such a gap in the eigenvalues of $\M^*$}. 
Since we only have access to $\hat{\M}$, 
the next theorem establishes that the loss function still converges under eigenvalue truncation when the corresponding eigenvalues of $\hat{\M}$ are small. 
%is a useful special case of Theorem \ref{thm:DR-main1} in terms of eigenvalues of $\hat{\M}$ instead of $\M^*$.

\begin{theorem} \label{thm:DR-main2}
    Assume that the noise model ${\rm Noise}(\eta)$ is simple. 
    For a given $\varepsilon,\delta >0$, assume that $(\hat{\M},\hat{\tau})$ minimizes $R_N(\M,\tau)$ for $N>N\left(\varepsilon/2,\delta\right)$.
    Define threshold $\gamma$ such that $\hat{\M}$ has $d - k$ eigenvalues which are less than $\gamma$. Let $\hat{\M}_k$ be the rank $k$ matrix obtained from $\hat{\M}$ by setting these $d-k$ eigenvalues to zero. Then, with probability at least $1-\delta$, 
    $$0< R(\hat{\M}_k, \hat{\tau}) -  R(\M^*, \tau^*) < \gamma + \varepsilon.$$
\end{theorem}
\begin{proof}
As $\|\hat{\M}-\hat{\M}_k\|_2 < \gamma$, using the proof of the  first inequality in Lemma~\ref{upperR1}, we obtain 
$$|R(\hat{\M},\hat{\tau}) - R(\hat{\M}_k,\hat{\tau})| < \zeta F\gamma = \gamma,$$
since we have assumed $\zeta F=1$.  On the other hand, by Theorem~\ref{lossapprox}, we have
$$0< R(\hat{\M}, \hat{\tau}) -  R(\M^*, \tau^*) < \varepsilon.$$
Combining these two inequalities implies the desired inequality.
\end{proof}
Combining Theorems~\ref{l_1_f_lowerbound} and ~\ref{thm:DR-main2}, we also have 
$\d\left((\hat{\M}_k,\hat{\tau}), (\M^*,\tau^*)\right)< \frac{20\sqrt{2}\times 18^{d}}{\omega C(d)}( \gamma + \varepsilon).$
Thus, if $\gamma\leq \varepsilon\frac{\omega C(d)}{20\sqrt{2}\times 18^{d}}$, then 
$\d\left((\hat{\M}_k,\hat{\tau}), (\M^*,\tau^*)\right)< 2\varepsilon,$ implying $\|\hat{\M}_k-\M^*\|_2 < 2\varepsilon$. Thus spectral truncation of the empirical $\hat{\M}$ yields a good approximation of $\M^*$.

\subsection{Invariance to Changes in Unit of Input}
Clearly, the learned $\hat{\M}$ and $\hat{\tau}$ are dependent on the units of feature space.
So, as an interesting question, we can study the behavior of  $\hat{\M}$ and $\hat{\tau}$ if we change the units in the original feature space. 
Mainly, we want to prove that if we change the units in feature space, we do not need to solve a 
new optimization problem to learn a new $\hat{\M}$ and $\hat{\tau}$. Instead, we can recover these parameters from the already solved optimization problem.
Assume that we have a non-singular matrix $\U_{d\times d}$ which changes the units and rotates the space of features, and 
let $\z'_i = \U\z_i$ and $\ell_i' = \ell_i$. We 
want to solve the following optimization problem 
\begin{align}\label{genMLE1unitchange}
\min_{\M\succeq 0,\tau\geq0} R'_N(\M,\tau) \, ,
\end{align}
where 
\begin{align*}
R'_N(\M,\tau) & = -\frac 1N\sum_{i=1}^N \log \sigma\left(\ell'_i(\|\z'_i\|^2_{\M} - \tau)\right)\\
& =  -\frac 1N\sum_{i=1}^N \log \sigma\left(\ell_i(\|\z_i\|^2_{\U^\top\M\U} - \tau)\right)
\end{align*}
Since $\U$ is non-singular, Optimization Problem~\ref{genMLE1} is minimized at $\hat{\M}, \hat{\tau}$ if and only if Optimization Problem~\ref{genMLE1unitchange}
is minimized at $\hat{\M}' = {\U^{-1}}^\top \hat\M \U^{-1}, \hat{\tau}' = \hat\tau$.
Hence, the solution is invariant to the choice of units, given knowledge of the conversion.

\section{Experimental Results}\label{ExperimentalResults}
In Section~\ref{sec:conv}, we described the optimal loss functions for four different noise distributions (Logistic, Normal, Laplace, and Hyperbolic Secant).  As the Normal noise model ends up with a probit in the loss function, and the probit function has no closed-form formula, we will not use this model in the experiments. Also, since in practice we generally do not know the noise distribution, we evaluate the performance of each model under a variety of possible noise distributions, thus testing the robustness of each model to misspecification of the noise. 
We investigate how the resulting accuracy depends on the sample complexity, amount of noise, and noise misspecification. 
% \sout{We study accuracy for different amounts of noise, sample complexity, and robustness against the amount of noise.} 

We start with synthetic data, described in Section \ref{data_generation}, so we can run precisely controlled experiments which are reported in Sections \ref{Logisticsubsection} and \ref{Ablation}.  
In Section~\ref{Comp-DML-eig}, we compare our model performance with DML-eig (\cite{ying2012distance}).
Then in Section \ref{sec:exp-real} we apply our methods to some real data experiments well suited to our proposed algorithm. 
% \sout{We presented experiments in a manner that facilitates their replication, to make it easy to reproduce} 
All experimental results are reproducible; see the GitHub repository by~\cite{OurGithubRep} containing data and source codes. 

\subsection{Data Generation}\label{data_generation} 
We start with $d$ random positive real values $\lambda_1,\ldots,\lambda_d$ and then we randomly generate a $d\times d$ covariance  matrix $\Sigma$ whose eigenvalues are $\lambda_1,\ldots,\lambda_d$.  
To this end, we first randomly and uniformly generate an orthonormal matrix $\U_{d\times d}$ and then set $\Sigma = \U \D \U^\top$, where $\D$ is the $d\times d$ diagonal matrix whose diagonal entries are $\lambda_1,\ldots,\lambda_d$.
We then independently sample $2N$ points $\x_1, \y_1,\ldots, \x_N, \y_N$ from $\N(\boldsymbol{0},\Sigma)$ to generate $N$ pairs $(\x_i, \y_i)$ for $i=1\ldots, N$.
Next we select $d$ nonnegative random real values $\gamma_1,\gamma_2,\ldots,\gamma_d\geq 0$ as the eigenvalues of the ground truth $\M^*$,
and randomly generate $\M^*$ to be a random positive semi-definite matrix with eigenvalues 
$\gamma_1,\ldots,\gamma_d$, as we did for $\Sigma$. 
We have
$$\mathbb{E}\left(\|\x-\y\|_{M^*}^2\right) = 2{\rm tr}\left(\Sigma M^*\right)$$ 
provided that $\x,\y\sim h(\x)$, 
where $h(\x)$ is a pdf for which ${\rm Cov}_{X\sim h}(X) = \Sigma$.
We now choose $\tau^*>0$ not far from $\mathbb{E}\left(\|\x-\y\|_{M^*}^2\right)$ so that we obtain a sufficient number of pairs $(\x,\y)$ labeled as both Close and Far. More specifically, in Sections \ref{Logisticsubsection}-\ref{Comp-DML-eig}, we consider the following setting. 
%We only change the choice of noise ${\rm Noise}(0,s)$ in different subsections. 
\begin{itemize}
\item We assume that the rank of $\M^*_{10\times 10}$ is $5$ and randomly and uniformly generate $5$ nonzero eigenvalues from $[0,1]$.
With two-digit precision, we obtain $0.32, 0.89, 0.59, 0.13, 0.14$ as the $5$ nonzero eigenvalues of $\M^*$.
%Now, we randomly generate a $10\times 10$ orthonormal random matrix $\U$ and set $\M^*=\U\D\U^t$, where the first $5$ entries on the diagonal of $D$ are the same as $0.32, 0.89, 0.59, 0.13, 0.14$, and the rest are zero.

\item %To generate the covariance matrix $\Sigma$, we follow a similar approach. 
We randomly and uniformly select $10$ nonzero numbers from $(0,1]$ as the eigenvalues of $\Sigma$. %Here, we have again set the random seed to $2023$.
With two digit precision, we obtain 
$$0.73,\ 0.7,\  0.68,\ 0.59,\  0.47,\ 0.45,\ 0.21,\ 0.19,\ 0.11,\ 0.04$$ as the eigenvalues of $\Sigma$.
%To generate $\Sigma$, we randomly generate a $10\times 10$ orthonormal random matrix $\U$ and set $\Sigma=\U\D\U^t$, 
%where the entries on the diagonal of $D$ respectively are again the same at $0.73,\ 0.7,\  0.68,\ 0.59,\  0.47,\ 0.45,\ 0.21,\ 0.19,\ 0.11,\ 0.04$.

\item As we are dealing with fixed random seeds, we obtain $\mathbb{E}\left(\|\x-\y\|_{M^*}^2\right) \approx 1.7$.

\item To obtain roughly balanced data, we set $\tau^* = 1.3$ and generate $20000$ data points. We split the data into $15000$ training and $5000$ test points.

\end{itemize}

We now describe the label generation. Note that in the theoretical formulation of the problem, we assume that the noisy labeling process depends on $\|x-y\|_{\M^*}^2$.
However one could also assume that the noise changes labels directly, independently of $\|x-y\|_{\M^*}^2$. We thus study both of the following settings empirically. 
\begin{itemize}
\item {\bf Noise affects the labeling through $\|x-y\|_{\M^*}^2$ (\ref{assumption0}).}\\
We consider a noise distribution ${\rm Noise}(0, s)$ with zero mean and scale parameter $s$ from the Logistic, Gaussian, Laplace, and Hyperbolic Secant distributions.
We then generate 
$\eta_1,\ldots, \eta_N\sim {\rm Noise}(0, s)$. 
For each pair $(\x_i, \y_i)$, we set $\ell_i = 1$ (``Far'') if $\|\x_i-\y_i\|_{\M^*}^2 + \eta_i \geq \tau^*$, and we set $\ell_i = -1$ (``Close'') if  
$\|\x_i-\y_i\|_{\M^*}^2 + \eta_i < \tau^*$.
We save these labels as $D_{\text{noisy}}$.
We also save the non-noisy labels to check the model's robustness against noise.
However, we do not use these labels during training. 
Indeed, for each pair $(\x_i, \y_i)$, we set $\ell^*_i = 1$ if $\|\x_i-\y_i\|_{\M^*}^2 \geq \tau^*$  and we set $\ell^*_i = -1$ if  
$\|\x_i-\y_i\|_{\M^*}^2 < \tau^*$.
We save these labels as $D^*$.

\item {\bf Noise directly affects the labeling (Noisy Labeling).}\\
%This setting is different from the previous four ones in nature. 
Here we assume that the noise affects the labels directly by randomly flipping them. 
We first generate $D^*$ as described in the previous paragraph.
%As before, for each pair $(\x_i, \y_i)$, if $\|\x_i-\y_i\|_{\M^*}^2 \geq \tau^*$, we set $\ell^*_i = 1$ (``Far'') and if  
%$\|\x_i-\y_i\|_{\M^*}^2 < \tau^*$, we set $\ell^*_i = -1$ (``Close''). We save these labels as $D^*$.
Then for each $i = 1,\ldots, N$, we flip a coin whose head chance is $p$. If the coin is tails, we set $\ell_i = \ell_i^*$; otherwise, we set 
$\ell_i \in \{-1,1\}$ randomly with the same chance. We save these labels as $D_{\text{noisy}}$.
In expectation, $p/2$ fraction of the labels are mislabeled in $D_{\text{noisy}}$. Although the amount of noise is the same as in the previous setting, i.e. the same number of mistakes are made, this regime is more challenging because in the first case the majority of mistakes occur close to the boundary, while the noisy labeling case results in ``big" mistakes. We thus  
%as the setting is entirely different from the setting we theoretically analyzed, we 
expect performance to be worse. 
%\sout{than all four former settings.} 
\end{itemize}

As a default, in both settings we set the noise parameter so that $10\%$ of the points are mislabeled.

\subsection{\bf Logistic Model with Different Noises}\label{Logisticsubsection}
Recall the Logistic distribution has density function $${\rm L}(x|\mu, s) = \frac{1}{4s}\mathrm{sech}^2\left(\frac{x-\mu}{2s}\right).$$
In Subsection~\ref{Logisticnoisemodel}, we saw that if the noise comes from a Logistic distribution, then 
$R_N(\M, \tau) = -\frac{1}{N}\sum_{i=1}^N \log \sigma\left(\ell_i(\|\z_i\|_{\M}^2 - \tau)\right)$ serves as an optimal proxy for our objective. 
In this section, we generate labels with different noise types including noisy labeling.  
%Note that, in noisy labeling, the noise affects the labels only. 
We set the corresponding noise parameter so that the number of mistakes is roughly $10\%$, and then investigate how the Logistic loss function performs on all these types of noise.

We solve Optimization Problem~(\ref{genMLE1AAt}) using gradient descent and setting 
$learning\_reate = 0.5$, $d = k = 10$, $number\_of\_iterations = 30000$, and $learning\_decay = .95$.  We did this $20$ times for independent sample observations and summarized the average accuracy in Table~\ref{accuracies}. 
Note that the model uses only the noisy labels during training; the non-noisy labels are only used to evaluate the model. 

\begin{table}[h]
    \centering
\begin{tabular}{p{4cm}p{1.5cm}p{1.5cm}p{1.5cm}p{1.5cm}p{2.5cm} }
%\hline
%\multicolumn{3}{|c|}{Accuracy} \\
\toprule[2pt]
Noise type: 				& Logistic   	    & Gaussian       	    & Laplace             & HS                 & Noisy Labeling  \\
\midrule
train accuracy w/ noise 		&90.07\% 		    &89.50\% 	             &87.45\%		&85.48\%		&85.73\% \\
train accuracy w/o noise 		&98.88\%   	    &98.81\% 	 	     &98.62\% 		&98.53\%		&94.68\%	\\  
test accuracy w/ noise 		&89.85\% 		    &89.47\%	             &87.33\% 		&85.28\%		&85.57\%	\\ 
test accuracy w/o noise    		&98.83\% 		    &98.79\% 	 	    &98.57\% 		&98.47\%		&94.51\% \\
\bottomrule[2pt]
\end{tabular}
   \caption{Logistic model accuracy with different noise types (average over 20 trails).}
    \label{accuracies}
\end{table}
  
The Logistic model learned the labeling function very well. With Logistic noise (first column), it reaches about $90\%$ accuracy on noisy labels (as high as possible with 10\% misclassification), and almost $99\%$ accuracy with respect to the ground truth labels.  This holds on both the training and test data sets, which indicates that the model is not overfitting. We also observe that as the noise becomes more and more different from the Logistic model (Gaussian then Laplace then HS then Noisy Labeling), 
%the loss function is derived for  
the accuracy gets worse.  This holds for both the noisy and ground truth labels, and on both the training and test data sets.  The deterioration is most prominent in the ``noisy labeling'' setting, where about $5\%$ accuracy is lost in comparison with the Logistic noise.  

%More precisely, as the Logistic model theoretically aligns with the optimal model for the Logistic noise, we expect the best performance in this setting, which is indeed observed by the experimental results (first column). 
%Also, in this setting, even for the labels with noise, we have the best accuracy since even the noise can be learned by the model. However, as the noise gets more different from the Logistic noise, we can see that the accuracy for noisy labels drops but for the non-noisy labels, it stays at the same level. Note that the Noisy labeling setting is completely different from the other four noises, which justifies why we observed the worst performance in this setting. 
Next, supporting Theorem~\ref{Mtauapprox}, we summarize the recovery of the model parameters $\M^* / \tau^*$ in Table~\ref{precisions}; we report the average recovery error based on $20$ independent sample observations.  We observe that the error is fairly small with a relative error of about $0.07$ for most noise types, but also that the error increases as the misspecification of the noise type increases.
%gets further from the derived loss function.  
For instance noisy labeling achieves only about $0.2$-relative error in a Frobenius or spectral sense.   

\begin{table}[h]
    \centering
\begin{tabular}{p{3.5cm}p{1.5cm}p{1.5cm}p{1.5cm} p{1.5cm} p{2.5cm}  }
%\hline
%\multicolumn{3}{|c|}{Accuracy} \\
\toprule[2pt]
Noise type	& Logistic  & Gaussian  & Laplace  &  HS  & Noisy Labeling  \\
\midrule
$\left\|\frac{\hat{M}}{\hat{\tau}} - \frac{M^*}{\tau^*}\right\|_2 / \left\|\frac{M^*}{\tau^*}\right\|_2$ 	& 0.068 	&0.071 	&0.074	&0.086 	&0.231 \\ 
%\hline
$\|\frac{\hat{M}}{\hat{\tau}} - \frac{M^*}{\tau^*}\|_F / \|\frac{M^*}{\tau^*}\|_F$				& 0.070	&0.068  	&0.080 	&0.088	&0.214 \\
\bottomrule[2pt]
%$ \frac{\left|\hat{\tau} - \frac{\tau^*}{s}\right|}{\frac{\tau^*}{s}}$ 						&0.023 		&?	& ? \\ 
%\hline
\end{tabular}
   \caption{Precisions for Logistic model with different noise types (averaged over 20 trials).}
    \label{precisions}
\end{table}

We plot the eigenvalues of $\M^*/\tau^*$ and $\hat{\M}/\hat{\tau}$ in Figure~\ref{fig:eigens}.  The large left figure shows the eigenvalue recovery by the Logistic, Laplace, and HS loss functions when the labels are generated from Logistic noise.  All do about the same, and capture all eigenvalues fairly well.   Four other plots are shown with other types of noise with similar results; the main exception being with noisy labeling, the top eigenvalue is predicted as much smaller than the true value. This experiment illustrates that although we focus on the Logistic loss function, performance is robust with respect to misspecification of the noise model.

\begin{figure}[h]     
    \includegraphics[width=0.49\textwidth]{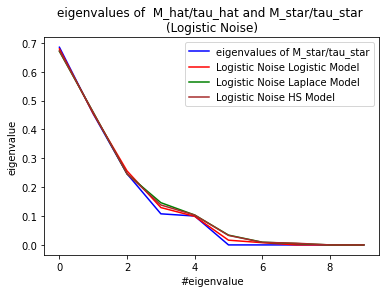}
\begin{minipage}[b]{0.5\linewidth}
   \includegraphics[width=0.49\textwidth]{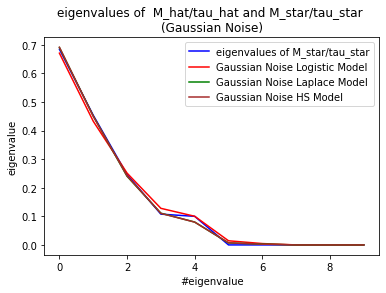}
    \includegraphics[width=0.49\textwidth]{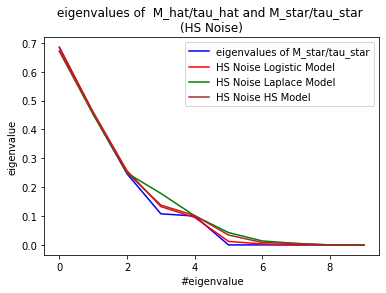} \\
    \includegraphics[width=0.49\textwidth]{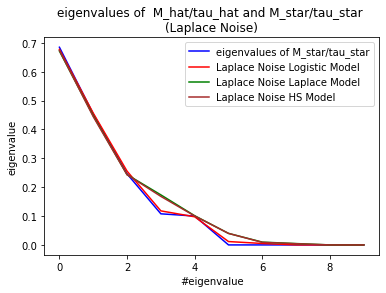}
    \includegraphics[width=0.49\textwidth]{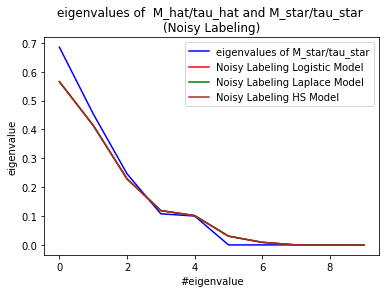}
    \end{minipage}
    \caption{\label{fig:eigens}Comparing eigenvalues of  $\hat{\M}/\hat{\tau}$ and $\M^*/\tau^*$.}
\end{figure}

% \begin{figure}[h]
%     \centering
%     \includegraphics[width=0.42\textwidth]{Images/Logistic_Logistic_mean_eigen.png}
%    \includegraphics[width=0.42\textwidth]{Images/gaussian_Logistic_mean_eigen.png}\\
%     \includegraphics[width=0.42\textwidth]{Images/hs_Logistic_mean_eigen.png}
%     \includegraphics[width=0.42\textwidth]{Images/laplace_Logistic_mean_eigen.png}\\
%     \includegraphics[width=0.42\textwidth]{Images/noisy_labeling_Logistic_mean_eigen.png}
%     \caption{Comparing average eigenvalues of  $\hat{\M}/\hat{\tau}$ and $\M^*/\tau^*$ for Logistic noise and different models with 20 trials.}
%     \label{eigens}
% \end{figure}

In Figure~\ref{accuracyepochs}, we summarize the accuracy of the Logistic model for different noises as the number of iterations increases; the Logistic noise plot is highlighted on the left.  Each plot shows the progression of training on the train and test accuracy.  As before, there is little difference between test and train accuracy. The accuracy with the noise-induced labels plateaus near $90\%$ which is as good as expected with $10\%$ noise.  And the accuracy on the ground truth labels continues to increase (to about $99\%$) as training continues.  The results are similar for other noise types, with convergence to lower plateaus, as expected.  
\begin{figure}[h]
    \includegraphics[width=0.49\textwidth]{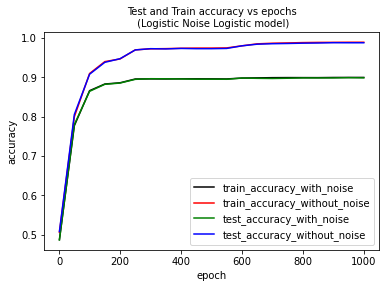}
    \begin{minipage}[b]{0.5\textwidth}
    \includegraphics[width=0.49\textwidth]{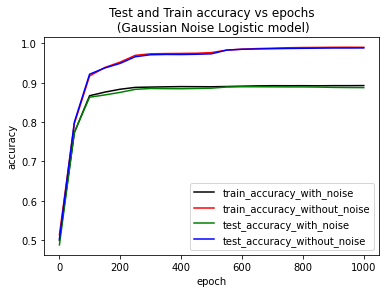}
    \includegraphics[width=0.49\textwidth]{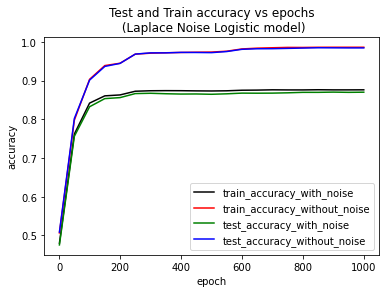}\\
    \includegraphics[width=0.49\textwidth]{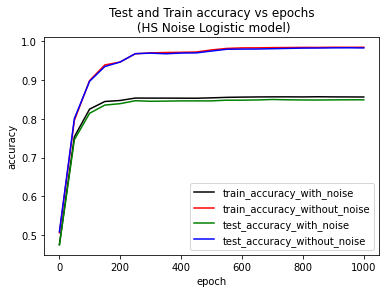}
    \includegraphics[width=0.49\textwidth]{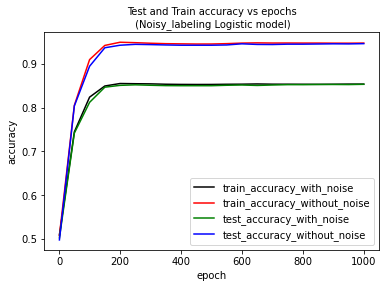}
    \end{minipage}
    \caption{Logistic model's accuracy VS epochs: Logistic noise, Gaussian noise, Laplace noise, HS noise, and Noisy labeling.}
    \label{accuracyepochs}
\end{figure}

% {\blue 
% \jeff{This does not do much for me, I could move this figure and discussion in the appendix if you want to keep it:} \meysam{As you prefer. } \avl{I agree.}
% In Figure~\ref{lossL1normhist}, as the iteration increases, we compare the values of the loss function $R_N$ 
% on $(\hat{\M}, \hat{\tau})$, compared with  $(\M^*/s, \tau^*/s)$ which we do not expect to surpass.  Observe that the loss on $(\M^*/s, \tau^*/s)$ is a constant red line at the bottom at around $0.23$. When considering Logistic noise (blue) we reach this loss around 700 iterations, and nearly due when considering Gaussian noise.  For other types of noise, the method does worse; Noisy labeling only achieves a loss value around $0.5$.  
% %In the right figure, {\blue we record the $L_1(f)$-distance between $(\frac{\hat{\M}}{\hat{\tau}},1)$ and $(\frac{\M^*}{\tau^*}, 1)$ as the iteration number increases.} \avl{Should define what we mean by this, since the $L_1(f)$ norm was defined on pairs $(M,\tau)$}

% \begin{figure}[h]
%     \centering
%     \includegraphics[width=0.5\textwidth]{Images/Logistic_model_loss_history_all_noises_vs_iteration.png}
%     %\includegraphics[width=0.45\textwidth]{Codes/Binary(d=10)/Images/norm_history_vs_iteration.png}
%     \caption{Loss history for Logistic model and different noises}
%     \label{lossL1normhist}
% \end{figure}
% }

To study the sample complexity behavior, we gradually increase the number of training samples and record the accuracy for each case in Figure \ref{Logistic2}.  Again the left figure illustrates the Logistic loss on labels generated with Logistic noise, plotting results for training and test data, with respect to the noisy and ground truth labels.  
Note that when the number of training points is too small (100 or less), the training accuracy is $1$ while the test accuracy is low; this indicates overfitting. 
However, when the number of training points increases to around $1000$, the overfitting problem vanishes, and the training and test accuracy start to align closely.  Between $5000$ and $10{,}000$ samples, they become indistinguishable in the plots.   Similar results hold for the other noise types, again with somewhat lower overall accuracy depending on how close the noise model is to the Logistic noise.  
\begin{figure}[h]
    \includegraphics[width=0.49\textwidth]{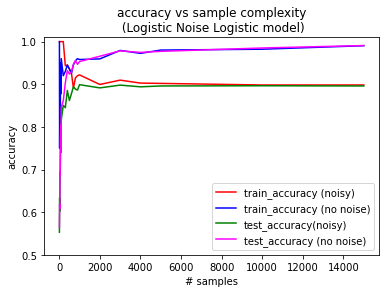}
    \begin{minipage}[b]{.5\textwidth}
    \includegraphics[width=0.49\textwidth]{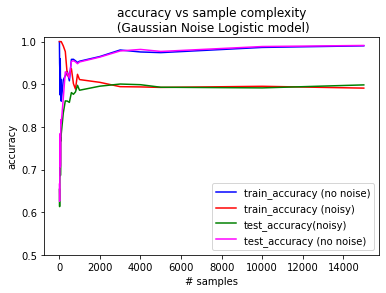}
    \includegraphics[width=0.49\textwidth]{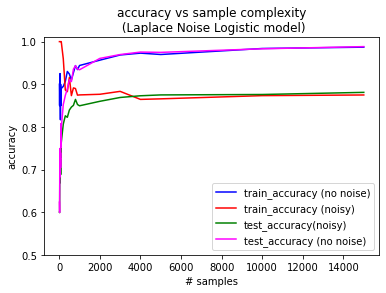}\\
    \includegraphics[width=0.49\textwidth]{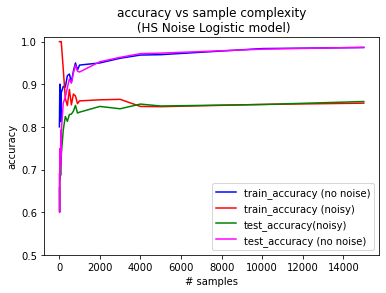}
    \includegraphics[width=0.49\textwidth]{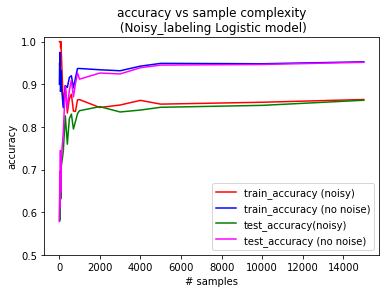}
    \end{minipage}
    \caption{Logistic model with different noises: accuracy VS sample complexity.}
    \label{Logistic2}
\end{figure}

\subsection{Further Ablation Study}\label{Ablation}
In the following two parts, we experimentally explore the robustness and the sample complexity of the Logistic model when faced with high Logistic noise.

\subsubsection{\bf How Much Noise can Break the Model}
Theoretically, we proved that the model can recover the ground truth parameters even if the labeling is noisy (under some assumptions). 
This result is also supported by the experimental evidence in the previous section when the noise causes $10\%$ mislabeling. 
In this section we increase the effect of the noise and check the model resistance. 
For a fixed number of training samples ($18000$ here), we increase the noise variance gradually and log the accuracy of the Logistic loss function when the noise also comes from a Logistic distribution.  In  Figure~\ref{fig:Noiseaccuracytest}, the $x$ axis shows the fraction of points that are mislabeled, which depends directly on the variance of the noise distribution, 
and the $y$ axis indicates accuracy. We generally observe that the model ignores the noise and recovers true labeling even when the noise is high. 
We can see that the train and test accuracies of the model for the noisy labels are aligned with the line $y= 1-x$, which is as expected. 
It indeed indicates that with $x$ amount of noise, the model cannot have better accuracy than $1-x$ on the noisy labels.  
However, for the ground truth labeling (described in the legend as ``no noise"), we observe that the model is pretty robust against noise, even when the amount of noise is pretty high.  For instance, for around $40\%$ mislabeling, we have around $95\%$ of accuracy for unseen data.  However, when the noise perturbs $45\%$ of the labels, it starts to collapse.  
When the noise disturbs $50\%$ of the labels, we might assume that random guessing would achieve the best accuracy,
but the model still achieves around $65\%$ accuracy for train and test points with respect to the ground truth labels. 
%This means that the model can ignore the noise and even in this setting, if we had more enough data, the model can achieve a better accuracy. 
Even though we have 50\% mislabeling, the ``extreme" examples are correct, so there is more information than purely random labels.
 %\avl{Not sure I agree with this interpretation: even though we have 50\% mislabeling, the ``extreme" examples are correct, so it is more information than purely random labels.} \meysam{Meysam: The purpose of this statement is exactly what you said. We can add it to the text for more .} 
 We will study this setting in the next paragraph.     

\begin{figure}[h]
    \centering
    \includegraphics[width=0.45\textwidth]{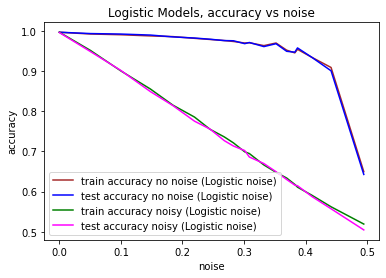}
    \caption{Logistic model with Logistic noises: Accuracy VS Noise.}
    \label{fig:Noiseaccuracytest}
\end{figure}

%\begin{figure}[h]
%    \centering
%    \includegraphics[width=0.45\textwidth]{Codes/Binary(d=10)/Images/Test_accuracy_(noisy)_vs_noise}
%    \includegraphics[width=0.45\textwidth]{Codes/Binary(d=10)/Images/Train_accuracy_(noisy)_vs_noise}
%    \caption{Accuracy VS Noise (Noisy)}
%    \label{Noiseaccuracytrain}
%\end{figure}
\subsubsection{\bf Sample Complexity in High Noise Setting}\label{noiseaffectcom}
Now we focus on the setting where the loss function and the noise are compatible.  
In other words, we only consider the Logistic model for Logistic noise, the Laplace model for Laplace noise, and the HS model for HS noise.
As explained in Section~\ref{sec:model}, the scale parameter $s$ for the noise distribution ${\rm Noise}(\eta | 0,s)$ directly determines the portion of mislabeling 
imposed on the data.  In Figure~\ref{fig:Noiseaccuracytest}, we observe that the accuracy drops when the noise gets more intense. 
However, in theory, we proved that each model could overcome any amount of noise perturbation. 
The noise scale parameter (variance) affects the sample complexity through the constants $\beta$ and $B$ (see Section~\ref{sec:model}:~\ref{assumption1} 
and the discussion after). In Theorems~\ref{lossapprox}~and~\ref{parapprox}, we proved that irrespective of the amount of noise, we can recover the ground truth parameters if the number of samples is sufficiently large. 
However in Figure~\ref{fig:Noiseaccuracytest}, we saw that if the Logistic noise changes around $50\%$ of the labels, then the test accuracy drops to $65\%$  
when we have $15000$ samples in the training set. Supported by the theoretical results, we should expect more and more accuracy 
if we increase the number of training points. 
To verify this, in this experiment we fix the amount of noise at $45\%$ and gradually increase the number of training points to $2\times 10^5$ samples; Figure~\ref{Noiseaccuracysamplecomplexity} reports the resulting accuracy.  
We observe that model accuracy with respect to the ground truth labels is approaching one.  With $2\times 10^5$ training samples, we have around $97\%$
accuracy on the test data. This observation adheres to our theoretical results about the recovery power of the method.

\begin{figure}[h]
    \centering
    \includegraphics[width=0.325\textwidth]{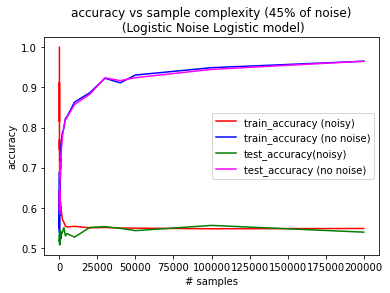}  
%    \begin{minipage}[b]{.5\textwidth}
    \includegraphics[width=0.325\textwidth]{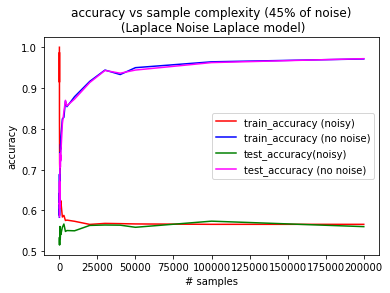}  
    \includegraphics[width=0.325\textwidth]{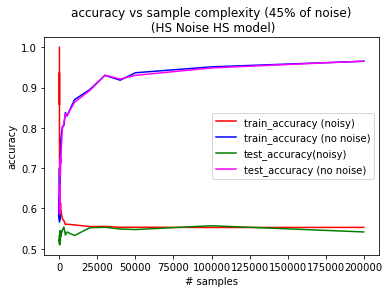}    
%    \end{minipage}
    \caption{Accuracy VS Sample complexity with $45\%$ noise when the model is aligned with the noise type.}
    \label{Noiseaccuracysamplecomplexity}
\end{figure}
For further experiments about the behavior of the loss function and a higher dimensional example, see Appendix \ref{Further_Experimental}.  

%\begin{figure}[h]
%    \centering
%    \includegraphics[width=0.4\textwidth]{Codes/Binary(d=10)/Images/Test_accuracy_no_noisy_vs_close_point_portion} 
%    \includegraphics[width=0.4\textwidth]{Codes/Binary(d=10)/Images/L_1_norm_vs_close_point_portion} 
%    \caption{Left: Test accuracy no noisy vs close point portion. Right: $L_1(f)$-distance between $\M^*/\tau^*$ and $\hat{\M}/\hat{\tau}$}
%    \label{close_point_portion}
%\end{figure}

\subsection{Comparing to DML-eig}\label{Comp-DML-eig}
Inspired by a work of \cite{NIPS2002_c3e4035a}, \cite{ying2012distance}  developed an eigenvalue optimization framework (called DML-eig) for learning a Mahalanobis metric. They define an acceptable optimization problem and elegantly reduce it to minimizing the maximal eigenvalue of a symmetric matrix problem \citep{doi:10.1137/0609021,lewis_overton_1996}. In their formulation, given pairs of similar data points and pairs of dissimilar data points, the goal is to learn a Mahalanobis metric which preserves similarity and dissimilarity. 
More specific, they look for a p.s.d. matrix $\M^*$ to maximize the minimal squared distances between dissimilar pairs while the sum of squared distances between similar pairs is at most $1$. This setting is comparable to ours since we also look for a matrix $\M^*$ (and also a threshold $\tau^*$) to distinguish labeled far pairs of points from labeled close pairs of points. Their work did not study how noise can affect their model, nor if it could potentially recover a ``ground truth" model that generates the dissimilar and similar labels. 
However, we can compare our model to theirs empirically by passing our similarities and dissimilarities to their model and checking whether their model can handle noise or recover ground truth parameters.
We can also use the $\hat{\M}_{\text{eig}}$ learned by their model with the best $\hat{\tau}$ to see how well they can predict the labels. 

\begin{figure}[h]     
    \includegraphics[width=0.49\textwidth]{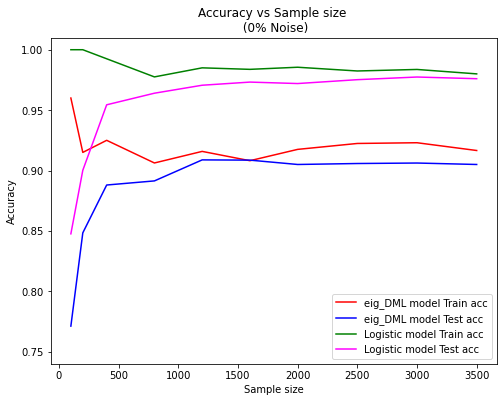}
%\begin{minipage}[b]{0.5\linewidth}
   \includegraphics[width=0.49\textwidth]{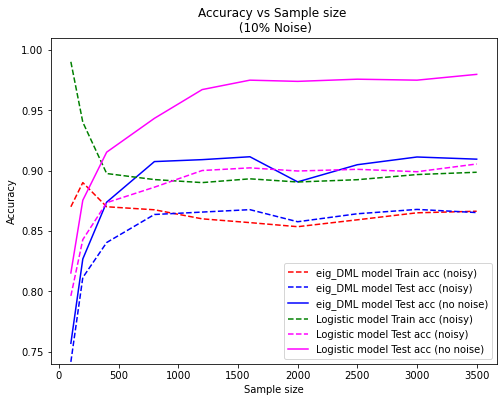}
%\end{minipage}
\caption{\label{fig:Logistic_eig_DML_Accuracy_vs_Sample_size} Performance of DML-eig with/without noise vs sample complexity.}
\end{figure}

In Figure~\ref{fig:Logistic_eig_DML_Accuracy_vs_Sample_size}, we compare our model with DML-eig using the data generated in Section \ref{data_generation}. We set the noise level at $0\%$ or at $10\%$ and use the Logistic noise distribution. The number of training points are indicated as the sample size and the test size is always fixed at $5000$ points. For the noisy data, we evaluate performance with respect to both the noisy and ground truth labels (described as ``no noise").
We can see that our model outperforms DML-eig in each setting and for any sample size. 
Although the DML-eig model can neutralize the noise (the blue curves in the left and right images are about the same), its accuracy in the noiseless setting is only around $90\%$ at best. 
In comparison, our model quickly overcomes the noise and its accuracy approaches $100\%$ (shown as the magenta curves).  
In the noisy setting, we train on the noisy labels training data, and show results for both techniques.  Then we compare against the test data with respect to both the noisy (dashed curves) and ground truth labels, and report the results.  Our approach can recover the parameters even under noise (the magenta curve), and the noisy test data matches the noisy training data (so no over fitting).  
On the other hand, DML-eig only achieves about 90\% test accuracy with respect to the ground truth labels, and about 85\% train and test accuracy with respect to the noisy labels.
%On the other hand, even without noise on the test data, DML-eig only roughly matches the noisy training results (about $90\%$ accuracy), and does worse with the noisy labels about $85\%$ accuracy.  

Next we observe that the DML-eig model is far less scalable than our approach.  This is because it takes several matrix multiplications and eigenvalue solves for each subgradient step.
In Figure~\ref{fig:time_complexity}, we compare its training time to our model.  While our model always takes less than a second on a sample size up to 3500, we see that DML-eig quickly surpasses the 20 minute mark (1200 seconds) and starts to become intractable.

\begin{figure}[h]    
    \centering
   \includegraphics[width=0.49\textwidth]{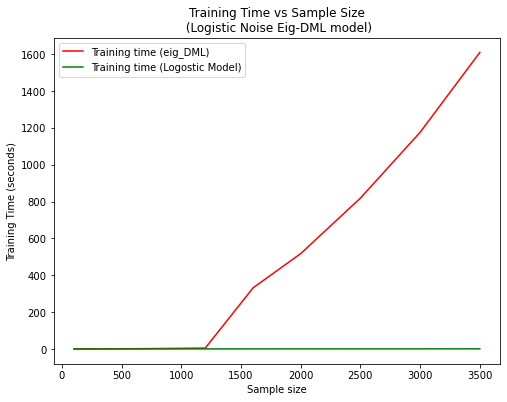}
    \caption{\label{fig:time_complexity} Training Time Complexity comparison between our Logistic and DML-eig.}
\end{figure}

If we provide our model with $10{,}000$ training points, after $17$ seconds, it can reach a test accuracy of $90\%$ with respect to the noisy labels and 99\% with respect to the ground truth labels, while DML-eig after 3 hours and 45 minutes cannot do better than $85\%$ and $90\%$ respectively.

 % run time = 17.223711013793945
 % train_acc_noisy = 0.8997
 % train_no_noisy_acc = 0.985

 % test_acc_noisy = 0.9058
 % test_no_noisy_acc = 0.988
    
 % run time = 13551.12533712387
 % best_train_acc_noisy = 0.8569
 % train_acc_no_noisy = 0.897

 % test_noisy_acc = 0.8574
 % test_no_noisy_acc = 0.8932
 
\subsection{Real Data Experiments}
\label{sec:exp-real}
In this section we use our methods to solve some real data challenges which demand or benefit from a learned Mahalanobis distance.  The first one is from a physical simulation where we aim to find a reduced order model that needs to pass through a linear projection.  Thus the learned scaling must be linear, and is desired to be low rank.  
The second one is a consumer satisfaction story, and instead of using pairs of data points $\z= \x-\y$, it directly uses data points  $\z$ as the input, where satisfaction is predicted as a Mahalanobis distance from the origin.  In both cases we show our method achieves the objectives with high accuracy.

\subsubsection{\bf Finding a Low-Rank Metric for Equations of State Combustion Simulation}
\label{sec:combustion}
%\meysam{Should be paraphrased!! I used original wording from the paper!!}\\

We first consider a data set generated by~\cite{spitfire} to represent instances of the equations of state of a thermo-chemical reacting system.  The goal is to model a combustion process to produce more efficient fuels and for easier CO2 capture and sequestration.  The data set consists of $30{,}000$ data points in $\R^{9}$, with $8$ dimensions capturing \emph{equations of state}, that is, the fractional composition of various chemicals (like O2, H, CO2) present in the system, and the $9$th coordinate being temperature~\citep{ZDYBAL2022}.  The goal is to learn a reduced order model (ROM) which should linearly project and scale~\citep{sutherland2009combustion} to a lower dimension, so one can then learn a PDE modeling the physics.  The PDE modeling on the ROM is only tractable through a linear transformation  \citep{sutherland2009combustion}, so non-linear approaches are not permitted.   And choosing an appropriate scaling is crucial since there is a difference of several orders of magnitude in coordinate ranges.  

We consider two ways of generating labels to apply our methodology.  The first is based on the best known engineered solution~\citep{ZDYBAL2022} which we attempt to replicate from the training part of a test/train split.  The second is via how close data points are on a critical simulation value called ``mixture fraction.''

For the best known engineered solution, we start with a provided ``ground truth" feature transformation matrix $\A^* \in \R^{9 \times 3}$ as found by \cite{ZDYBAL2022}.  We set $\M^* = \A^* (\A^*)^\top$.  
Given the original and the projected data, we choose a threshold and label the disjoint pairs of the original data points as far and close based on their projected distance and the threshold. Now, we have $15{,}000$ pairs of original points, and we divide them into a train set of size $10{,}000$ and a test set of size $5{,}000$. We only use the train pairs of data points and their labels to recover $\M^*$. 
Recovering the labels, we obtain $99.57\%$ and $99.51\%$  accuracy on the train and test points, respectively.
We have summarized the results in Tables~\ref{dim-reduction} and \ref{dim-reduction:truncated}.

\begin{table}[h]
    \centering
\begin{tabular}{lccc }
%\hline
%\multicolumn{3}{|c|}{Accuracy} \\
\hline
																		& Raw data   	& Normalized data  
																		& Covariance normalization  \\
\hline
Mean training accuracy 															& -- 		& 99.57 \%  &   99.81 \%  \\[.2cm] 
%\hline
Mean test accuracy																& --	     	& 99.51\%   &   99.73\%  \\[.2cm]
$\left\|\frac{\hat{\M}}{\hat{\tau}} - \frac{\M^*}{\tau^*}\right\|_2 / \left\|\frac{\M^*}{\tau^*}\right\|_2$ 				& --	     	& 0.748  	   &   0.041 \\[.2cm]
$\left\|\frac{\hat{\M}}{\hat{\tau}} - \frac{\M^*}{\tau^*}\right\|_F / \left\|\frac{\M^*}{\tau^*}\right\|_F$      & --	     	& 0.742  	   &   0.046 \\[.2cm]
\hline 
%$ \frac{\left|\hat{\tau} - \frac{\tau^*}{s}\right|}{\frac{\tau^*}{s}}$ 						&0.023 		&?	& ? \\ 
%\hline
\end{tabular}
   \caption{Mean accuracies and precisions for recovery using $\hat{\M}$ (average over $20$ trails).}
    \label{dim-reduction}
\end{table}

\paragraph{\it Normalizing data.}
The method does not converge when we input the raw data as provided to our solver; this is due to the different scaling of coordinates.  
The magnitude of some coordinates of the data are huge (temperature) and some are very small (CO2 percentage). 
So in the gradient descent, the learning rate is the same for all variables; the algorithm does not show convergence on the variable with very small values, even after a large number of steps. We can solve this problem by doing coordinate-wise normalization of the data, thus scaling all coordinates similarly, and this process is labeled \emph{Normalized data} in Table \ref{dim-reduction}. However, even in this setting, we do not have a good parameter recovery; see the last two rows of the second column of Table~\ref{dim-reduction} with about $0.75$ relative error in $\hat \M/\hat \tau$. 
To analyze this situation, we compute the estimated covariance matrix from the data, we observe that the variance of the data in some directions is almost zero.   Note that these directions are not along coordinates, they are a linear combination of the coordinates, so coordinate-wise normalization does not correct for it.    
Note that in the theoretical results, we have the parameter recovery only if the support of data distribution has a nonzero Lebesgue measure which is effectively not the case here. 

To address the remaining issue, note that if we change the behavior of $\A^*:\R^d\longrightarrow \R^k$ only for those directions that the data variance is zero, then the distance in the projected space remains almost always the same and thus the labeling remains the same as well. 
This implies that there is no unique $\M^*$ to recover.
However, if we rescale the data and $\A^*$ by $\sqrt{C}$ and the data by $(\sqrt{C})^{-1}$, where $C$ is the covariance matrix of the data, we correct for this issue.  
Indeed, if we set $\A^*_{new} = \sqrt{C}\A^*$ and $X_{new} = X_{normal} (\sqrt{C})^{-1}$, then, in this setting, $\M^*_{new} = \sqrt{C}\A^*(\A^*)^\top \sqrt{C}$ has a negligible effect in the  directions where $C$ has a small variance. As the gradient descent initiates $\M$ with very small entries and it will receive no substantial update in those directions, finally we recover $\M^*_{new}$ very well, see the third column of Table~\ref{dim-reduction} labeled \emph{Covariance normalization}.

\paragraph{\it Low-rank recovery.}  
Moreover, we can truncate the matrix $\hat{\M}$ to a rank-$k$ matrix $\hat{\M}_k$ by setting the last $d-k$ eigenvalues of $\hat{\M}$ to zero.  In Table~\ref{dim-reduction:truncated}, we summarize the average accuracies (over $20$ independent sample observations) for the case that we use $\hat{\M}_k$ instead of $\hat{\M}$ for $k=1,2,3,4$. We can see that the for $k =1 $, we still have $90\%$ accuracy, for $k=2$ we have $98\%$ accuracy.
For $k=3$, the $\hat{\M}_k$ and $\hat{\M}$ are indistinguishable.  Hence, we can recover the best Mahalanobis distance $\M^*$ up to very high classification accuracy and in parameters, even with a desired low-rank solution. We preprocess data here using Covariance normalization.  
% \sout{{\blue Here we work with the raw data without any kind of normalization. }}
% \jeff{I don't believe this.  Otherwise how to do we get parameter recover?} \meysam{ Right. My mistake.}

\begin{table}[h]
    \centering
\begin{tabular}{p{4.5cm}p{2.cm}p{2cm} p{2cm} p{2cm}}
%\hline
%\multicolumn{3}{|c|}{Accuracy} \\
\hline
																			& $k=1$   		& $k=2$ 	   & $k=3$ 	& $k= 4$  \\
\hline
Mean training accuracy 															& 89.77\% 	& 97.87 \%  &   99.81 \%  &99.81 \% \\ 
%\hline
Mean test accuracy																& 89.75\%	     	& 97.86\%   &   99.72\%  	&99.73\%\\
$\left\|\frac{\hat{\M}_k}{\hat{\tau}} - \frac{\M^*}{\tau^*}\right\|_2 / \left\|\frac{\M^*}{\tau^*}\right\|_2$ 		& 0.301	     	& 0.052 	   &   0.041 	&0.041\\
$\|\frac{\hat{\M}_k}{\hat{\tau}} - \frac{\M^*}{\tau^*}\|_F / \|\frac{\M^*}{\tau^*}\|_F$      				& 0.294	     	& 0.064  	   &   0.046 	&0.046\\
\hline 
%$ \frac{\left|\hat{\tau} - \frac{\tau^*}{s}\right|}{\frac{\tau^*}{s}}$ 						&0.023 		&?	& ? \\ 
%\hline
\end{tabular}
   \caption{Mean accuracies and precisions for dimension reduction recovery using truncated $\hat{\M}_k$ for $k=1,2,3,4$ 
   (average over $20$ trails).}
    \label{dim-reduction:truncated}
\end{table}

\begin{paragraph}{\it Mixture fraction labeling.}
One of the features in the data is called the mixture fraction, which takes values between $0$ and $1$. We remove this feature from the data set and use it to label the points as far and close.
We first randomly extract $15{,}000$ disjoint pairs from the data. Then we compute the absolute of their mixture fraction difference, and based on an appropriate threshold, we assign the far and close label to the pairs. We choose a threshold $\tau^*$ such that the generated labels are balanced. We partition the data into $10{,}000$ training points and $5{,}000$ test points. Now, we try to see whether there is a matrix $\M^*$ and a threshold $\tau^*$ that can replicate this labeling. 
Indeed, we are able to find $\hat{\M}$ and $\hat{\tau}$ for which we have $99.68\%$ accuracy for the test set, which is basically as good as our recovery of the best engineered solution from \cite{ZDYBAL2022}. Notably, we again only have this accuracy if we normalize the data. If we work with the raw data directly, the best performance is about $70\%$. We summarize the corresponding results in Table~\ref{mixture-fraction-labeling}.

\begin{table}[h]
    \centering
\begin{tabular}{lcc}
%\hline
%\multicolumn{3}{|c|}{Accuracy} \\
\hline
				& Raw data   		& Normalized data    \\
\hline
Training accuracy 	& 71.34\% 	     	& 99.83 	 \\ 
%\hline
Test accuracy		& 70.74\%	     		& 99.78  	 \\
\hline 
%$ \frac{\left|\hat{\tau} - \frac{\tau^*}{s}\right|}{\frac{\tau^*}{s}}$ 						&0.023 		&?	& ? \\ 
%\hline
\end{tabular}
   \caption{Accuracy for the mixture fraction labeling.}
    \label{mixture-fraction-labeling}
\end{table}

\end{paragraph}

\subsubsection{\bf Airline Passenger Satisfaction}\label{airlinedata}
We consider a data set containing a training set of around $100{,}000$ points and a test set of around $26{,}000$ points from the
Airline Passenger Satisfaction \citep{Airlinesatis} data set. 
Each data point contains $24$ features; $20$ are real-valued, and $4$ are categorical features containing 
{\bf Gender:} Female, Male,  
{\bf Customer Type:} Loyal, disloyal,
{\bf Type of Travel:} Business, Personal, 
{\bf Class:} Eco, Eco Plus, Business.
As the first three are binary, we simply convert them to $0$ and $1$. The fourth one is also ordinal, and we convert it to $0$, $1$, and $2$, respectively. 
The satisfaction of each passenger is given as either ``Neutral or dissatisfied'' or ``Satisfied''.
We simply interpret ``Neutral or dissatisfied'' as 0 (${\rm Close}$) and ``Satisfied'' as 1 (${\rm Far}$). The train and test sets contain about 45\% of satisfied passengers which means that the data is almost balanced. 
Here we assume that passenger satisfaction is determined by a Mahalanobis norm (distance from the origin).
Indeed, we model the problem as if there is some matrix $\M^*$ and threshold $\tau^*$, so for each data point $\z$, $\|\z\|_{\M^*}^2 + {\rm Noise}_{\z} \geq \tau^*$, for some unknown unbiased noise ${\rm Noise}_{\z}$, if and only if the corresponding passenger is satisfied with the airline. We would like to find the generating $\M^*$ and $\tau^*$.
Compared to the theoretical setting, here we are given $\z=\x-\y$ or in other words, we can assume that $\y=\0$.
Using each of the Logistic, Laplace, and HS models ({\it $learning\_rate = 0.045$ and $number\_iterations = 20{,}000$}), we can recover the satisfaction labeling with $93\%$ accuracy on the training and test data. 
It is notable that, similar to previous observations, we obtain these accuracies only for normalized data. 
Other methods, such as random forest, can boost the labeling accuracy to $96\%$ (see~\citealp{Airlinesatis}).
However, our method gives us a feature transformation $\hat{A}$, which is more interpretable compared to something like the random forest. We can find the most important directions for $\hat{A}$ to see what combination of features has the greatest effect on the satisfaction level of passengers. 

 While this can be modeled as a standard classification task, we believe it is worthwhile to consider it as a linear distance metric learning one.  For one, it shows that we can achieve near-state-of-the-art accuracy against classic classification approaches.  But moreover, the real satisfaction can be thought of as a continuous value; then we can use an underlying metric 
to determine this satisfaction value by measuring the distance (based on that metric) of a data point from the origin. There is also a threshold such that if the satisfaction value is greater, then the passenger is happy with the service. We here assumed that such a metric and threshold can be modeled via our linear distance metric learning approach. More-so, that  threshold can vary among different persons, resulting in a noise labeling process which our formulation accounts for. Thus, the linear distance metric learning approach provides not just a classifier, but also an interpretable generative model that adds extra transparency to the task solution.

\subsubsection{\bf Breast Cancer Wisconsin (Diagnostic)}\label{Breast_Cancer} 
We here use the Breast Cancer Wisconsin Diagnostic Data Set which is publicly available through the University of California Irvine (UCI) Machine Learning Repository~\citep{misc_breast_cancer_17}.
This data has been created based on 669 samples collected from January 1989 to November 1991. There are 30 real-valued features assigned to each sample which are computed from a digitized image of a fine needle aspirate (FNA) of a breast mass. This data set contains 241 instances as malignant, and 458 instances as benign. So, this is the first unbalanced data set we explore (see Unbalanced labeling part in Appendix~\ref{Unbalanced_labeling} for an ablation study on the balance).
This data sets has been explored in the literature mostly as a classification task where the best performance has been observed by ensembles of ANN and SVM with a $100\%$ accuracy (see~\cite{Salod:2020aa}). Similar to the airline satisfaction task, we here assume that there is an underlying metric and a threshold such that comparing  the distance from the origin based on the underlying metric with the threshold determines whether a sample is benign or malignant. 
Modeling this setting via linear DML (Logistic noise, Logistic model), we shuffle and split the data into a train set of size 450 and a test set of size 119 and observed the train and test accuracies for all the samples, the benign samples, and the malignant samples separately. We did this experiment 10 times and recorded the (weighted) average accuracies as in Table~\ref{Breast_Cancer_accuracies}.
Even though the data set is unbalanced, we can see that the model does pretty well on each class of samples.
As an advantage over the ensembling classification approaches, we here have a feature transformation $\hat{A}$ which gives us a linear interpretable view of the features.  

\begin{table}[h]
    \centering
\begin{tabular}{lccc}
%\hline
%\multicolumn{3}{|c|}{Accuracy} \\
\hline
				    & All samples   	& Benign samples   &    malignant samples \\
\hline
Training accuracy 	& 99.08\% 	     	& 99.79 	       &       97.92\\ 
%\hline
Test accuracy		& 97.47\%	     	& 98.41  	       &       95.88\\
\hline 
%$ \frac{\left|\hat{\tau} - \frac{\tau^*}{s}\right|}{\frac{\tau^*}{s}}$ 						&0.023 		&?	& ? \\ 
%\hline
\end{tabular}
   \caption{Different accuracies for the Breast Cancer Wisconsin Diagnostic Data Set.}
    \label{Breast_Cancer_accuracies}
\end{table}

\section{Related Work on Linear DML} 
A considerable amount of works have been devoted to distance metric learning (\cite{barhillel05a,NGUYEN2017215,JMLR:v8:sugiyama07b,NIPS2006_dc6a7e65,4270149,XIANG20083600}). Although recent work has focused primarily on nonlinear distance metric learning, the works most relevant to this article are more classic linear approaches. The method we developed can be categorized as fully supervised linear metric learning in which the scalability is in terms of both number of examples and dimension.  Ours has bounded sample complexity is $O(\frac{d^2}{\eps^2} \log \frac{d}{\eps})$ in the dimension $d$ for $\eps$ error in our loss function, and in practice we run gradient descent, where each iteration is linear in the data size $N$ and has quadratic dependence for dimension $d$.   Related works do not provide sample complexity bounds. 
% \jeff{Is this true?}\meysam{Meysam: I am almost sure. However, let me review the papers once again. }
%\jeff{State the complexity here then in terms of $n$ and $d$, and compare to others if possible.} \meysam{Meysam: As I saw in the papers, they do not have specific bound for the complexities. I can state our complexity here and review the papers once again.}  We can use our method in combination of dimensionality reduction as well via low-rank approximation.   
Note that our method learns a Mahalanobis distance (a positive semi-definite matrix $\M$) and a threshold $\tau$.  
We next compare with the most similar prior work. 
%There are some works that tries to learn such a metric so dealing with the p.s.d. constraint.. {\blue We review some of these works here and explain how they are similar or different comparing to our method.} \jeff{what is the difference?}  %OK

\subsection{Constrained Optimization Approaches}
\cite{NIPS2002_c3e4035a} provide the first method to learn a Mahalanobis distance 
%\jeff{only a metric if full rank, they might not have that could accidentally be low-rank, so just say distance} 
by maximizing the sum of distances between points in the dissimilarity set ($\mathcal{D}$) under the constraint that the sum of squared distance between points in the similarity set ($\mathcal{S}$) is upper-bounded:  
\begin{align*}
    \max\limits_{\M \succeq 0} & \sum_{(\x_i,\x_j)\in \mathcal{D}} d_\M(x_i,\x_j) \\
    \text{s.t.}  & \sum_{(\x_i,\x_j)\in \mathcal{S}} d^2_\M(\x_i,\x_j)\leq 1.  
\end{align*}
It can be shown that this is a convex optimization which was solved by a proximal gradient ascent which, in each step, takes a gradient ascent step of the objective function, then projects back to the set of constraints, the cone of p.s.d. matrices. 
The projection to the p.s.d. cone uses full eigen-decomposition with $O(d^3)$ time complexity. So, as the dimension gets large it quickly gets intractable, while in our method we only deal with computing the 
Mahalanobis distance which takes $O(d^2)$ time complexity.

%Also, such a gradient projection gradient method usually takes a large number of iterations to become convergent.
%\jeff{do we argume this last sentence mathematically, or emperically?  What specifically are you claiming?}
%\meysam{Meysam: I saw this statement in a paper and also I remember that you mentioned that projection in p.s.d cone involves matrix eigen-decomposition which is expensive.}
%  We argue the next one (DML-eig) performs better than this one, and compare to it, so we are fine.  

Note that the model proposed by \cite{NIPS2002_c3e4035a} takes into account all the information of similar and dissimilar pairs by aggregating all similarity constraints together as well as all dissimilarity constraints. In contrast, the DML-eig method proposed by \cite{ying2012distance} maximizes the minimum distance between dissimilar pairs instead of maximizing the sum of their distances.  They develop a subgradient ascent procedure to optimize their formulation which does not require a projection, but still uses an $O(d^3)$ eigendecomposition step.  
%\jeff{In theory can reduce this to $O(d^2)$ maybe since only need top eigenvector ... but lets not dwell on it.}
Intuitively, this model prioritizes separating dissimilar pairs over keeping similar pairs close. Experimentally, they show that their method outperforms \cite{NIPS2002_c3e4035a}. In Subsection~\ref{Comp-DML-eig} we experimentally compare our model with DML-eig showing that our approach works better in terms of performance, accuracy, and dealing with noise. 

\subsection{Unconstrained Optimization over $A$}
Taking into account the fact that any p.s.d matrix $\M$ can be decomposed into $\M = \A\A^\top$, \cite{NIPS2004_42fe8808} define the expected leave-one-out error of a stochastic nearest neighbor classifier in the projection space induced by $\A$. They defined the probability that $\x_i$ is similar (close) to $\x_j$ as 
\begin{equation*}
    p_{ij}(\A) = \frac{\exp\left({-\|\x_i-\x_j\|^2_{\M}}\right)}{
    \sum_{k\neq i}\exp\left({-\|\x_i-\x_k\|^2_{\M}}\right)
    },\quad p_{ii} = 0
\end{equation*}
and the probability that $\x_i$ is correctly classified as 
$$p_i = \sum_{\{j: (\x_i,\x_j)\in \mathcal{S}\}}p_{ij}.$$
To learn $\A$, they solve $\arg\max\limits_{\A} \sum_{i}p_i$. 
This is not a convex optimization and thus leads to a local maximum rather than a global one.

Using an MLE approach, \cite{5459197} define a Logistic loss function which matches a special case of our loss function when the noise comes from a Logistic distribution. They start with the assumption that the similarity is determined via a Bernoulli distribution whose success probability (being similar) is $\sigma(\tau-d_\M(x_i,x_j))$. Maximizing the likelihood under this probabilistic model then yields an approximation of $\M$ and $\tau$. This is exactly the approach we take in Subsection~\ref{Logisticnoisemodel}. However, instead of directly assuming such a model, we derive our model from certain noise assumptions and thus obtain the best theoretical model possible under these assumptions. We also prove that alternate (nonLogistic) formulations of the loss function are the MLE solution under different noise assumptions. Moreover, we conclude that each of these models is capable of parameter recovery with a sufficiently large sample size.

In each of the above settings, we can rewrite the optimization function in terms of $\A$, where $\M =\A\A^\top$. %\avl{I'm a little confused by this sentence since previously we state that \cite{NIPS2002_c3e4035a} requires an expensive $O(d^3)$ projection step.}\meysam{ Although it fixes the expensive sub-gradient projection issue, it makes the problem non-convex and thus no grantee for convergence! So, in their formulation they cannot simultaneously use this trick and have the convexity.} 
This makes the problem unconstrained which is a good advantage since it eliminates expensive sub-gradient projection and moreover, we can restrict $
\A$ to be rectangular inducing a low-rank $\M$. The main limitation of this formulation (even when $\A$ is a square matrix) is that it is non-convex and thus subject to local maxima. However, in our formulation, a result by \citet{doi:10.1137/080731359} resolves this issue.

\subsection{Empirical Loss Minimization Framework}
There are some other related works proposing an empirical loss minimization framework; for a thorough review see Chapter 8 of the book by \cite{Bellet2015d}. The prediction performance of the learned metric has been studied in some works such as \cite{Balcan2008,NIPS2009_a666587a,10.1007/978-3-642-23780-5_22,Bellet2012,BelletHS12,NECO_a_00556,CaoGuo2016}. 
Broadly speaking, for an unknown data probability distribution, they considered different meaningful  constrained cost functions  as the true risk and then they studied the convergence of the empirical risk to the true risk. 
As their setting is a bit different from ours, we recall their data assumptions here. 
Given labeled examples
$\{(\x_i,y_i)\colon i =1,\ldots, N\}$ where $\x_i\in\R^d$, $\|\x_i\|\leq F$ and $y_i\in \{1,\ldots,m\}$, they define a similar (close) pair set $\S$ and a dissimilar (far) pair set $\mathcal{D}$ as follows: 
$$\S = \{(\x_i,\x_j\colon y_i = y_j\}\quad\text{and}\quad \mathcal{D} = \{(\x_i,\x_j\colon y_i \neq y_j\}.$$
Note that the pairs here are not i.i.d. as in our formulation.
From $O(n)$ given data points, we can feed our cost function with only $n$ i.i.d. pairs while they can do it with $O(n^2)$ pairs. It might give the impression that these methods should allow for much stronger convergence results, but it is not the case. 

In the following we briefly review some of their results and then compare them to ours. The primary outcome of these findings is to demonstrate the overall reliability of a metric learning approach, rather than offering precise estimations of the generalization loss.

Following the idea of maximum margin classifiers, \cite{NIPS2009_a666587a} adapted the uniform stability framework (\cite{bousquetstability} and McDiarmid's inequality) to metric learning to obtain a generalization bound. They considered 
\begin{equation}\label{eq:lossold}
    C(\M) = \frac{2c}{n(n-1)}\sum_{i<j} L\left(y_{ij}(1-\|\x_i-\x_j\|^2_{\M})\right) + \frac{1}{2}\|\M\|^2_F
\end{equation}
as the regularized empirical cost function where $y_{ij} = 1$ if $(i,j)\in \S$ and $y_{ij} = -1$ if $(i,j)\in \mathcal{D}$ and $L(z)$ is a standard loss function which is $\zeta$-Lipschitz. 
As a main result, they proved that the empirical loss $C(\M)$ converges in probability measure to the true cost $\mathbb{E}_{\x,\x', y} L\left(y(1-\|\x-\x'\|^2_{\M})\right) + \frac{1}{2}\|\M\|^2_F$  with the sample complexity %$O\left({\left(c \zeta^2 F^4 + s(d)\right)^2\ln 1/\delta \over \eps^2}\right)$ 
$O\left({s(d)^2\ln 1/\delta \over \eps^2}\right)$ 
where $s(d)$ comes from a constraint ${\rm trace}(\M)\leq s(d)$, where the hidden constant depends on $\zeta, F$.
%{\blue The constant $\zeta$ and $F$ also contributes to the bound as a constants. }
Note that if $s(d)$ is considered a constant, this provides a sample complexity independent of dimension; but it may be that the best $\M$ minimizing the cost function does not follow this constraint. Comparing to our sample complexity (Theorem~\ref{uniformconv1}), 
both bounds share almost the same dominant part (assuming $d$ is fixed). 
%{\blue When the rank of $\M$ is not constrained and $s(d)\sim d$, they essentially obtain a similar sample complexity bound similar to us but linear in terms of $d$. }
%\avl{Consider rephrasing this: ``When the rank of $M$ is not constrained and $s(d)\sim d$, they essentially obtain the same sample complexity bound as we do."} \meysam{Meysam: They proved that this constraint activates only if $s(d)\leq \sqrt{d}$. I change the sentence. Is it OK now?} \avl{Ah well in that case the red text seems misleading because it is $s(d)\sim \sqrt{d}$ which leads to the linear scaling in $d$. Maybe we should just delete all but the blue text? Or we just say they obtain a similar bound to ours but with $s(d)$ replacing $d$.}
However, they use $O(n^2)$ pairs in their optimization while we only use $n$; thus our optimization framework is more scalable in $n$. 
 In our setting, we can sample $O(n)$ disjoint pairs from the $O(n^2)$ pairs to simulate our required i.i.d. samples.   
%I think this is all we need to say; theirs must assume something is iid to get something called sampled complexity, so those extra comments about non-iid, I think, are not relevant.  
%However, in practice, one may sample a linear number of pairs by imposing local constraints in a neighborhood or select prototypes to define the pairs to overcome this problem. Note that our results are applicable whenever the sampling method grantees the iid assumption. If we are given the points individually, then simply partitioning the points to the pairs give us an iid sample of the pairs.  But, if we are given some pairs, then it is not clear how to extract an iid subset of them. } 

\cite{WBDT} considered a similar loss (without the regularizer term).
They theoretically and empirically studied the convergence of the empirical risk in this setting for some appropriate choice of $L$, focusing on log loss and hinge loss. They proved that the empirical loss converges (in probability measure) to the corresponding true loss with the rate $O(1/\sqrt{n})$. This bound does not resolve the sample complexity since it is presented as a function of $n$ without working out the dependencies to the other parameters. They also concluded that the minimizer of the empirical loss $(\hat{\M}, \hat{\tau})$ converges in measure to the minimizer of the true loss $(\M^*, \tau^*)$ as $n$ goes to infinity, but does not identify specific conditions that must hold for this to be true, or finite sample bounds, as our work does.

In a sequel, \cite{6203595} examined a data assumption that closely resembled ours, along with empirical and accurate losses like our own. 
They demonstrated that the empirical loss converges to the true loss on the optimal model, with a sample complexity comparable to ours in terms of $\eps,\delta$, and $d$. 
%\jeff{in what way?  based on $\eps$ or based on $d$ or both?}  
However, it is worth noting that their study did not include noise in their setting, it was not proven that their optimization model is theoretically optimal under some generating model parameterized by $\M^*$ as is done in this article, 
%\jeff{what does "in a theoretical sense" mean?   I don't like the work "theoretic" since technically it means something that may not be real}, 
and they did not investigate the recovery of ground truth parameters. %\sout{as we intend to do.}

% \avl{I'm not sure this is true; looking at Theorem 1 in \cite{6203595}, if one chooses $\rho=d$ to make a comparison to our framework, and we plug in the covering number from our paper, I think one will obtain an identical sample complexity bound to ours.}\meysam{I change it. See the former red lines.} 

Extending the robustness framework (\cite{XuMannor2012}) to metric learning, \cite{BELLET2015259} studied the deviation between the true and empirical loss.
The cost function they worked with is again similar to the one in Equation~\ref{eq:lossold}. They proved that the empirical loss converges in probability measure to the corresponding true loss.
We can simplify their result to the sample complexity bound  $O({s(d)+ \ln(1/\delta) \over \eps^2})$ where $s(d)$ can be exponentially large in terms of $d$.
It should be noted that the constant appearing in their data assumption impacts the constant in this sample complexity bound.
%$s(d)$ is the size of a minimal $\gamma$-cover for an appropriate $\gamma$. \avl{Where $\gamma=O(\epsilon)$, cover of what? Otherwise it is a little unclear what is being said...}
% For Frobenius or spectral norm, we know  $s(d)$ can be exponentially large in terms of $d$. 
For a fixed $d$, comparing our complexity bound with theirs, 
we again can see that the dominant parts are almost the same while their algorithm operates on $n^2$ distances for a set of $n$ points. Afterwards \cite{NECO_a_00556,CaoGuo2016}, employing a different similarity learning optimization problem,
%and for a sample-dependent hypothesis space, 
established a comparable error bound in terms of Rademacher average which is upper bounded by a function of data bound $F$. 
More precisely, under our data assumption, we can translate their error bound to a sample complexity bound which depends linearly on 
$d$ and has the dominant term  $O({\ln (1/\delta)\over \eps^2})$. It is similar to the other above-mentioned bounds.

In the following we briefly recall some advantages of our model compared to the above mentioned methods.  
\begin{itemize}
    
\item All methods discussed above model metric learning as an optimization problem which penalizes mismatches, including using constrained optimization on $\M$. However, they do not prove that these optimization problems are theoretically optimal under some generating model parameterized by $\M^*$ as is done in this article. 

\item These algorithms deal with a more expensive optimization problem which uses $O(n^2)$ pairs for $n$ points while our method uses only $O(n)$ pairs.  Despite the extra information, these algorithms do not lead to a more favorable scaling between the sample size $n$ and the error $\eps$ compared to our method.

%Despite the extra information, these algorithms do not present a considerable trade-off between sample size $n$ and error $\eps$ compared to our method. \avl{Wording unclear here... do we mean ``Despite the extra information, these algorithms do not lead to a more favorable scaling between the sample size $n$ and the error $\eps$ compared to our method."} 

\item These other methods do not provide recovery guarantees on generating parameters as we do.   This in turn allows us to provide low-rank approximation and dimensionality reduction results, since we can bound the effects of truncating small model parameters.   

\item Furthermore, since we derive the loss functions from various noise models, we can recover these model parameters even in the presence of (correctly modeled) noise.  

\end{itemize}

\subsection{Information Theoretic Modeling}
\cite{NGUYEN2017215} assume that 
$\z=\x-\y| \x,\y\in \mathcal{S}\sim \N(0,\Sigma_{\mathcal{S}})$ and $\z = \x-\y| \x,\y\in \mathcal{D}\sim \N(0,\Sigma_{\mathcal{D}})$. 
For any linear transformation $\x' = \A^\top\x$, this can be written 
$$\z' = \x'-\y'| \x,\y\in \mathcal{S}\sim \N(0,\A^\top\Sigma_{\mathcal{S}}\A) = f_\A(\z')$$ and 
$$\z' = \x'-\y'| \x,\y\in \mathcal{D}\sim \N(0,\A^\top\Sigma_{\mathcal{D}}\A) = g_\A(\z').$$
Their goal is to find $\A$ maximizing Jeffrey divergence, i.e., to solve the following optimization problem; 
\begin{equation*}
    \max\limits_{\A\in \R^{d\times k}} 
    {\rm KL}(f_\A, g_\A) + {\rm KL}(g_\A, f_\A), 
\end{equation*}
where ${\rm KL}$ stands here for Kullback-Leibler divergence. As both distributions $g_\A$ and  $f_\A$ are multivariate Gaussian, one can compute ${\rm KL}(f_A, g_A) + {\rm KL}(g_A, f_A)$ as a function of $\A$ and reduce the optimization problem to 
\begin{equation*}
    \max\limits_{\A\in \R^{d\times k}}
    {\rm tr}\left((\A^\top\Sigma_{\mathcal{S}}\A)^{-1}(\A^\top\Sigma_{\mathcal{D}}\A) +
    (\A^\top\Sigma_{\mathcal{D}}\A)^{-1}(\A^\top\Sigma_{\mathcal{S}}\A)\right)\, .
\end{equation*}
Setting the derivative of this objective function to zero, they present a solution to this non-convex optimization problem. In practice, they use MLE to replace $\Sigma_{\mathcal{S}}$ and $\Sigma_{\mathcal{D}}$ by their sample estimations, and since the formulation is not convex, the identified answer may be a local optimum.

\subsection{Low-Rank Metric Learning}
As explained in Section~\ref{intro}, linear distance metric learning approaches can be used for linear dimensionality reduction (see~\citealp{Fei-Jimeng2015}). When we are dealing with high dimensional space or a huge number of data points, solving 
Optimizations~\ref{genMLE1} (or Optimizations~\ref{genMLE1AAt} for $k=\Theta(d)$) will be costly. As reducing the matrix size in these optimizations reduces the complexity of search spaces, to resolve this issue, we can think of adding some low-rank constraint to these optimizations, e.g., ${\rm rank}(\M) \ll d$ in Optimizations~\ref{genMLE1} or $k\ll d$ in Optimizations~\ref{genMLE1AAt}. As a downside, it turns these optimization problems non-convex and thus the regular approaches such as gradient decent tend to fail easily (see~\citealp{Mu_2016,Zaiwen_Wotao2013}). 
\cite{NEURIPS2019_0d0fd7c6}, dealing with this challenge, introduced a fast low-rank metric learning method that worked well for several data points as benchmarks. Although they still have to face a non-convex optimization, as a standard way, they employed  manifold optimization methods \citep{Ankita_Saket_2015,balzano2010online,vandereycken2013low} to overcome the issue. As they worked with a different input setting as in linear DML, we are not able to present a direct comparison between their approach and ours. 

% \avl{An alternative approach is to first compress the data via a random projection, and then solve a convex optimization problem in the low-dimensional space as is done in \cite{palias2023effect}.}

%Finally, using the maximum likelihood estimation, they replaced $\Sigma_{\mathcal{S}}$ and $\Sigma_{\mathcal{D}}$ by their sample estimations. 
%\jeff{Neat.  I don't think I've seen this one in detail before.  So no optimization needed.  I wonder how it compares ... it would be great to either (a) say how it is not really comparable to ours, or (b) compare it.  I hope we can find (a) so we don't need more experiments.}
%\meysam{I will read the paper once again to answer your question. One downside I already know about this method is that they find {\bf A} solution and since the objective function is not convex that could be a local minima! Also, they need to estimate covariance matices $\Sigma_S$ and $\Sigma_D$ by the sample covariance. }
%  I think this is enough -- it has not guarantees then.  Its a neat thick, but sounds like just a hueristic ...

\section{Conclusion and Discussion}

In this paper we provide and analyze a simple and elegant method to solve the linear distance metric learning problem.  That is, given a set of iid pairs of points labeled close or far, our method learns a Mahalanobis distance that maintains these labels for some threshold. This arises when in data analysis one needs to learn how to compare various coordinates, which may be in different units, but not introduce non-linearity for reasons of interpretability, equation preservation, or maintaining linear structure.  
Our method reduces to unconstrained gradient descent, has a simple sample complexity bound, and shows convergence in a loss function and in parameter space.  In fact, this convergence holds even under noisy observations.  

Moreover, our method is the first approach to show that the learned Mahalanobis distance can be truncated to a low-rank model that can provably maintain the accuracy in the loss function and in parameters.  

Finally, we demonstrate that this method works empirically as well as the theory predicts.  We can obtain high accuracy (over $99\%$) and parameter recovery (less than $1.01$ multiplicative error) on noiseless and noisy data, and on synthetic and real data.  For instance, even if $45\%$ of the data is mislabeled we can with very high accuracy recover the true model parameters.  
Additionally we show this simple solution nearly matches the best engineered solutions on two real world data challenges.

\subsection{Limitations}
In our formulation of linear DML, we assume that we are given $N$ i.i.d. observations of pair $(\x_i, \y_i) \in \R^d \times \R^d$ and each pair is given a label $\ell_i \in \{{\rm Far, Close}\}$.  We discuss the Airline Passenger Satisfaction modeling problem in Section~\ref{airlinedata} where by always setting $\y_i$ to the origin, this is a direct and natural modeling: each observation generates one pair.  
However, in other real-world settings, we are only provided with $n$ observations with the ability to query if any pair has label ${\rm Far}$ or ${\rm Close}$. This can induce $\Theta(n^2)$ pairs, but they are not i.i.d. 
In practice, one would like to use all of these pairs, or perhaps restricted to some local constraints in a neighborhood. But since these are not i.i.d., our analysis does not apply.  
What we can do is randomly partition the data into $n/2$ pairs; now if the original $n$ observations are i.i.d., then these pairs are also i.i.d. from some distribution, and our analysis holds.  While this only generates $N = O(n)$ pairs, not $O(n^2)$, our analysis shows error converges at a rate of roughly $1/\sqrt{N}$.  This basically matches the convergence rate for known methods which use all $\Theta(n^2)$ pairs, and converge at a rate $1/\sqrt{n}$.  Managing such dependency, and potentially improving to a $1/n$ convergence rate, is a challenging direction to consider for future work.  

Another limitation in our modeling in the noisy setting, is that we prove strong convergence and parameter recovery, only when the loss function corresponds with the noise model generating the data.  In practice one does not always know the noise model, and in fact there may not be one well-defined noise model.  As a result, a user must chose a loss function.  
In this case, we generically recommend the Logistic loss.  It is widely used, has a closed form, the cdf is $1$-log-Lipschitz, and as discussed, it is fairly similar to other common noise models.  Moreover, we show empirically that it performs nearly as well under other noise types we considered as it does under Logistic noise.  

Finally, the analysis requires several clearly stated assumptions on the model \ref{assumption1} (that the parameters $\M^*$ and $\tau^*$ are bounded) and the data \ref{assumption2} (that the $2$-norm of the data is bounded).  If we do not have such bounds, then our analysis does not have a guarantee.  In fact, we observe in the experiment on Equations of State for Combustion Simulation in Section \ref{sec:combustion} that without properly normalizing, the algorithm performs poorly.  This normalization has the effect of properly shaping the data, and as a result the optimal model, so that it satisfies these assumptions.  In this case our algorithm converges quickly to small loss and recovers the near-optimal parameters.

% Acknowledgments---Will not appear in anonymized version
\acks{JP thanks NSF IIS-1816149, CCF-2115677, and CDS\&E-1953350; AL thanks NSF DMS-2309570 and NSF DMS-2136198.}

% \bibliography{ref}

\appendix

\section{Proof of Lemma~\ref{indicator}}\label{app:indicator}

\begin{proof}
To prove Lemma~\ref{indicator}, it suffices to show that $\|\z\|^2_{\M_1}=\|\z\|^2_{\M_2}$ for every $\z\in\mathbb{R}^d$.
% If $\tau_1 = 0$ (res. $\tau_2=0$), then setting $\z=\0$, implies $\tau_2 = 0$ (res. $\tau_1=0$). This yields $\|\z\|^2_{\M_1}=\|\z\|^2_{\M_2}$ for each $\z$. So, we may assume that $\tau_1,\tau_1> 0$ and thus 
W.l.o.g. we may assume that $\tau_1 = \tau_2 = 1$. 
For an arbitrary $\z_0\in\mathbb{R}^d$, 
consider the two following different cases. 
\begin{itemize}
\item $\|\z_0\|^2_{\M_1}  = 0$. In this case, if $\|\z_0\|^2_{\M_2}  > 0$, then for sufficiently large $\alpha$, we should have  $\|\alpha\z_0\|^2_{\M_2}  > 1$
while  $\|\alpha\z_0\|^2_{\M_1}  =0 < 1$ contradicting the fact that  
the two functions 
$\z\mapsto\1_{\left\{\|\z\|^2_{\M_1} - \tau_1\geq 0\right\}}$ and $\z\mapsto\1_{\left\{\|\z\|^2_{\M_2} - \tau_2\geq 0\right\}}$ agree for all $\z$.

% $\1_{\left\{\z\colon\|\z\|^2_{\M_1} - \tau_1\geq 0\right\}}$ and $\1_{\left\{\z\colon\|\z\|^2_{\M_2} - \tau_2\geq 0\right\}}$ are pointwise equal.
\item  $\|\z_0\|^2_{\M_1}  >0$ and $\|\z_0\|^2_{\M_2}  >0$. 
Consider $\alpha , \beta >0$ such that $\|\alpha\z_0\|^2_{\M_1} = \|\beta\z_0\|^2_{\M_2} = 1$.
We need to show that $\alpha = \beta$. For a contradiction, consider $c>0$ such that $\alpha<c<\beta$.
Now it is clear that 
$$\1_{\left\{\z\colon\|\z\|^2_{\M_1} - \tau_1\geq 0\right\}}(c\z_0) = 1\quad\quad\text{and}\quad\quad\1_{\left\{\z\colon\|\z\|^2_{\M_2} - \tau_2\geq 0\right\}}
(c\z_0)= 0,$$
a contradiction. 
\end{itemize}
\end{proof}

\section{Basic Properties Related to Optimization Problem~\ref{genMLE1}}\label{Basicproperties}

In this part, we derive some basic properties related to Optimization Problem~\ref{genMLE1}. 
In particular, we will see that it is a convex optimization. Using this as the main result of this subsection,
we prove that the true loss is uniquely minimized at the ground truth parameters. 

First, note that since any convex combination of two p.s.d. matrices is still p.s.d, 
using triangle inequality for spectral norm, we conclude that 
the search space $\mathcal{M}\times [0, B]$ is convex. 
\begin{observation}\label{ob1}
For every fixed $\z\in\mathbb{R}^d$ and $\ell\in\{-1,1\}$,  $-\log \sigma(\ell(\|\z\|^2_\M - \tau))$ as a function of $(\M, \tau)$ is convex.
\end{observation}
\begin{proof}
To prove the assertion, it suffices to show that $\log \sigma(\ell(\|\z\|^2_\M - \tau))$ is concave. 
Consider arbitrary $\M_1,\M_2\in\mathcal{M}$ and $\tau_1,\tau_2\in [0,B]$. 
We remind that $\log\sigma(\cdot)$ is a concave function. Also, for each $\lambda\in[0,1]$, 
\begin{align*}
\|\z\|^2_{\lambda\M_1+(1-\lambda)\M_2} - (\lambda \tau_1 +(1-\lambda)\tau_2)
& = \lambda(\|\z\|^2_{\M_1}  -  \tau_1) +
(1-\lambda)(\|\z\|^2_{\M_2}  - \tau_2).
\end{align*}
Combining these two facts implies  
 $\log \sigma(\ell(\|\z\|^2_\M-b))$ as a function of $(\M, \tau)$ is concave, completing the proof.
\end{proof}
This observation immediately implies that both $R_N(\M,\tau)$ and $R(\M,\tau)$ are convex functions as well. 
Thus 
$$\min_{(\M,\tau)\in \mathcal{M}\times[0,B]}  R(\M,\tau)\quad\quad and \quad\quad \min_{(\M,\tau)\in \mathcal{M}\times[0,B]}  R_N(\M,\tau)$$ 
are both convex optimization problems. Although the following observation is quite technical, it is necessary for the proof of succeeding results.  

%As the proofs of the next two results are technical, we present them in Appendices~\ref{app:proofofnon-singular} and~\ref{app:proofofLemma:M1=M2}.
\begin{observation}\label{non-singular}
Let $S\subseteq \mathbb{R}^d_{\geq 0}$ be a set with zero Lebesgue measure. Then 
$S^{\frac{1}{2}} = \{\x\in \mathbb{R}^d\colon \x^2\in S\}$ has also zero Lebesgue measure. 
\end{observation}
\begin{proof} 
For each $I =\{i_1,\ldots,i_k\}\subseteq [d]$, define 
$$S^{\frac{1}{2}} _I = \big\{\z\colon \exists\ \x\in S\  s.t.\ z_i = \sqrt{x_i} \text{ for } i\in[n]\setminus I \text{ and } z_i = -\sqrt{x_i} \text{ for } i\in I\big\}.$$
It is clear that $$S^{\frac{1}{2}} = \bigcup_{I\subseteq [d]}S^{\frac{1}{2}} _I.$$
Accordingly, it suffices to show $\mu(S^{\frac{1}{2}} _I) =0$ for each $I\subseteq [d]$. For an arbitrary $I\subseteq [d]$, define 
$f_i:\mathbb{R}^d_{\geq 0}\longrightarrow \mathbb{R}^d$ such that 
$$f_I(\x) = (l_1\sqrt{x_1},\ldots,l_d\sqrt{x_d})\quad\quad\text{ where} \quad\quad
l_i  = 
\left\{\begin{array}{rl}
-1 &  i\in[I]\\
 1 & i\not\in[I].
\end{array}\right.$$  
It is clear that $f_I$ is a one-to-one continuous  function and $f_I(S) = S^{\frac{1}{2}} _I$. 
%It is also useful to note that $\|f_I(\x)\|_2 = \|\x\|_1$. {\blue We know that $\|\x\|_2\leq \|\x\|_1\leq \sqrt{d}\|\x\|_2$.}
To fulfill the proof, we prove that for each $L\in\mathbb{N}$, $\mu_L\left(S^{\frac{1}{2}} _I\cap \left[-\sqrt{L},\sqrt{L}\right]^d\right) = 0$. 
This implies that 
\begin{align*}
\mu_L\left(S^{\frac{1}{2}} _I\right) & = \mu_L\left(\bigcup_{L=1}^\infty\left(S^{\frac{1}{2}} _I\cap \left[-\sqrt{L},\sqrt{L}\right]^d\right)\right)\\
& \leq \sum_{L=1}^\infty \mu_L\left(S^{\frac{1}{2}} _I\cap \left[-\sqrt{L},\sqrt{L}\right]^d\right)\\
& = 0
\end{align*}
Let $L$ be a fixed positive integer and
consider an arbitrary $\varepsilon>0$. 
For each $i\in[d]$, consider the interval $$J_i = [0,L]\times\cdots\times\underbrace{[0, \frac{\varepsilon^2}{2^{2d}L^{d-1}d^2}]}_\text{the $i$-th interval}\times\cdots\times[0, L].$$
It is clear that the volume of each $J_i$ is $\frac{\varepsilon^2}{2^{2d}d^2}$. 
Also, for the image of each $J_i$, we have 
$$f_I(J_i) \subset [-\sqrt{L},\sqrt{L}]\times\cdots\times\underbrace{[-\frac{\varepsilon}{2^dL^{\frac{d-1}{2}}d}, \frac{\varepsilon}{2^dL^{\frac{d-1}{2}}d}]}_\text{the $i$-th interval}\times\cdots\times[-\sqrt{L},\sqrt{L}].$$
This implies that the volume of $f_I(J_i)$ is at most $\frac{\varepsilon}{d}$.
Since, $f_I(\cdot)$ is Lipschitz over $[0, L]\times\cdots\times[0, L]\setminus \bigcup_{i\in[d]}J_i$, the zero Lebesgue measure sets will be mapped to 
zero Lebesgue measure sets by $f_I$.
Therefore,  
$$\mu_L\left(f_I\left(S \cap ([0, L]^d\setminus \cup_{i\in[d]}J_i)\right)\right) = 0.$$ 
Consequently, 
\begin{align*}
\mu_L\left(S^{\frac{1}{2}} _I\cap \left[-\sqrt{L},\sqrt{L}\right]\right) & = \mu_L\left(f_I\left(S \cap ([0, L]^d\setminus \cup_{i\in[d]}J_i)\right)\right) 
+ \mu_L\left(f_I\left(S \cap (\cup_{i\in[d]}J_i)\right)\right)\\
& \leq  0 + d\frac{\varepsilon}{d} = \varepsilon.
\end{align*}
Since $\varepsilon$ is arbitrary, $\mu_L\left(S^{\frac{1}{2}} _I\cap \left[-\sqrt{L},\sqrt{L}\right]^d\right)=0$, completing the proof. 
\end{proof}
 
Using this observation, we can prove the following useful lemma.
\begin{lemma}\label{M1=M2}
Let $M_1, M_2$ be two symmetric matrices and $c\in \mathbb{R}$.
If there is $Q\subseteq \mathbb{R}^d$ such that $\mu_L(Q)>0$ (Lebesgue measure) and 
$$\|\x\|^2_{\M_1} = \|\x\|^2_{\M_2} +c\quad\quad\text{for each $\x\in Q$},$$ 
then $\M_1=\M_2$ and $c=0$.
\end{lemma}
\begin{proof}
For each $\z\in Q$, we clearly have $\x_i^t (\M_1 - \M_2) \x_i= c$. 
Since $\M=\M_1 - \M_2$ is a real value symmetric metric, $\M = \U^t\D\U$ where $\U$ is an orthonormal matrix and $\D$ is a diagonal $d\times d$ matrix whose $(i,i)$ element is $a_i$. 
To prove the assertion, it suffices to show that $\D=0$. For a contradiction, suppose that $\D\neq 0$.
Set $Q' = \{\U\x: \x\in Q\}$ and $S = \left\{\y \in\mathbb{R}_{\geq 0}^{^d}: \langle\y,(a_1,\ldots,a_d)\rangle =c\right\}$. For each $\z=(z_1,\ldots,z_d)\in Q'$,  
%$$c = \z^t D \z =  \sum_{i=1}^d z_i^2a_i.$$ 
\begin{align*}
\sum_{i=1}^d z_i^2a_i & = \z^t \D \z\\
&  = \x^t\U^t\D\U\x\\
&  = \x^t\M\x = c,
\end{align*}
which concludes that 
$$Q'\subseteq \left\{\z \in\mathbb{R}^d: \langle\z^2,(a_1,\ldots,a_d)\rangle =c\right\} = S^{\frac{1}{2}}.$$
Using Observation~\ref{non-singular}, since we know $\mu_L(S) = 0$ we obtain $\mu_L(S^{\frac{1}{2}}) = 0$ and 
thus $\mu_L(Q') =0$. On the other hand, $\mu_L(Q) = \mu_L(Q')$ which implies $\mu_L(Q)=0$, a contradiction. 
\end{proof}

\section{$\varepsilon$-Cover ($\varepsilon$-Net) for $\mathcal{M}\times[0,B]$}\label{epsiloncover}
As we are going to prove a uniform convergence theorem between empirical and true losses, we need to define the $\varepsilon$-cover of a metric space. 
For a metric space $(\mathcal{X}, d)$, an $\varepsilon$-cover $\mathcal{E}$ is a subset of $\mathcal{X}$ such that for each $x\in \mathcal{X}$, 
there is some $y\in \mathcal{E}$ with $d(x,y)<\varepsilon$. 
In the following, we introduce an $\varepsilon$-cover for $\mathcal{M}\times[0, B]$. 
However, we should first define a metric over $\mathcal{X} = \mathcal{M}\times[0, B]$. 
For each $(\M_1, \tau_1), (\M_2, \tau_2)\in \mathcal{X}$, we define 
$${\rm d}\left((\M_1, \tau_1), (\M_2, \tau_2)\right) = \|\M_1- \M_2\|_2 + |\tau_1 - \tau_2|.$$

\begin{lemma}   \label{lem:eps-cover}
    There exists an $\eps$-cover $\mathcal{E}$ of $\mathcal{M} \times [0,B]$ under metric ${\rm d}$ of size 
    \[
     \frac{B}{\eps} \left( \frac{4 \beta d \sqrt{d}}{\eps} \right)^{d^2}.
    \]
\end{lemma}
\begin{proof}
%It is clear that ${\rm d}$ is a metric. 
The inequality  
$$\|\M\|_F \leq  \sqrt{d}\|\M\|_2\leq \beta \sqrt{d}$$ 
indicates that $\mathcal{M}$ is a subset of the $d^2$ dimensional Euclidean ball of the radius $\beta \sqrt{d}$ centered at the origin, i.e.,
$$\mathcal{M}=\left\{\M_{d\times d}\colon \M \text{ is p.s.d. and } \|\M\|_2 \leq \beta\right\}
\subseteq \{\M\in \mathbb{R}^{d^2}\colon \|\M\|_F\leq \beta \sqrt{d}\}.$$ 
It is known that a $k$-dimentional Euclidean ball of radius $r$ can be covered by at most $\left(\frac{2r\sqrt{k}}{\varepsilon}\right)^k$ number of balls of radius $\varepsilon$. 
So, $\mathcal{M}$ has an $\frac{\varepsilon}{2}$-cover $\mathcal{E}_2$ of size at most
$\left(\frac{4\beta d\sqrt{d}}{\varepsilon}\right)^{d^2}$.
As an $\frac{\varepsilon}{2}$-cover $\mathcal{E}_2$ for $[0, B]$ (with respect to $L_1$-norm), we can partition $[0, B]$ into $\frac{B}{\varepsilon}$ intervals of length $\varepsilon$ and consider the end points of those intervals as the $\frac{\varepsilon}{2}$-cover. 
Now the cartesian product of these two $\frac{\varepsilon}{2}$-covers,  $\mathcal{E}_1\times  \mathcal{E}_2$, is an $\varepsilon$-cover of size 
$$\frac{B}{\varepsilon}\left(\frac{4\beta d\sqrt{d}}{\varepsilon}\right)^{d^2}.$$
for $\mathcal{X} = \mathcal{M}\times [0,B]$ with respect to metric ${\rm d}$. 
\end{proof}

\section{Uniform Convergence of $R_N$ to $R$}\label{app:uc}
Although we have proved in Theorem~\ref{uniqueR} that the true loss $R(\M,\tau)$ is uniquely minimized at $(\M^*, \tau^*)$ , in reality, we do not have the true loss.
Indeed, we only have access to the empirical loss $R_N(\M,\tau)$. 
In this part, broadly speaking, we will show that $R_N(\M,\tau)$ is uniformly close to $R(\M,\tau)$ as  $N$ gets large, and then, we conclude, instead of minimizing $R_N(\M,\tau)$, we can minimize $R_N(\M,\tau)$ to approximately find $(\M^*, \tau^*)$.

In the next lemma, we will see that if the two p.s.d. matrices are close via spectrum norm, then the Mahalanobis norm defined based on these two matrices are also close. 
\begin{observation}[Equation~\ref{eq:MahCS-obs}]\label{obs4}
For given two p.s.d. $\M_1$ and $\M_2$, 
$$| \|\x\|^2_{\M_1} - \|\x\|^2_{\M_2}|\leq \|\M_1 -\M_2\|_2 \|\x\|^2.$$
\end{observation}
\begin{proof}
Using Cauchy–Schwarz inequality and the definition of the spectral norm,  we obtain 
\begin{align*}
| \|\x\|^2_{\M_1} - \|\x\|^2_{\M_2}|  & = |\x^\top(\M_1 -\M_2)\x|\\
& = \langle \x^\top,  \x^\top(\M_1 -\M_2)\rangle \\
& \leq \|\x\|\|(\M_1 -\M_2)\x\|\\
& \leq \|(\M_1 -\M_2)\|_2 \|\x\|^2,
\end{align*}
concluding the inequality.
\end{proof}

As $-\log\Phi_{\rm Noise}(\cdot)$ is a decreasing function, using Equation~\ref{uppernormz}, we have 
$$0\leq -\log \Phi_{\rm Noise}\left(\ell_i(\|\z_i\|^2_{\M} - \tau)\right)\leq -\log\Phi_{\rm Noise}(-\beta A) = T,$$
which indicates that the random variables $\z_i$'s are bounded. 
Whenever we are dealing with a summation of bounded i.i.d. random variables, 
one strong concentration inequality to use is Chernoff-Hoeffding bound.
This inequality states if $X_1, ..., X_N$ are $N$ independent random variables such that 
${\displaystyle X_{i}\in [a_{i},b_{i}]}$ almost surely for all $i$, and 
$\displaystyle S_{N}=\frac{X_{1}+\cdots +X_{N}}{N}$, then 
$$\displaystyle \operatorname {P} \left(\left|S_{n}-\mathrm {E} \left[S_{n}\right]\right|\geq \alpha\right)\leq 2\exp \left(-{\frac {2N^2\alpha^{2}}{\sum _{i=1}^{N}(b_{i}-a_{i})^{2}}}\right).$$
Since, 
\begin{align*}
R_N(\M,\tau) & =  \frac 1N\sum_{i=1}^N -\log \Phi_{\rm Noise}\left(\ell_i(\|\z_i\|^2_{\M} - \tau)\right) 
\end{align*}
and $\mathbb{E}(R_N(\M,\tau)) = R(\M,\tau)$, 
 we can use Chernoff-Hoeffding bound to control the probability that 
$|R_N(\M,\tau) - R(\M,\tau)|$ is large. 

\begin{lemma}\label{chernoffnet}
If $E=\{(\M_i, \tau_i)\colon i = 1,\ldots,m=m(\alpha)\}$ is an $\alpha$-cover for $\mathcal{M}\times [0,B]$, then 
\begin{align*}
P\big(|R_N(\M_i, \tau_i) - R(\M_i, \tau_i)|\geq \alpha \text{ for some $i\in[m]$}\big) 
& \leq 2me^{-\frac{2N\alpha^2}{T^2}}.
\end{align*}
\end{lemma}
\begin{proof}
Consider a fixed $i\in[m]$. 
For simplicity of notation, set $Z  = -\log\sigma\left(\ell(\|\x\|_{\M_i}-\tau_i)\right)$ and $Z_j  = -\log\sigma\left(\ell_j(\|\x_j\|_{\M_i}- \tau_i)\right)$. 
As we explained above, 
$Z\in[0, T]$, see Table \ref{tbl:simplenoises}, and, using Chernoff-Hoeffding Inequality, we obtain
\begin{align*}
P\big(|R_N(\M_i, \tau_i) - R(\M_i, \tau_i)|\geq \alpha \big) & = P\Big(\Big|\frac 1N\sum_{j=1}^N Z_j - E(Z)\Big|\geq \alpha\Big)\\
& \leq 2e^{-\frac{2N\alpha^2}{T^2}}.
\end{align*}
Now, using union bound, we have the desired inequality.
\end{proof}
 In the next theorem, we prove that, with high probability,  the empirical loss $R_N$ is everywhere close to the true loss $R$.
%For simplicity of notation, set 
%$$N_d(\varepsilon, \delta) =O\left(\frac{1}{\varepsilon^2}\Big[
%\log\frac{1}{\delta}  + d^2\log \frac{d}{\varepsilon} \Big]\right).$$

\begin{theorem}\label{uniformconv}[ Theorem~\ref{uniformconv1}, Restated]
For any $\varepsilon,\delta>0$, assume parameters $B$, $F$, and $\beta$ are constants, define
$$N_d(\varepsilon, \delta) =O\left(\frac{1}{\varepsilon^2}\Big[
\log\frac{1}{\delta}  + d^2\log \frac{d}{\varepsilon} \Big]\right).$$
If 
$N>N_d(\varepsilon, \delta)$, then with probability at least $1-\delta$, 
$$\sup_{(\M,\tau)\in \mathcal{M}\times [0, B]}\left|R_N(\M, \tau) - R(\M, \tau)\right|<\varepsilon.$$
\end{theorem}
\begin{proof}
To prove the assertion, we can equivalently prove 
$$P\left(\sup_{(\M, \tau)\in \mathcal{M}\times [0,B]}\left|R_N(\M, \tau) - R(\M, \tau)\right|\geq \varepsilon\right)<\delta.$$
Set $\alpha = \frac{\varepsilon}{3\zeta(F+1)}$.  
Consider $\mathcal{E}=\{(\M_i, \tau_i); i = 1,\ldots,m=m(\alpha)\}$ as an $\alpha$-cover for $\mathcal{M}\times[0, B].$
For an arbitrary $(\M, \tau)$, there is an index $i\in[m]$ such that 
$$d\left((\M, \tau)-(\M_i, \tau_i)\right) < \alpha.$$
Consequently, using Lemma~\ref{upperR1}, 
$$|R(\M, \tau) - R(\M_i, \tau_i)| <  (F+1)\zeta\alpha = \frac{\varepsilon}{3}$$
and 
\begin{align*}
|R_N(\M, \tau) - R_N(\M_i, \tau_i)| < (F+1)\zeta\alpha =\frac{\varepsilon}{3}.
\end{align*}
So far, we have proved that 
for every $(\M,b)$, there exists an index $i\in[m]$ such that 
$$|R(\M, \tau) - R(\M_i, \tau_i)| < \frac{\varepsilon}{3} \quad\quad\text{and}\quad\quad 
|R_N(\M, \tau) - R_N(\M_i, \tau_i)| < \frac{\varepsilon}{3}.$$
Using triangle inequality, it concludes 
\begin{align*}
\left|R_N(\M, \tau) - R(\M, \tau)\right| \leq &  \left|R_N(\M, \tau) - R_N(\M_i, \tau_i)\right| + \left|R_N(\M_i, \tau_i) - R(\M_i, \tau_i)\right| \\
& + \left|R(\M_i, \tau_i) - R(\M, \tau)\right|\\
 \leq & \frac{2\varepsilon}{3} + \left|R_N(\M_i, \tau_i) - R(\M_i, \tau_i)\right|.
\end{align*}
Via Lemma \ref{lem:eps-cover}, there is an
$\alpha$-cover of size 
\begin{align*}
m(\alpha) & = \frac{B}{\alpha}\left(\frac{4\beta d\sqrt{d}}{\alpha}\right)^{d^2}\\
& = \frac{3\zeta(F+1)B}{\varepsilon}\left(\frac{12\zeta(F+1)\beta d\sqrt{d}}{\varepsilon}\right)^{d^2}
\end{align*}
for $\mathcal{X} = \mathcal{M}\times [0,B]$ with respect to metric ${\rm d}$. 
On the other hand, 
\begin{align}\label{samcomexact}
\frac{T^2}{2\alpha^2}\log\frac{2m}{\delta} & =  \frac{9\zeta(F+1)^2T^2}{2\varepsilon^2}\Big[ 
\log \frac{6\zeta(1+F)B}{\varepsilon\delta} + d^2\log \frac{12\zeta(1+F)\beta}{\varepsilon} +\frac{3}{2}d^2\log d
\Big]\\
& = O\left(\frac{1}{\varepsilon^2}\Big[
\log\frac{1}{\delta}  + d^2\log \frac{d}{\varepsilon} \Big]\right)\nonumber
\end{align}
As setting 
\begin{align*}
N & > \frac{T^2}{2\alpha^2}\log\frac{2m}{\delta}\\
& = O\left(\frac{1}{\varepsilon^2}\Big[
\log\frac{1}{\delta}  + d^2\log \frac{d}{\varepsilon} \Big]\right)
\end{align*}
implies  $2me^{-\frac{2N\alpha^2}{T^2}}<\delta$, using Lemma~\ref{chernoffnet}, we obtain, with probability at least $1-\delta$,
$$\left|R_N(\M_j, \tau_j) - R(\M_j, \tau_j)\right|<\alpha = \frac{\varepsilon}{3\zeta(F+1)}<\frac{\varepsilon}{3}\quad\quad \text{for all $j\in[m]$}.$$
Therefore, if $N>O\left(\frac{1}{\varepsilon^2}\Big[
\log\frac{1}{\delta}  + d^2\log \frac{d}{\varepsilon} \Big]\right)$, with probability at least $1-\delta$, 
for all $(\M, \tau)$ we have
$$\left|R_N(\M, \tau) - R(\M, \tau)\right|\leq \varepsilon,$$
as desired.
\end{proof}

\section{Simple Noise Properties in Table~\ref{tbl:simplenoises}}\label{app:table1proof}
In Subsections~\ref{Logisticnoisemodel},~\ref{subsec:normalnoise},~\ref{subsec:laplacenoise},~\ref{subsec:Hyperbolicnoise}, we derived some properties of 
$\Phi_{\rm Noise}(\cdot)$ when noise is one of the simple noises listed in Table~\ref{tbl:simplenoises}. This section can be seen as a complementary section for those sections. For noises listed in Table~\ref{tbl:simplenoises}, one can 
%check that  $-\log\Phi_{\rm Noise}(\eta)$ is two times differentiable, and its second derivative is always positive. It implies 
verify that each of those noise distributions is simple. Here, we verify some other information listed in that table. 
%\paragraph{log-Lipschitz property of simple noises}
When  ${\rm Noise}(\eta)$ is simple, $-\log\Phi_{\rm Noise}(\eta)$ is a decreasing convex function which implies that 
${\d\over \d\eta}\left(-\log\Phi_{\rm Noise}(\eta)\right) = -\frac{{\rm Noise}(\eta)}{\Phi_{\rm Noise}(\eta)}$ is a negative increasing function. 
Therefore, $-\log\Phi_{\rm Noise}(\eta)$ is $\zeta$-Lipschitz over $[-\beta F, \beta B]$ for 
$$\zeta = \frac{{\rm Noise}(-\beta F)}{\Phi_{\rm Noise}(-\beta F)}.$$
Also, from Equation~\ref{Tvalue}, we know $$T = -\log \Phi_{{\rm Normal}}(-\beta F).$$
In what follows, we approximate $\zeta$ and $T$ for each noise choice. 
\begin{itemize}
\item {\bf Logistic noise.} In this case, it is easy to see that 
$$\zeta = \frac{\sigma(-\beta F)\sigma(\beta F)}{\sigma(-\beta F)} = \sigma(\beta F) <1$$
and $T = -\log\sigma(-\beta F) = \log(1+e^{\beta F})\leq 1 +\beta F$.
Thus, $\Phi_{\rm Logistic}(\eta)$ in $1$-log-Lipschitz and $T= O(\beta F).$ 
\item {\bf Normal noise.}  
We start with an approximation of $\Phi_{\rm Normal}(\eta)$. The derivation of the lower bound is sourced from the online content authored by John D. Cook (refer to~\cite{JohnCook}).
We set  $\Phi^c_{\rm Normal}(\eta) = 1 -  \Phi_{\rm Normal}(\eta) =  \Phi_{\rm Normal}(-\eta)$. 

% For $t>0$, 
% \begin{eqnarray*} 
% \Phi(-t) = \Phi^c(t) &=& \frac{1}{\sqrt{2\pi}} \int_t^\infty e^{-x^2/2}\, dx \\ 
% &<& \frac{1}{\sqrt{2\pi}} \int_t^\infty \frac{x}{t}e^{-x^2/2}\, dx \\ 
% &=&\frac{1}{\sqrt{2\pi}} \frac{1}{t} e^{-t^2/2}. \label{upper} 
% \end{eqnarray*}

Set $g(t) = \Phi^c(t) - \frac{1}{\sqrt{2\pi}} \frac{t}{t^2 + 1} e^{-t^2/2}$.
As $g(0) >0$,  $g'(t) = -\frac{2}{\sqrt{2\pi}} \frac{e^{-t^2/2}}{(t^2 + 1)^2}< 0$ and $\lim\limits_{t\rightarrow +\infty} g(t) = 0$, we obtain 
$\Phi^c(t) \geq \frac{1}{\sqrt{2\pi}} \frac{t}{t^2 + 1} e^{-t^2/2}$. 
% \begin{eqnarray*} 
% \Phi^c(t) > \frac{1}{\sqrt{2\pi}} \frac{t}{t^2 + 1} e^{-t^2/2}.
% \end{eqnarray*}
Using the above formula for $\zeta$, we have 
\begin{align*}
\zeta & = {1\over \sqrt{2\pi}}\frac{e^{-\frac{(\beta F)^2}{2}}}{\Phi(-\beta F)}\\
& =  {1\over \sqrt{2\pi}}\frac{e^{-\frac{(\beta F)^2}{2}}}{\Phi^c(\beta F)}\\
& \leq  \frac{(\beta F)^2 + 1}{\beta F} = O(\beta F).
\end{align*}
Also,
$$T = -\log \Phi_{{\rm Normal}}(-\beta F) = -\log \Phi^c_{{\rm Normal}}(\beta F)= O((\beta F)^2).$$
\item {\bf Laplace noise.}
In this setting, 
$$\zeta = \frac{{\rm Noise}(-\beta F)}{\Phi_{\rm Noise}(-\beta F)} = \frac{\frac{1}{2}e^{-\beta F}}{\frac{1}{2}e^{-\beta F}} = 1$$
and $T = -\log \Phi_{{\rm Normal}}(-\beta F) = \beta F \log 2 = O(\beta F)$.

\item {\bf Hyperbolic secant noise.}

Remind that $\Phi^c(t) = \int_{t}^\infty \frac{1}{2}{\rm sech}(\frac{\pi}{2}\eta)\d \eta$. 
Set 
\begin{align*}
g(t) & = \Phi^c(t) - \frac{1}{\pi}{\rm sech}(\frac{\pi}{2}t)
\end{align*}

%$$\frac {d}{dx}\operatorname {sech} x =-\tanh x\operatorname {sech} x$$
Also, note that  $g(0) >0$,  $\lim\limits_{t\rightarrow +\infty} g(t) = 0$, and 
\begin{align*}
g'(t) & =-\frac{1}{2}{\rm sech}(\frac{\pi}{2}t) + \frac{1}{2}\tanh(\frac{\pi}{2}t) \operatorname {sech}(\frac{\pi}{2}t) \\
& = \frac{1}{2}{\rm sech}(\frac{\pi}{2}t)\left(-1 + \tanh(\frac{\pi}{2}t)\right)<0\quad\quad \forall t \in\R.
\end{align*}
This implies that, for each $t\in \R$, 
$$\Phi^c(t)> \frac{1}{\pi}{\rm sech}(\frac{\pi}{2}t).$$
Using this inequality, we obtain 
\begin{align*}
\zeta & = \frac{{\rm Noise}(-\beta F)}{\Phi_{\rm Noise}(-\beta F)}
 = \frac{\frac{1}{2}{\rm sech}(-\frac{\pi}{2}\beta F)}{\Phi_{\rm Noise}^c(\beta F)}
 \leq \frac{\pi}{2}.
\end{align*}
Furthermore,
\begin{align*}
T & = -\log \Phi_{{\rm Normal}}(-\beta F)\leq -\log\frac{1}{\pi}{\rm sech}(\frac{\pi}{2}\beta F)\\
& = \log \pi + \log \cosh(\frac{\pi}{2}\beta F) = O(\beta F). 
\end{align*}
\end{itemize}

\section{Connection Between $L_1(f)$ Norm and Spectral Norm}\label{app:l_1_norm_spectral_norm}

In Theorem~\ref{parapprox} and Corollary~\ref{corl1upperbound}, we 
study the connection between $L_f(f)$ norm and the sample complexity of our problem. However, as the $L_1(f)$-metric is dependent on the distribution $f(\z)$, which is unavoidable, it is not very intuitive. We indeed prefer some more informative norms such as spectral norm. However, to this end, we must restrict the distribution $f(\z)$. 
\begin{itemize}
\item We here assume that there is a constant $c>0$, such that $f(\z)\geq c$ for each $\z\in B^d(1);$ recall we assume almost surely $\|\z\|^2 \leq F = 1$ in this section, and $B^d(1) = \{ \z \in \R^d \mid \|\z\| \leq 1 \}$.  
\end{itemize}
We next prove a statement similar to Corollary~\ref{corl1upperbound} but in terms of the $\d$-metric instead of the $L_1(f)$-metric. The following definition is also needed for the following two results. 
For $0\leq a < b\leq 1$, where $z_1$ is the first coordinate of $\z$, define  
\begin{equation}\label{conedefin}
{\rm Cone}(a,b) = \left\{\z = (z_1,\ldots,z_d)\mid 3z_1^2\geq 2\|\z\|_2^2\ \text{ and }  a\leq z_1\leq b\right\}.
\end{equation}

%Let $S = \left\{\y = (y_1,\ldots,y_d)\colon 3y_1^2\geq 2\|\y\|_2^2\right\}\cap B^d.$
%x$${\rm Volume}(Cone^{d}(h,r)) = \frac{1}{n+1}h \times {\rm Volume}(B^{d-1}_r).$$
%$${\rm Volume}(B^{d}_r) = \frac{\pi^{d/2}}{\Gamma(d/2+1)} r^d$$

% \avl{Anna: I think it would be cleaner to just include the volumes bound in this statement; could perhaps move the volume calculation to a lemma in appendix and cite it.}

\begin{lemma}\label{uniformlowerL_1_f_proof}
If $f(\z)\geq c>0$ for each $\z\in B^d(1)$, then for all $(\M,\tau)\in \mathcal{M}\times [0, B]$ 
$$\|(\M, \tau) - (\M^*,\tau^*)\|_{L_1(f)}\geq  \frac{c\pi^{d/2}}{20\Gamma(d/2+1)} (\frac{1}{18})^d \d\left((\M,\tau), (\M^*,\tau^*)\right).$$
In particular, if $f(\z)$ is uniform on unit disk, then for all $(\M,\tau)\in \mathcal{M}\times [0, B]$ 
$$\|(\M, \tau) - (\M^*,\tau^*)\|_{L_1(f)}\geq  \frac{1}{20} (\frac{1}{18})^d \d\left((\M,\tau), (\M^*,\tau^*)\right).$$
\end{lemma}
\begin{proof}
We remind that 
\begin{align*}
\|(\M, \tau) - (\M^*,\tau^*)\|_{L_1(f)} & = \int f(\z) \left|(\|\z\|^2_{\M}-\tau) - (\|\z\|^2_{\M^*}-\tau^*)\right|\d\z\\
& = \int f(\z) \Big |\z^\top (\underbrace{\hat{\M}-\M^*}_{=\bar{\M}})\z - (\underbrace{\tau-\tau^*}_{=\bar{\tau}})\Big |\d\z\\
& = \int f(\z) \left |\z^\top\bar{\M}\z - \bar{\tau}\right |\d\z.
\end{align*}
Note that $\bar{\M}$ is a symmetric matrix. So, there are a real value orthonormal matrix $\U$ and a real value diagonal matrix $\D$ such that 
$\bar{\M} = \U^\top \D \U$. Let the vector $\blambda$ denote the diagonal of $\D$.
Without loss of generality, we assume that $\lambda_1 = \max\{|\lambda_1|,\ldots, |\lambda_d|\}.$
One can verify that $\lambda_1 = \|\bar{\M}\|_2.$
As $\U$ is orthonormal, we have 
\begin{align*}
\int f(\z) \left |\z^\top\bar{\M}\z - \bar{\tau}\right |\d\z & = \int f(\z) \left |(\U\z)^\top\D (\U\z) - \bar{\tau}\right |\d\z\\
& =  \int f(\U^\top\z) \left |\z^\top\D \z - \bar{\tau}\right |\d\z\\
& = \int f(\U^\top\z)\left | \sum_{i=1}^{d}z_i^2\lambda_i - \bar{\tau}\right |\d\z
\end{align*}
Note that 
$\sum_{i=1}^d z_i^2 \lambda_i\leq \lambda_1\|\z\|^2$ and, for each $\z\in {\rm Cone}(0,1)$,
\begin{align*}
\sum_{i=1}^d z_i^2 \lambda_i
& \geq \lambda_1 z_1^2 - \sum_{i=2}^d \lambda_1 z_i^2\\ 
& = \lambda_1 \left(2z_1^2 - \sum_{i=1}^d z_i^2\right)\\
& = \lambda_1 (2z_1^2  - \|\z\|_2^2)\geq  \frac{\lambda_1}{2}z_1^2
\end{align*}
Set $q = \lambda_1 + |\bar{\tau}| = \d\left((\bar{\M},\bar{\tau}), (\M^*,\tau^*)\right)$. 
We next consider two cases based on $|\bar{\tau}|$ and $q$.  
\begin{itemize}
\item $|\bar{\tau}| \leq 0.1 q$. It implies $\lambda_1\geq 0.9 q$ and thus, if $\z\in {\rm Cone}(\frac{1}{\sqrt{3}},1)$, then 
\begin{align*}
\sum_{i=1}^d z_i^2 \lambda_i - \bar{\tau} & \geq \frac{\lambda_1}{2}z_1^2-\bar{\tau}\\
& \geq q\left(\frac{9}{20}z_1^2 - \frac{1}{10}\right)\\
& \geq \frac{1}{20}q.
\end{align*}
Therefore, 
\begin{align}
\int f(\U^\top\z)\Big | \sum_{i=1}^{d}z_i^2\lambda_i - \bar{\tau}\Big |\d\z & \geq \frac{q}{20}\int_{\z\in {\rm Cone}(\frac{1}{\sqrt{3}},1)} f(\U^\top\z)\d\z\nonumber\\
& = \frac{q}{20}\mu_f(\U^\top{\rm Cone}(\frac{1}{\sqrt{3}},1))\label{conelowermeasure}\\
& \geq \frac{cq}{20}\int_{\z\in {\rm Cone}(\frac{1}{\sqrt{3}},1)\cap B^d(1)} \d\z\nonumber\\
& = \frac{cq}{20}\times {\rm Volume}\left({\rm Cone}(\frac{1}{\sqrt{3}},1)\cap B^d(1) \right)\nonumber\\
& \geq \frac{cq}{20} \times {\rm Volume}\left({\rm Cone}(\frac{1}{\sqrt{3}},\sqrt\frac{2}{3}) \right)\nonumber\\
& = \frac{cq}{20d}\left[
\sqrt{\frac{2}{3}} V^{d-1}(\sqrt{\frac{1}{3}})) - 
\sqrt{\frac{1}{3}} V^{d-1}(\sqrt{\frac{1}{6}})\right]\nonumber\\
& = \frac{cq}{20}\frac{\pi^{\frac{d-1}{2}}}{d3^\frac{d}{2}\Gamma(\frac{d+1}{2})}\left[
\sqrt{2} - \frac{1}{2^\frac{d-1}{2}}
\right]\label{conelowervolume}.
\end{align}
\item $|\bar{\tau}| > 0.1 q$.  It implies $\lambda_1 < 0.9 q$ and thus  
\begin{align*}
\Big |\bar{\tau} - \sum_{i=1}^{d}z_i^2\lambda_i\Big | & \geq |\bar{\tau}| - \lambda_1\|\z\|_2^2\\
& \geq q\left (\frac{1}{10} - \frac{9}{10}\|\z\|_2^2\right )\\
& \geq \frac{q}{20}\quad\text{ if $\|\z\|_2^2\leq \frac{1}{18}$}
\end{align*}
Therefore, 
\begin{align}
\int f(\U^\top\z)\Big | \sum_{i=1}^{d}z_i^2\lambda_i - \bar{\tau}\Big |\d\z & \geq \frac{q}{20}\int_{\z\in B^q(\frac{1}{18})} f(\U^\top\z)\d\y\nonumber\\
& = \frac{q}{20}\mu_f(B^q(\frac{1}{18}))\label{balllowermeasure}\\
& \geq \frac{cq}{20} \times {\rm Volume}\left(B^q(\frac{1}{18})\right)\nonumber\\
& = \frac{cq}{20} \frac{\pi^{d/2}}{\Gamma(d/2+1)} (\frac{1}{18})^d\label{balllowervolume}
\end{align}
\end{itemize}
Now Lower bounds (\ref{conelowervolume}) and (\ref{balllowervolume}) implies the proof.
\end{proof}
If $f(\z)$ is a rotationally symmetric pdf, then 
$\mu_f(\U^\top B) =\mu_f(B)$ for any measurable set $B$. 
Therefore, combining two Lower bounds (\ref{conelowermeasure}) and (\ref{balllowermeasure}), we obtain the following lemma
\begin{lemma}\label{lowermainlemma}
If $f(\z)$ is a rotationally symmetric pdf, then for all $(\M,\tau)\in \mathcal{M}\times [0, B]$ 
\begin{equation*}
    {\|(\M, \tau) - (\M^*,\tau^*)\|_{L_1(f)}\over\d\left((\M, \tau), (\M^*,\tau^*)\right)} \geq \frac{1}{20}\max\left (\mu_f({\rm Cone}(\frac{1}{\sqrt{3}},1), \mu_f(B^d(\frac{1}{18})) \right ) .
\end{equation*}
\end{lemma}
% $$\|(\hat{\M}, \hat{\tau}) - (\M^*,\tau^*)\|_{L_1(f)} \geq \frac{1}{20}\max\left (\mu_f({\rm Cone}(\frac{1}{\sqrt{3}},1), \mu_f(B^d(\frac{1}{18})) \right ) \d\left((\hat{\M},\hat{\tau}), (\M^*,\tau^*)\right).$$

\section{Further Experimental Study}\label{Further_Experimental}
This section can be seen as a complementary section for 
Section~\ref{ExperimentalResults}. 

\subsection{\bf Loss Function Behavior} 
In Subsection~\ref{Logisticsubsection}, we experimentally studied the Logistic model with different noises. In Figures~\ref{fig:eigens} and \ref{accuracyepochs}, we evaluate the model for different noise in terms of eigenvalue recover and the accuracies. As a complementary information to these figures, 
in Figure~\ref{lossL1normhist}, as the iteration increases, we compare the values of the loss function $R_N$ 
on $(\hat{\M}, \hat{\tau})$, compared with  $(\M^*/s, \tau^*/s)$ which we do not expect to surpass.  Observe that the loss on $(\M^*/s, \tau^*/s)$ is a constant red line at the bottom at around $0.23$. When considering Logistic noise (blue) we reach this loss around 700 iterations, and nearly do when considering Gaussian noise.  For other types of noise, the method does worse; Noisy labeling only achieves a loss value around $0.5$.  
%In the right figure, {\blue we record the $L_1(f)$-distance between $(\frac{\hat{\M}}{\hat{\tau}},1)$ and $(\frac{\M^*}{\tau^*}, 1)$ as the iteration number increases.} \avl{Should define what we mean by this, since the $L_1(f)$ norm was defined on pairs $(M,\tau)$}

\begin{figure}[h]
    \centering
    \includegraphics[width=0.5\textwidth]{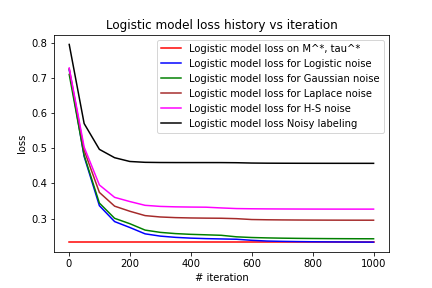}
    \caption{Loss history for Logistic model and different noises.}
    \label{lossL1normhist}
\end{figure}

\subsection{\bf Larger Dimension} 
In Section~\ref{ExperimentalResults}, we dealt with small value for dimension $d$ ($d=10$ for synthetic and $d =9, 24$ for real data).
In this section, following the same approach as in Subsection~\ref{data_generation}, we generate synthetic data with $d=100$ and ${\rm rank}(\M^*) =30$. We also set the level of noise at $20\%$. In Figure~\ref{samplecomplexitylargedim}, we observe that how the sample complexity can be affected by the dimension. When the sample size is less than $30K$, the model overfits, which is completely natural as our model has $d^2+1$ parameters. But for the larger values, we can see that the model starts to neutralize the noise and the non-noisy accuracy approaches to $1$ (blue and magenta curves). We have  $97.59\%, 97.56\%$ non-noisy train and test accuracy for $600K$ sample size.

\begin{figure}[h]
    \centering
    \includegraphics[width=0.5\textwidth]{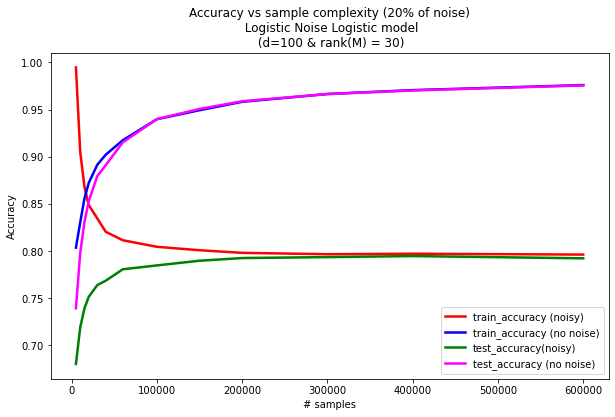}
    \caption{Sample complexity for $d=100$.}
    \label{samplecomplexitylargedim}
\end{figure} 
\subsection{\bf Unbalanced Labeling} \label{Unbalanced_labeling}
In prior experiments, the parameters are set to ensure a balanced data set, which happens often in real world scenarios. Here we experimentally study the robustness of our model against unbalanced data sets generated according to the same data generation schema explained in Section~\ref{data_generation}.   Then we gradually increase $\tau^*$ from $0.1$ to $6.1$ (for $30$ values) and record the performance of our model in predicting the ground truth no-noisy labels for each of the classes of ${\rm Far}$ and ${\rm Close}$ pairs separately. The number of generated pairs is 60,000, of which 20,000 are reserved for testing. 
At the starting state with $\tau^* = 0.1$, we have around $5\%$ of pairs as ${\rm Close}$ (the rests are ${\rm Far}$) while at the end with $\tau^* = 6.1$, around $98\%$ of the pairs are labeled ${\rm Close}$. Because of noise, we cannot expect all the pairs to have the same label for any positive $\tau^*$.  
The results are shown in  Figure~\ref{fig:Logistic_Logistc_Accuracy_unbalanced}. 
We can see that the accuracy of the model when measured on the whole data set (see magenta curve) is always very good, irrespective of the distribution of ${\rm Close}$ and ${\rm Far}$ pairs. The performance of the model on ${\rm Far}$ pairs (green curve) in the worst case drops to $93\%$ for the test set. The model on ${\rm Close}$ pairs (blue curve) drops to $78\%$ at the worst performance. 
When there are ${\rm Close}$ pairs are between about $10\%$ and $98\%$ the algorithm recovers above $90\%$ accuracy on all labeled subsets of the data.  

%The model is more Robust when the far points are more!!

\begin{figure}[h]     
   \includegraphics[width=0.49\textwidth]{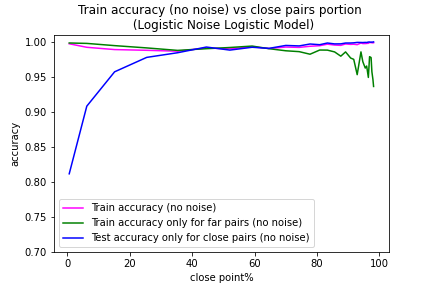}
    \includegraphics[width=0.49\textwidth]{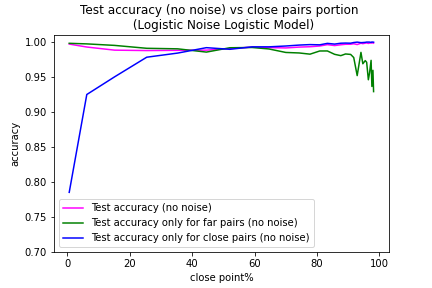}
\caption{\label{fig:Logistic_Logistc_Accuracy_unbalanced} Performance of Logistic noise Logistic model with unbalanced data.}
\end{figure}

%\bibliography{ref}
%---------------------------------------------------------------------------------------------------------------------------------------------------------------------
%---------------------------------------------------------------------------------------------------------------------------------------------------------------------

\end{document}